\PassOptionsToPackage{dvipsnames,svgnames}{xcolor}
\documentclass[a4paper,11pt]{article}

\usepackage{amsthm}
\usepackage[most]{tcolorbox}

\newtheoremstyle{upright}   %
  {10pt}                    %
  {10pt}                    %
  {}                        %
  {}                        %
  {\bfseries}               %
  {.}                       %
  { }                       %
  {}                        %

\theoremstyle{upright}

\tcbuselibrary{theorems}
\definecolor{theoremblue}{HTML}{2196F3} %

\tcolorboxenvironment{theorem}{
  enhanced jigsaw,
  sharp corners,
  colframe=theoremblue,
  colback=theoremblue!5,
  coltitle=black,
  fonttitle=\bfseries,
  borderline west={2pt}{0pt}{theoremblue},
  before skip=10pt,
  after skip=10pt,
  boxrule=0pt,
  left=10pt,
  right=10pt,
  breakable
}

\tcolorboxenvironment{definition}{
  enhanced jigsaw,
  sharp corners,
  colframe=theoremblue,
  colback=theoremblue!5,
  coltitle=black,
  fonttitle=\bfseries,
  borderline west={2pt}{0pt}{theoremblue},
  before skip=10pt,
  after skip=10pt,
  boxrule=0pt,
  left=10pt,
  right=10pt,
  breakable
}

\tcolorboxenvironment{corollary}{
  enhanced jigsaw,
  sharp corners,
  colframe=theoremblue,
  colback=theoremblue!5,
  coltitle=black,
  fonttitle=\bfseries,
  borderline west={2pt}{0pt}{theoremblue},
  before skip=10pt,
  after skip=10pt,
  boxrule=0pt,
  left=10pt,
  right=10pt,
  breakable
}

\tcolorboxenvironment{prop}{
  enhanced jigsaw,
  sharp corners,
  colframe=theoremblue,
  colback=theoremblue!5,
  coltitle=black,
  fonttitle=\bfseries,
  borderline west={2pt}{0pt}{theoremblue},
  before skip=10pt,
  after skip=10pt,
  boxrule=0pt,
  left=10pt,
  right=10pt,
  breakable
}

\tcolorboxenvironment{condition}{
  enhanced jigsaw,
  sharp corners,
  colframe=theoremblue,
  colback=theoremblue!5,
  coltitle=black,
  fonttitle=\bfseries,
  borderline west={2pt}{0pt}{theoremblue},
  before skip=10pt,
  after skip=10pt,
  boxrule=0pt,
  left=10pt,
  right=10pt,
  breakable
}

\tcolorboxenvironment{example}{
  enhanced jigsaw,
  sharp corners,
  colframe=theoremblue,
  colback=theoremblue!5,
  coltitle=black,
  fonttitle=\bfseries,
  borderline west={2pt}{0pt}{theoremblue},
  before skip=10pt,
  after skip=10pt,
  boxrule=0pt,
  left=10pt,
  right=10pt,
  breakable
}

\tcolorboxenvironment{lemma}{
  enhanced jigsaw,
  sharp corners,
  colframe=theoremblue,
  colback=theoremblue!5,
  coltitle=black,
  fonttitle=\bfseries,
  borderline west={2pt}{0pt}{theoremblue},
  before skip=10pt,
  after skip=10pt,
  boxrule=0pt,
  left=10pt,
  right=10pt,
  breakable
}

\tcolorboxenvironment{assumption}{
  enhanced jigsaw,
  sharp corners,
  colframe=theoremblue,
  colback=theoremblue!5,
  coltitle=black,
  fonttitle=\bfseries,
  borderline west={2pt}{0pt}{theoremblue},
  before skip=10pt,
  after skip=10pt,
  boxrule=0pt,
  left=10pt,
  right=10pt,
  breakable
}

\tcolorboxenvironment{remark}{
  enhanced jigsaw,
  sharp corners,
  colframe=theoremblue,
  colback=theoremblue!5,
  coltitle=black,
  fonttitle=\bfseries,
  borderline west={2pt}{0pt}{theoremblue},
  before skip=10pt,
  after skip=10pt,
  boxrule=0pt,
  left=10pt,
  right=10pt,
  breakable
}

\usepackage{hyperref}
\usepackage[utf8]{inputenc}
\usepackage[T1]{fontenc}
\usepackage{amsmath}
\usepackage{amsthm}
\usepackage{amsfonts}
\usepackage{amssymb}
\usepackage{graphicx,caption,subcaption}
\usepackage{fancybox}
\usepackage{enumitem}
\usepackage{soul}
\usepackage{mathtools}
\usepackage{lmodern}
\usepackage{stmaryrd}
\usepackage{multirow}
\usepackage{multicol}
\usepackage{xcolor}
\usepackage{tikz}
\usetikzlibrary{arrows.meta,shapes.geometric}
\usepackage{tabularx}
\usepackage{array}
\usepackage{colortbl}
\usepackage{varwidth}
\PassOptionsToPackage{dvipsnames,svgnames}{xcolor}     
\usepackage{wrapfig}
\usepackage[explicit,calcwidth]{titlesec}
\usepackage{fancyhdr}
\usepackage[left=2.1cm, right=2.1cm, top=2cm, bottom=2cm]{geometry}
\usepackage{esint}
\usepackage{tocloft}
\usepackage[skip=8pt plus1pt, indent=0pt]{parskip}
\usepackage{twoopt}
\usepackage{authblk}
\usepackage{accents}
\usepackage{mathrsfs}
\usepackage{booktabs}
\usepackage{adjustbox} %
\hypersetup{ colorlinks=true, %
    linktoc=all,     %
    linkcolor=blue,  %
}
\PassOptionsToPackage{unicode}{hyperref}
\PassOptionsToPackage{naturalnames}{hyperref}
\usepackage{float}
\usepackage{import}
\usepackage{placeins} %

\usepackage[
    backend=biber,
    uniquename=false,
    bibencoding=utf8,
    sorting=nyt,
    doi=false, isbn=false, url=false,
    style=alphabetic,
    maxcitenames=2,
    uniquelist=false,
    maxbibnames=123,
    backref=false,
    hyperref=true
]{biblatex}
\usepackage[ruled,vlined,linesnumbered]{algorithm2e}

\usepackage{todonotes}
\setuptodonotes{fancyline, color=red!20, inline}

\usepackage{silence}

\usepackage[bbgreekl]{mathbbol}
\DeclareSymbolFont{bbold}{U}{bbold}{m}{n}
\DeclareSymbolFontAlphabet{\mathbbold}{bbold}
\DeclareSymbolFontAlphabet{\mathbb}{AMSb}%

\usepackage[nameinlink, capitalise]{cleveref}

\newtheorem{theorem}{Theorem}[section]
\newtheorem*{theorem*}{Theorem}

\newtheorem{lemma}[theorem]{Lemma}
\newtheorem*{lemma*}{Lemma}

\newtheorem{prop}[theorem]{Proposition}
\newtheorem{example}[theorem]{Example}
\crefname{example}{Example}{Examples}
\Crefname{example}{Example}{Examples}
\newtheorem{assumption}{Assumption}
\newtheorem{condition}{Condition}

\usepackage{pythonhighlight}

\makeatletter  %
\renewcommand\paragraph[1]{%
  \@startsection{paragraph}{4}{\z@}%
    {3.25ex \@plus1ex \@minus.2ex}%
    {-1em}%
    {\normalfont\normalsize\bfseries}%
  {#1.}%
}
\makeatother

\long\def\sammod#1{{\color{Blue}#1}}

\newcommand{\N}[0]{\mathbb{N}}

\newcommand{\R}[0]{\mathbb{R}}
\newcommand{\C}[0]{\mathbb{C}}

\renewcommand{\L}[0]{\mathcal{L}}

\newcommand{\Sum}[2]{\displaystyle\sum\limits_{#1}^{#2}}
\newcommand{\Prod}[2]{\displaystyle\prod\limits_{#1}^{#2}}

\newcommand{\E}[1]{\mathbb{E}\left[#1\right]}
\newcommand{\V}[1]{\mathbb{V}\left[#1\right]}

\newcommand{\dd}[0]{\mathrm{d}}
\definecolor{mybluei}{RGB}{0,173,239}

\newcommand{\Span}[0]{\mathrm{Span}}

\newcommand{\Tr}[0]{\mathrm{Tr}}

\renewcommand{\epsilon}{\varepsilon}

\DeclareMathOperator{\diag}{diag}

\newcommand{\W}[0]{\mathrm{W}}

\newcommand{\app}[4]{\left\lbrace\begin{array}{ccc}
   #1 & \longrightarrow & #2 \\
   #3 & \longmapsto & #4 \\
\end{array} \right.}

\newcommand{\oll}[1]{\overline{#1}}

\renewcommand{\O}{\mathcal{O}}

\DeclareMathOperator{\KL}{KL}

\newcommand{\opn}[1]{\left\|#1\right\|_{\mathrm{op}}}
\newcommand{\PD}[1][d]{S_{#1}^{++}(\R)}
\newcommand{\PSD}[1][d]{S_{#1}^{+}(\R)}
\newcommand{\Sym}[1][d]{S_{#1}(\R)}
\newcommand{\dBW}{d_{\mathrm{BW}}}
\newcommand{\GMM}{\mathrm{GMM}}
\newcommand{\MW}{\mathrm{MW}}
\newcommand{\UMW}{\mathrm{UMW}}
\newcommand{\EEMMW}{\mathcal{E}_{\mathrm{EM-}\MW_2^2}}
\newcommand{\FEM}{F}
\newcommand{\Ethree}{\mathcal{E}_3}
\newcommand{\Etwo}{\mathcal{E}_2}
\newcommand{\JOS}{J_{\mathrm{OS}}}
\newcommand{\JAI}{J_{\mathrm{AI}}}
\newcommand{\JAD}{J_{\mathrm{AD}}}
\newcommand{\JFD}{J_{\mathrm{FD}}}
\newcommand{\FDstep}{\varepsilon_{\mathrm{FD}}}
\newcommand{\simplex}[1]{\triangle_{#1}}
\newcommand{\X}{\mathcal{X}}
\newcommand{\GD}{\mathrm{GD}}
\newcommand{\EM}{\mathrm{EM}}

\newcommand{\bfX}{\mathbf{X}}
\newcommand{\bfY}{\mathbf{Y}}

\newcommand{\redhighlight}[1]{\textcolor{red}{#1}}
\newcommand{\calE}{\mathcal{E}}

\makeatletter
\renewcommand*{\fps@figure}{h}  %
\makeatother

\renewcommand{\sammod}[1]{#1}

\title{Differentiable Expectation-Maximisation and Applications to Gaussian
Mixture Model Optimal Transport}
\author[1]{Samuel Boïté*}
\author[1]{Eloi Tanguy*}
\author[1]{Julie Delon}
\author[2]{Agnès Desolneux}
\author[3]{R\'emi Flamary}
\affil[1]{Universit\'e Paris Cit\'e, CNRS, MAP5, F-75006 Paris, France}
\affil[2]{Centre Borelli, CNRS and ENS Paris-Saclay, F-91190 Gif-sur-Yvette, France}
\affil[3]{CMAP, CNRS, Ecole Polytechnique, Institut Polytechnique de Paris}

\date{29th September 2025}
  
\begin{document}
\maketitle
*: equal contribution.

\begin{abstract}
	The Expectation-Maximisation (EM) algorithm is a central tool in statistics and machine learning, widely used for latent-variable models such as Gaussian Mixture Models (GMMs). Despite its ubiquity, EM is typically treated as a non-differentiable black box, preventing its integration into modern learning pipelines where end-to-end gradient propagation is essential. In this work, we present and compare several differentiation strategies for EM, from full automatic differentiation to approximate methods, assessing their accuracy and computational efficiency. As a key application, we leverage this differentiable EM in the computation of the Mixture Wasserstein distance $\mathrm{MW}_2$ between GMMs, allowing $\mathrm{MW}_2$ to be used as a differentiable loss in imaging and machine learning tasks. To complement our practical use of $\mathrm{MW}_2$, we contribute a novel stability result which provides theoretical justification for the use of $\mathrm{MW}_2$ with EM, and also introduce a novel unbalanced variant of $\mathrm{MW}_2$. Numerical experiments on barycentre computation, colour and style transfer, image generation, and texture synthesis illustrate the versatility of the proposed approach in different settings.

\end{abstract}

\tableofcontents

\section{Introduction}

The Expectation-Maximisation (EM) algorithm \cite{dempster1977maximum} is a
ubiquitous tool in statistics to fit mixture models on data
\cite{bailey1994fitting,enders2003using,vu2010blind,ng2013recent}. Numerous
variants of the EM algorithm have been proposed in the statistics and machine
learning communities
\cite{dayan1997using,friedman1998bayesian,fessler2002space,caffo2005ascent,ganchev2007expectation,varadhan2008simple,cappe2009line,cappe2011online,samdani2012unified,greff2017neural,zaffalon2021causal,kim2022differentiable},
and are the focus of various monographs
\cite{mclachlan2007algorithm,mclachlan2000finite}. From a theoretical
standpoint, the EM algorithm is only known to converge under specific conditions
\cite{wu1983convergence,boyles1983convergence,mclachlan2007algorithm,xu2016global},
and its behaviour is still not completely understood in full generality. In
machine learning, Gaussian priors have been used for latent space
representations
\cite{rasmussen2003gaussian,kingma2014auto,rezende2014stochastic,ho2020denoising}
with resounding success, while Gaussian Mixture Models (GMMs) are used more
sparingly \cite{nasios2006variational,viroli2019deep,yuan2020deepgmr}. 

Beyond the difficulties of GMM estimation, the core challenge is that the EM
algorithm is not easily integrated into end-to-end learning pipelines, as its
differentiation with respect to the input data is not straightforward.
Theoretical ties between the EM algorithm and an alternate minimisation of an
energy involving Entropic Optimal Transport \cite{cuturi2013sinkhorn} were
recently highlighted
\cite{rigollet2018entropic,mena2020sinkhorn,diebold2024unified,vayer2025note}.
To our knowledge, these connections do not provide insight into the
differentiation of the EM algorithm. In this paper, we tackle the theory and
practice of differentiating the EM algorithm, in the hope of sparking further
research in this direction, beyond the various applications that we present.

\sammod{\paragraph{Why differentiate EM with respect to data?} Surprisingly,
while EM is well-known and extensively studied, the question of its
differentiation seldom appears in the literature. To the best of our knowledge,
the first work in this direction is \cite{kim2022differentiable}, which rewrites
a Bayesian variant of EM which is related to the Optimal Transport Kernel
Embedding \cite{mialon2021trainable}. Differentiating EM with respect to data
allows one to include it as a layer in a neural network, as was explored in
\cite{hinton2018matrix}. Combined with a differentiable metric between mixtures
(such as Optimal Transport), this enables the use of EM for a differentiable
loss function between datasets. This would permit \textit{supervised} generative
or feature learning tasks, for instance as generalisations of the non-supervised
clustering technique proposed in \cite{adipoetra2025deep}.}

Given its iterative nature, a natural approach to differentiating EM is with
automatic differentiation
\cite{wengert1964simple,linnainmaa1970representation,rumelhart1986learning,griewank2008evaluating},
which is a powerful method for complex optimisation problems. Since EM is a
likelihood maximisation problem, the optimisation of the EM output is a
``bi-level'' problem. Numerous automatic differentiation methods for bi-level
iterative minimisation were studied in
\cite{gilbert1992automatic,beck1994automatic,shaban2019truncated,mehmood2020automatic,bolte2022automatic,bolte2024one}.
Another approach to bi-level optimisation involving fixed-point problems is the
\textit{implicit method}
\cite{luise2018differential,bai2019deep,lorraine2020optimizing,bolte2021nonsmooth,blondel2022efficient,ehrhardt2024analyzing}.

In order to illustrate the potential of differentiable EM, we apply it to
imaging tasks which rely on the comparison of GMMs with Optimal Transport (OT)
\cite{monge1781memoire,kantorovich1942translocation}. Specifically, we leverage
a variant of the Wasserstein distance between GMMs, called the
Mixture-Wasserstein distance $\MW_2$ \cite{delon2020wasserstein}, which compares
GMMs by matching their components using a small-scale discrete OT problem that
can be solved efficiently \cite{computational_ot}. The Mixture-Wasserstein
distance is one of many examples of recent advances in computational OT, where
less costly surrogates of the Wasserstein distance such as regularised
transport~\cite{cuturi2013sinkhorn} or sliced
transport~\cite{bonnotte2013sliced} have seen a wide range of success in machine
learning
\cite{bonneel2011displacement,courty2017domain,kolouri2018slicedae,karras2019style,feydy2019miccai}.
Computing transport distances between empirical distributions remains
challenging when the number of samples or the dimensionality of the space
becomes too large, even though several solutions have been proposed in the
literature~\cite{weed2019sample,genevay2019sample,chizat2020faster}.

The Mixture-Wasserstein distance has been used in texture
synthesis~\cite{leclaire2023optimal}, for the evaluation of generative neural
networks~\cite{luzi2023evaluating}, in quantum
chemistry~\cite{dalery2023nonlinear}, and for domain
adaptation~\cite{montesuma2024lighter}. An efficient barycentre computation
algorithm for this metric was also recently proposed
in~\cite{tanguy2024computing}. When using $\MW_2$ on discrete data, the space
dimension $d$ and the number of samples $n$ only appear in two stages of the
whole computation: the GMM inference on the data, and the computation of Bures
distances between the covariance matrices of the GMM components. This makes the
approach highly versatile and robust to dimensionality in practice.
Nevertheless, a current limitation of the $\MW_2$ distance is the inference of
the GMMs, which our work renders differentiable with EM, thus allowing the use
of $\MW_2$ as a differentiable loss function between datasets in machine
learning tasks.

\paragraph{Objectives} Our goal in this paper is to propose different ways to
differentiate EM, allowing for applications in imaging or machine learning
problems. As a focal application, we will pair differentiable EM with the
Mixture-Wasserstein distance $\MW_2$, which is not difficult to differentiate in
practice (using classical results on the differentiation of discrete
OT~\cite{peyre2019computational}, see also \cite[Proposition
4.3]{tanguy2025constrained}). Differentiation of the Expectation-Maximisation
algorithm is a more involved process. In this paper, we propose several
approaches for this differentiation, exact via automatic
differentiation\footnote{barring numerical approximation, which in certain cases
may cause substantial errors.} or approximate, and compare their performance on
different applications, ranging from toy examples to larger-scale machine
learning tasks. Given that our contribution is methodological in nature, our
goal is not to achieve state-of-the-art results on these applications, but
rather to illustrate the versatility of the proposed approach.

\paragraph{Paper outline} In \cref{sec:diff_em}, we begin by recalling the EM
algorithm and expressing it as a fixed point problem. We give precise
mathematical meaning to the differentiation of a solution of the EM algorithm,
and present numerous strategies to compute the differential of $T$ steps of the
method with respect to the input data. In \cref{sec:gmmot}, we provide a short
reminder on the Mixture-Wasserstein distance $\MW_2$ and show a stability result
for the estimation of $\MW_2$ between GMMs. We discuss practical difficulties in
the differentiation of the $\MW_2$ distance between GMMs estimated from data,
and provide rationale for the importance of fixing EM weights. To circumvent the
numerical difficulties incurred by weight optimisation and to ensure robustness
to Gaussian outliers, we introduce an unbalanced variant of $\MW_2$.  
In \cref{sec:diff_em_illustrations} we illustrate our methods on the flow of
$\MW_2$ composed with the EM algorithm, and perform a quantitative study on the
convergence of the EM algorithm and the quality of the gradient approximations.
In \cref{sec:diff_em_applications}, we present several applications of
differentiable EM: barycentre computation, colour transfer with inspiration from
\cite{rabin2012wasserstein}, style transfer in the spirit of
\cite{gatys2015texture}, image generation through $\MW_2$-based Generative
Adversarial Networks, and a novel texture synthesis method related to
\cite{galerne2018texture,leclaire2023optimal,houdard2023generative}.

\newpage
\section{Differentiation of the Expectation-Maximisation
Algorithm}\label{sec:diff_em}

\sammod{In this section, we study automatic differentiation methods for the
Expectation-Maximisation (EM) algorithm, which can be used to fit a Gaussian
Mixture Model (GMM) to a dataset $X$. This allows us to understand how the
mixture changes when the data are modified
(\cref{fig:em_differentiation_motivation}, left). Differentiable EM is used in
various optimisation and learning applications, such as capsule networks
\cite{hinton2018matrix} and generative clustering \cite{adipoetra2025deep}.
Differentiable EM can also serve as an intermediate layer in distribution
matching tasks (e.g. colour transfer, texture generation, and style transfer);
see the particle flow illustration (\cref{fig:em_differentiation_motivation},
right),  
where we minimise a repulsive energy $\mathcal{E}(X)=\sum_{k<\ell}
V(\mu_k(X)-\mu_\ell(X))$, with $V(x)=(\frac{1}{\lVert x\rVert} - 1)^2$ and
$(\mu_k(X))_{1\leq k\leq K}$ the means in the fitted GMM.

\begin{figure}[ht]
    \centering
    \begin{subfigure}[t]{0.47\textwidth}
        \centering
        \includegraphics[width=\linewidth]{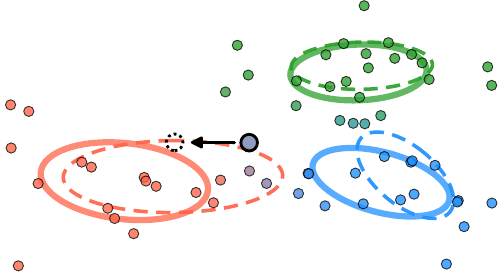}
            \end{subfigure}
    \hfill
    \begin{subfigure}[t]{0.47\textwidth}
        \centering
        \includegraphics[width=\linewidth]{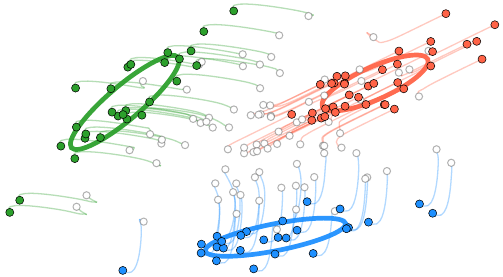}
    \end{subfigure}
    \caption{Left: sensitivity of the EM output to perturbation. Right: particle flow with repulsive energy.}
    \label{fig:em_differentiation_motivation}
\end{figure}}

\subsection{Main Ideas of the EM Algorithm}

The EM algorithm \cite{dempster1977maximum} attempts
to fit a GMM to a dataset $X = (x_1, \cdots, x_n) \in \R^{n\times d}$, with a
fixed number of components $K$. We introduce the (hidden) quantities $Y \in
\llbracket 1, K\rrbracket^n$ which encode the component index of each sample
$x_i$. The GMM parameters $\theta := \left(w, (m_k)_k, (\Sigma_k)_k\right)$ lie
in the space $\Theta := \simplex{K}\times (\R^d)^K \times (\PD)^K$, where
$\simplex{K}$ is the $K$-simplex defined as
\begin{equation}\label{eqn:simplex}
    \simplex{K} := \left\{w \in (0, 1)^K\;|\; \Sum{k=1}{K}w_k = 1\right\},
\end{equation}
and $\PD$ is the set of symmetric positive definite matrices. We will denote by
$\mu(\theta)$ the GMM probability measure of parameters $\theta$. Given $X \in
\R^{n\times d}$ and $Y \in \llbracket 1, K \rrbracket^n$, the complete
likelihood and its logarithm are respectively
\begin{equation}\label{eqn:likelihood}
    L_\theta(X, Y) =
    \Prod{i=1}{n}\Prod{k=1}{K}\left(w_k g_{m_k, \Sigma_k}(x_i)
    \right)^{\mathbbold{1}(Y_i = k)},
    \quad \ell_\theta(X, Y) =
    \Sum{i=1}{n}\Sum{k=1}{K}\mathbbold{1}(Y_i = k)
    \log\left(w_k g_{m_k, \Sigma_k}(x_i)\right),
\end{equation}
where $g_{m_k, \Sigma_k}$ is the Gaussian density with mean $m_k$ and covariance
$\Sigma_k$ (recalled in \cref{eqn:gaussian_density_g_def}). Note that
$\ell_\theta$ cannot be optimised in $\theta$ directly since we do not know the
hidden variables $Y$. The EM algorithm \cite{dempster1977maximum} maximises the
log-likelihood by iterating two steps over $\theta_t$, first computing the
``responsibilities'' $\gamma_{ik}(\theta_t)$ which are the posterior
probabilities of the hidden quantities $Y_i$ given the data $X$ and the current
parameters $\theta_t = (w^{(t)}, (m_k^{(t)})_k, (\Sigma_k^{(t)})_k) \in \Theta$:
\begin{equation}\label{eqn:gamma_ik}
    \gamma_{ik}(\theta_t) = \mathbb{P}_{(\bfX, \bfY) \sim 
    \mu(\theta_t)^{\otimes n}} \left[\bfY_i = k | \bfX = 
    X\right] = \left[w_k^{(t)} g_{m_k^{(t)}, \Sigma_k^{(t)}}(x_i)\right] / 
    \left[\Sum{\ell=1}{K}w_\ell^{(t)} 
    g_{m_\ell^{(t)}, \Sigma_\ell^{(t)}}(x_i)\right].
\end{equation}
The next iteration $\theta_{t+1}$ corresponds to a maximisation of
$\ell_\theta(X, Y)$ with the unknown variables $Y$ replaced by the posterior
probabilities $\gamma_{ik}(\theta_t)$, which leads to the following closed-form
expressions for the parameters:
\begin{equation}\label{eqn:theta_next}
    w_k^{(t+1)} = \frac{1}{n}\Sum{i=1}{n}\gamma_{ik}(\theta_t),\quad
    m_k^{(t+1)} = \frac{\Sum{i=1}{n}\gamma_{ik}(\theta_t)x_i}{\Sum{j=1}{n}\gamma_{jk}(\theta_t)},\quad 
    \Sigma_k^{(t+1)} = \frac{\Sum{i=1}{n}\gamma_{ik}(\theta_t)(x_i-m_k^{(t+1)})(x_i-m_k^{(t+1)})^\top}{\Sum{j=1}{n}\gamma_{jk}(\theta_t)}.
\end{equation}
It is a standard result (see \cite{moon1996expectation}) that the log-likelihood
$\ell_\theta(X) = \sum_{i=1}^{n}\log\left(\sum_{k=1}^{K}w_k g_{m_k,
\Sigma_k}(x_i)\right)$ increases with respect to $\theta = (w, (m_k),
(\Sigma_k))$ with each EM iteration. In practice, we shall see that it is
sometimes preferable to tweak the standard EM algorithm by not updating the
weights $w$ and keeping the weights $w_0$ of the initialisation $\theta_0$ (we
refer to this as fixing the weights). We summarise the two algorithms in
\cref{alg:EM,alg:EM_fw}, with the difference highlighted in red.

\begin{center}
    \begin{minipage}{0.49\textwidth}
        \centering
        \begin{algorithm}[H]
            \caption{EM Algorithm}
            \label{alg:EM}
            \KwIn{$\theta_0 \in \Theta,\; X \in \R^{n\times d},\; T,K$.} 
            \For{$t \in \llbracket 0, T - 1 \rrbracket$}{ 
                \textbf{Expectation:} Compute the responsibilities
                $\gamma(\theta_t)$ 
                using \cref{eqn:gamma_ik}\;
                \textbf{Maximisation:} Update
                $\theta_{t+1} = (\redhighlight{w^{(t+1)}}, (m_k^{(t+1)})_k, 
                (\Sigma_k^{(t+1)})_k)$ using \cref{eqn:theta_next}\;         
                }
        \end{algorithm}
    \end{minipage}
    \hfill
    \begin{minipage}{0.49\textwidth}
        \centering
        \begin{algorithm}[H]
            \caption{Fixed-Weights EM}
            \label{alg:EM_fw}
            \KwIn{$\theta_0 \in \Theta,\; X \in \R^{n\times d},\; T,K$.} 
            \For{$t \in \llbracket 0, T - 1 \rrbracket$}{ 
                \textbf{Expectation:} Compute the responsibilities
                $\gamma(\theta_t)$ 
                using \cref{eqn:gamma_ik}\;
                \textbf{Maximisation:} Update
                $\theta_{t+1} = (\redhighlight{w_0}, (m_k^{(t+1)})_k, 
                (\Sigma_k^{(t+1)})_k)$ using \cref{eqn:theta_next}\;         
                }
        \end{algorithm}
    \end{minipage}
\end{center}

For applications minimising a loss involving the output of the EM algorithm with
respect to the data $X$, it is important to highlight that even in the
fixed-weights version (\cref{alg:EM_fw}), the responsibilities $\gamma$ evolve
with the EM steps and with the updates of $X$. As a result, running the EM algorithm
along with $X$ updates is paramount, and it is not sufficient to keep an initial
responsibility assignment $\gamma$.

\subsection{Fixed-Point Formulation and Differentiability}

The goal of this section is to express an EM step as a fixed-point operation.
For this, a technical condition is required to ensure that the term
$\Sigma_k^{(t+1)}$ is invertible (symmetry and non-negativity of the eigenvalues
is immediate), which requires a slightly stronger condition than assuming that
the points $(x_i)$ span $\R^d$. Since the terms $\gamma_{i,k}$ are positive, we
can write the condition as:
\begin{equation}\label{eqn:data_general_position} 
    X = (x_1, \cdots, x_n) \in \X := \left\{(x_1, \cdots, x_n) \in (\R^d)^n\ |\ 
    \forall y \in \R^d,\; \Span((x_i-y)_{i\in \llbracket 1, n \rrbracket}) 
    = \R^d\right\}.
\end{equation}
For $X\in \X$ and $\theta_0\in \Theta$, the next iteration $\theta_1$ obtained
using \cref{eqn:theta_next} satisfies $\theta_1\in \Theta$. Note that if $\nu$ is
an absolutely continuous probability measure on $\R^d$, and if $n\geq d+1$, then
\cref{eqn:data_general_position} is verified almost surely for
$\bfX\sim\nu^{\otimes n}$. This condition can be seen as a weaker variant of
being in a ``standard configuration''. It can be shown that $\X$ is an open subset
of $(\R^d)^n\simeq \R^{n\times d}$.

We study the map $\FEM: \Theta \times \X \longrightarrow \Theta$ that maps a
parameter $\theta$ to the value of the next M-step using \cref{eqn:theta_next}.
For convenience, we also write $\FEM_X := \FEM(\cdot, X)$ and $\FEM_X^t$ the
$t$-th iteration $\FEM_X \circ \cdots \circ \FEM_X$ for $t\in \N$. An optimal
solution to the EM algorithm can be seen as a fixed point of $\FEM$, i.e.
$\theta^* = \FEM_X(\theta^*)$. Numerically, one takes a final iterate
$\theta_T$ of the iteration scheme
$$\forall t \in \llbracket 0, T-1 \rrbracket,\; \theta_{t+1} = \FEM(\theta_{t},
X),
$$ with an arbitrary initialisation $\theta_0 \in \Theta$. Due to the definition
of $\Theta$ as the set of parameters with weights $w_k \in (0, 1)$ and
\textit{positive} definite matrices $\Sigma_k$, the map $\FEM$ is of class
$\mathcal{C}^\infty$ (jointly in $(\theta, X)$), by virtue of the explicit
expressions \cref{eqn:theta_next}.

We now aim to give meaning to the gradient with respect to the data of a
``true'' solution of the EM algorithm. To this end, we need to work under the
assumption of convergence of EM to a non-degenerate fixed point:

\begin{assumption}\label{ass:EM_cv} There exists $(\theta_0, X_0) \in
    \Theta\times\X$ such that $\theta^*(\theta_0, X_0) :=
    \underset{t\longrightarrow +\infty}{\lim}\FEM_{X_0}^t(\theta_0)$ exists in
    $\Theta$, with $\tfrac{\partial \FEM}{\partial \theta}(\theta^*(\theta_0,
    X_0), X_0) - I$ invertible.
\end{assumption}

While convergence is typically observed in practice numerically, from a
theoretical standpoint, it is a delicate matter. See \cite[Chapter
3]{mclachlan2007algorithm} for a reference on this field of research. 
\sammod{We do not claim that the invertibility assumption holds in general and discuss counter-examples in~\cref{app:non-invertibility}.}
We are now
ready to formulate \cref{prop:existence_partial_theta_star}, which shows that
the gradient of an EM solution with respect to the data is well-defined, using
the implicit function theorem. 

\begin{prop}\label{prop:existence_partial_theta_star} Under \cref{ass:EM_cv},
    there exist open neighbourhoods $\Theta^*,\X_0$ such that $\theta^*(\theta_0,
    X_0)\in\Theta^*\subset \Theta$ and $X_0\in\X_0\subset\X$, where there exists
    $\theta^*(\theta_0, \cdot) \in \mathcal{C}^\infty(\X_0, \Theta^*)$ with:
    \begin{equation}\label{eqn:theta_star_map_def}
        \forall (\theta, X) \in \Theta^*\times\X_0,\; \FEM(\theta, X) = \theta
        \Longleftrightarrow \theta = \theta^*(\theta_0, X).
    \end{equation}
\end{prop}

\begin{proof}
    Let $G := \app{\Theta\times\X}{\Theta}{(\theta, X)}{\FEM(\theta, X) -
    \theta}$. Thanks to the regularity of $\FEM$, $G$ is of class
    $\mathcal{C}^\infty$. \cref{ass:EM_cv} implies that $G(\theta^*(X_0),
    X_0)=0$, with $\tfrac{\partial G}{\partial \theta}(\theta^*(X_0), X_0)$
    invertible.

    By the implicit function theorem, there exist open neighbourhoods
    $\Theta^*,\X_0$ such that $\theta^*(\theta_0, X_0)\in\Theta^*\subset \Theta$
    and $X_0\in\X_0\subset\X$, and there exists $g: \X_0 \longrightarrow
    \Theta^*$ of class $\mathcal{C}^\infty$ such that $g(X_0) =
    \theta^*(\theta_0, X_0)$, and:
    $$\forall (\theta, X) \in \Theta^*\times\X_0,\; \FEM(\theta, X) = \theta
    \Longleftrightarrow \theta = g(X).$$ For $X \in \X_0$, we define
    $\theta^*(\theta_0, X) := g(X)$.
\end{proof}

To alleviate notation, we will write simply $\theta^*(X) := \theta^*(\theta_0,
X)$ and continue with \cref{ass:EM_cv} and the map $\theta^*$ from
\cref{prop:existence_partial_theta_star}. For the sake of completeness and to pave
the way for future theoretical study, we provide the explicit expressions of the
partial differentials of $F$ in \cref{sec:explicit_euclidean_differentials}.
Note that from a statistical viewpoint, the parameter $\theta^*$ is not a
Maximum-Likelihood Estimator (which, in fact, does not exist for GMMs), it is
only the output of the EM algorithm which is often a local maximum of the
likelihood.

\subsection{Gradient Computation Methods}\label{sec:em_grad_methods}

In this section, we are concerned with practical implementation of the
computation of the gradient of the EM algorithm with respect to the data $
\partial_X[F^T_X(\theta_0)],$ given some initialisation $\theta_0 \in \Theta$
and number of iterations $T \geq 1$. In addition to the automatic differentiation
method, we present two alternative approximate strategies which rely only on the
last parameter $\theta_T$. These methods work under the assumption that
$\theta_T(X) \approx \theta^*(X)$, where $\theta^*$ is defined in
\cref{prop:existence_partial_theta_star} and refers to the fixed point
$\lim_{t\longrightarrow +\infty}F_X^t(\theta_0)$, assuming convergence.

\paragraph{Full Automatic Differentiation (AD)} The most naïve approach consists
in computing the gradient through all iterations using the backpropagation
algorithm (for instance, using PyTorch's automatic differentiation
\cite{pytorch}). In other words, the ``full automatic gradient'' corresponds to
letting automatic differentiation compute $\partial_X[F^T_X(\theta_0)]$
directly, using an automatically differentiable implementation of the EM
algorithm. For a large number of iterations $T$, AD can be considered as a
natural baseline for computing the exact gradient, up to numerical precision. It
is still an approximation, due to  propagation of numerical errors, but is to
the best of our knowledge the best available approximation of the true gradient,
as discussed in \cref{sec:em_grad_gt}. Nevertheless, we use AD as a natural
baseline method for comparisons. The AD method may be costly (both in time and
memory) if the number of iterations $T$ is very large.

\paragraph{Approximate Implicit Gradient (AI)} Our goal is to approximate the
gradient $\tfrac{\partial \theta^*}{\partial X}(X)$ at a fixed point $\theta^*$
of $\FEM(\cdot, X)$. Thanks to the differentiability property of
\cref{prop:existence_partial_theta_star} and under \cref{ass:EM_cv}, we can
differentiate with respect to $X$ using the chain rule:
\begin{equation}\label{eqn:partial_theta_star_chain_rule}
    \cfrac{\partial \theta^*}{\partial X}(X) 
    = \cfrac{\partial}{\partial X}\left[\FEM(\theta^*, X)\right] 
    = \cfrac{\partial \FEM}{\partial \theta}(\theta^*, X)
    \cfrac{\partial \theta^*}{\partial X}(X) 
    + \cfrac{\partial \FEM}{\partial X}(\theta^*, X).
\end{equation}
We deduce the following equation on $\cfrac{\partial \theta^*}{\partial X}(X)$:
\begin{equation}\label{eqn:partial_theta_star_implicit}
    \left(I - \cfrac{\partial \FEM}{\partial \theta}(\theta^*, X)\right)
    \cfrac{\partial \theta^*}{\partial X}(X)
    = \cfrac{\partial \FEM }{\partial X}(\theta^*, X),
\end{equation}
using \cref{ass:EM_cv} again, we can invert the matrix on the left hand-side,
yielding 
\begin{equation}\label{eqn:partial_theta_star_explicit}
    \cfrac{\partial \theta^*}{\partial X}(X) 
    = \left(I - \cfrac{\partial \FEM}{\partial \theta}(\theta^*, X)\right)^{-1}
    \cfrac{\partial \FEM }{\partial X}(\theta^*, X).
\end{equation}
We define the approximate implicit gradient by approximating $\theta^* \approx
\theta_T$:
\begin{equation}\label{eqn:def_AI}
    \JAI := \left(I - \cfrac{\partial \FEM}{\partial \theta}(
    \theta_{T}, X)\right)^{-1}\cfrac{\partial \FEM }{\partial X}(\theta_{T}, X)
\end{equation}
The implicit approximation is theoretically exact (barring numerical imprecision
in the inversion in particular) when $\theta_T = \theta^*$, however it requires
additional costly computations: first evaluate the differential $\partial_\theta
\FEM(\theta_T, X)$, and then solve a large linear system to compute $(I -
\partial_\theta \FEM( \theta^*, X))^{-1}\partial_X \FEM(\theta^*, X)$.

\paragraph{One-Step Gradient Approximation (OS)}
The One-Step method (OS) studied in \cite{bolte2024one} (within a particular
framework of bi-level optimisation), works under the following condition:
\begin{condition}\label{cond:OS_spectrum} $\opn{\tfrac{\partial \FEM}{\partial
    \theta}(\theta, X)} \leq \rho \ll 1$ for any $\theta, X$.
\end{condition}
In the case of the EM algorithm, \cref{cond:OS_spectrum} is not verified, since
the eigenvalues of the partial differential $\tfrac{\partial \FEM}{\partial
\theta}(\theta, X)$ are commonly larger than 1, even in the neighbourhood of a fixed
point. Under \cref{cond:OS_spectrum}, the OS approximation further neglects the
term $(I - \partial_\theta \FEM( \theta^*, X))^{-1}$ in \cref{eqn:def_AI},
yielding the following expression:
\begin{equation}\label{eqn:def_OS}
    \JOS := \cfrac{\partial \FEM }{\partial X}(\theta_{T-1}, X) 
    \approx \cfrac{\partial \FEM }{\partial X}(\theta^*, X) 
    \approx \left(I - \cfrac{\partial \FEM}{\partial \theta}
    (\theta^*, X)\right)^{-1}\cfrac{\partial \FEM }{\partial X}(\theta^*, X) 
    = \cfrac{\partial \theta^*}{\partial X}(X).
\end{equation}
In practice, the OS method corresponds to computing the gradient of the EM
output $\theta_T(X) = F_X^T(\theta_0)$ with respect to the data $X$ only through
the last iteration $\theta_T(X) = F_X(\theta_{T-1})$, neglecting the dependence
of the penultimate iteration $\theta_{T-1}$ on $X$. Using automatic
differentiation (for example with PyTorch \cite{pytorch}), this is done
conveniently by computing $\theta_{T-1} = F_X^{T-1}(\theta_0)$ without gradient
computation (e.g. \texttt{with torch.no\_grad()}), and performing the last step
with gradient computation. Due to this computation method, the OS gradient is
numerically inexpensive compared to the others, albeit at the expense of
precision. See \cite{bolte2024one} for a detailed presentation.

\paragraph{Time complexity and memory footprint} \begin{table}[ht]
    \centering
    \begin{tabular}{lcc}
        \toprule
        Method & Time & Memory \\
        \midrule
        Full Automatic Differentiation (AD) 
        & $\O\bigl(T(nKd^{2}+Kd^{3})\bigr)$ &
        $\O(TKd^{2}+nd)$ \\
        Approximate Implicit Gradient (AI) & $\O\bigl(TnKd^2 + K^3d^6 +
        nK^2d^5\bigr)$ & $\O(nK^2d^4)$ \\
        One-Step gradient (OS) & $\O\bigl(T(nKd^{2}+Kd^{3})\bigr)$ &
        $\O(Kd^{2}+nd)$ \\
        \bottomrule
    \end{tabular}
    \caption{Complexities of the \texttt{backward} passes (gradient
    computations) for $T$ EM iterations on $n$ points in $\mathbb{R}^{d}$ with
    $K$ components.}
    \label{tab:complexities}
\end{table}

The complexities of the gradient computation approaches are summarised in
\cref{tab:complexities}. As a baseline, the time complexity of the
\texttt{forward} pass (EM algorithm without gradients) is
$\O\bigl(T(nKd^{2}+Kd^{3})\bigr)$, while its memory footprint is
$\O(Kd^{2}+nd)$. The complexities are deduced from the differential expressions
in \cref{sec:explicit_euclidean_differentials}. The $\O(Kd^{3})$ factor
corresponds to inverting the $K$ covariance matrices of size $d\times d$ during
each E-step. The $\O(nKd^{2})$ factor comes from the M-step parameter updates
(weights, means, covariances) and the differentiation of these updates with
respect to $X$ or $\theta$.

\paragraph{The Warm-Start Method for Iteration of Differentiable EM} In many
practical applications, we are interested in minimising a certain loss function
$\mathcal{L}$ applied to the output $\theta_T(X)$ of the EM
algorithm\footnote{for example $\mathcal{L}(\theta_T(X)) =
\sammod{\MW_2^2(\mu(\theta_T(X)), \nu)}$, where $\nu$ is a target GMM, see
\cref{sec:gmmot}.}. In this case, after a (small) gradient descent step
$X_{t+1}$ computed from $X_t$, the output of the EM algorithm with data $X_t$
will often be a good initialisation for the EM algorithm with data $X_{t+1}$. As
a result, we suggest operating the EM algorithm with only one iteration and
using the output of the previous step as an initialisation. This leads to the
following algorithm, which we refer to as the Warm-Start EM Flow of a loss
$\mathcal{L}: \Theta \longrightarrow \R$.

\begin{center}
    \begin{minipage}{0.55\textwidth}
        \centering
        \begin{algorithm}[H]
            \KwIn{$\theta_0 \in \Theta,\; X_0 \in \X,\; T_\GD\in \N$,
            $\mathcal{L}: \Theta\longrightarrow \R$.} \For{$t \in \llbracket 0,
            T_\GD - 1 \rrbracket$}{ $\theta_{t+1} = \FEM(\theta_{t}, X_t)$\;
            $X_{t+1} = X_t - \eta_t \tfrac{\partial \mathcal{L}}{\partial
            \theta}(\sammod{\theta_{t+1}}) \tfrac{\partial \FEM}{\partial X}(\sammod{\theta_{t+1}}, X_t)$\; }
            \caption{Warm-Start EM Flow}
            \label{alg:warm_start_EM_flow}
        \end{algorithm}
    \end{minipage}
\end{center}

The gradient step at line 3 means that we perform automatic differentiation of
the expression $\mathcal{L}(\FEM(\sammod{\theta_{t+1}}, X_t))$ with respect to
$X$ at $X_t$, seeing $\sammod{\theta_{t+1}}$ as a constant and storing its for
the next iteration. In essence, the Warm-Start method corresponds to an online
OS gradient, which does not suffer from the approximation error of the OS method
(since only one step is performed), yet benefits from the low memory footprint
of the OS method.

\sammod{A possible intuition for the Warm-Start method is that it tackles the
bi-level problem by alternating between the ``inner'' optimisation over $\theta$
(using an EM step) and the ``outer'' optimisation over $X$ (by a gradient step).
In practice, if $\theta_{t+1} \approx \theta^\star(X_t)$ at line 2, then since
$F$ is very regular, for a small step $\eta_t$, the update $X_{t+1}$ will be
close to $X_t$ and thus $\theta_{t+1}$ will be sufficiently close to
$\theta^\star(X_{t+1})$ for one EM step to suffice (in practice, EM converges in
few iterations in many moderate-dimensional settings).}

\section{Gaussian Mixture Model Optimal Transport}\label{sec:gmmot}

\subsection{Reminders on GMM-OT}

This section summarises the main results from~\cite{delon2020wasserstein}. The
quadratic Wasserstein distance between two probability measures $\mu_0$ and
$\mu_1$ on $\mathbb{R}^d$ with finite second moments is defined by
\begin{equation}
    \label{OTgeneral:eq}
    \W_2^2(\mu_0,\mu_1) := \inf_{\pi \in \Pi(\mu_0,\mu_1)} \int_{\mathbb{R}^d\times\mathbb{R}^d} \|y_0-y_1\|_2^2 \dd\pi(y_0,y_1),
\end{equation} 
where $\Pi(\mu_0,\mu_1)$ denotes the set of probability measures with finite
second moments on $\mathbb{R}^d\times \mathbb{R}^d$ whose marginals are $\mu_0$
and $\mu_1$. A solution $\pi^*$ to~\cref{OTgeneral:eq} is called an optimal
transport plan between $\mu_0$ and $\mu_1$. This distance has been widely used
over the past fifteen years for various applications in data science. Let
$\GMM_d$ denote the set of probability measures that can be written as finite
Gaussian Mixture Models (GMMs) on~$\mathbb{R}^d$. Transport plans and
barycentres between GMMs with respect to $\W_2$ are generally not
GMMs, which is a limitation when such representations are used for data analysis
or generation. For this reason, the authors of~\cite{delon2020wasserstein}
propose to modify the $\W_2$ formulation by restricting the couplings to be GMMs
on $\R^d \times \R^d$. More precisely, given $\mu_0, \mu_1 \in \GMM_d$, one can
define
\begin{equation}\label{OTGMM:eq}
    \MW_2^2(\mu_0,\mu_1) := \inf_{\pi \in \Pi^{\text{GMM}}(\mu_0,\mu_1)} 
    \int_{\mathbb{R}^{2d}} \|y_0-y_1\|_2^2 \dd\pi(y_0,y_1),
\end{equation}
where $\Pi^{\text{GMM}}(\mu_0,\mu_1)$ denotes the set of probability measures in
$\GMM_{2d}$ with marginals $\mu_0$ and $\mu_1$. The problem is well-defined
since this set contains the product measure ${\mu_0 \otimes \mu_1}$. The authors
show that $\MW_2$ defines a distance between elements of $\GMM_d$. Moreover, if
$\mu_0 = \sum_{k=1}^{K_0} w_k^{(0)} \mu_k^{(0)}$ and $\mu_1 =
\sum_{\ell=1}^{K_1} w_\ell^{(1)} \mu_\ell^{(1)}$, where $(w_k^{(0)})_k \in
\simplex{K_0}$ and $(w_\ell^{(1)})_\ell \in \simplex{K_1}$, and $\mu_k^{(0)},
\mu_\ell^{(1)}$ are Gaussian measures, then it can be shown (\cite[Proposition
4]{delon2020wasserstein}) that
\begin{equation}\label{eq:mw2}
    \MW_2^2(\mu_0,\mu_1) = \min_{P \in \Pi(w_0,w_1)} \sum_{k,\ell} P_{kl} \W_2^2(\mu_k^{(0)},\mu_\ell^{(1)}),
\end{equation}
where $\Pi(w_0,w_1)$ is the set of $K_0 \times K_1$ matrices with non-negative
entries and marginals $w_0$ and $w_1$:
\[
\Pi(w_0,w_1) = \left\{P \in \mathcal{M}_{K_0,K_1}(\mathbb{R}^+);\; 
\forall k,\sum_j P_{kj} = w_k^{(0)} \text{ and } \forall j,\;
\sum_k P_{kj} = w_j^{(1)} \right\}.
\]
This discrete formulation makes $\MW_2$ very easy to compute in practice, even
in high dimensions. Indeed, the $\W_2$ distance between two Gaussian measures
$\mu = \mathcal{N}(m,\Sigma)$ and $\tilde{\mu} =
\mathcal{N}(\tilde{m},\tilde{\Sigma})$ admits a closed-form expression:
\begin{equation}
    \label{eq:wasserstein_gaussian}
    \W_2^2(\mu,\tilde{\mu}) = \|m - \tilde{m}\|_2^2 + \mathrm{tr}\left( \Sigma + \tilde{\Sigma} - 2 \left( \Sigma^{\frac 1 2} \tilde{\Sigma} \Sigma^{\frac 1 2} \right)^{\frac 1 2} \right),
\end{equation}
where $M^{\frac 1 2}$ denotes the unique positive semidefinite square root of
the positive semidefinite matrix $M$. If the parameters of the GMMs $\mu_0$ and
$\mu_1$ are known, computing~\cref{eq:mw2} amounts to evaluating $K_0 \times
K_1$ Wasserstein distances between Gaussians and solving a discrete transport
problem of size $K_0 \times K_1$. It is also possible to define barycentres for
$\MW_2$ \cite{delon2020wasserstein,tanguy2024computing}, which leads to a
similar discrete formulation. Given point clouds, \cite{delon2020wasserstein}
suggest using EM to fit GMMs to the data, allowing comparison of the point
clouds with EM-$\MW_2^2$.

\subsection{Stability of \texorpdfstring{$\MW_2^2$}{MW2} With Respect to GMM Parameters}\label{sec:sample_complexity}

To study the stability of the $\MW_2$ distance with respect to the GMM
parameters, we leverage a discrete OT stability result from
\cite{tanguy2023discrete_sw_losses}. To relate this problem to discrete OT
stability, we see $\MW_2^2(\mu_0, \mu_1)$ as a particular discrete Kantorovich
problem with cost matrix $C_{k_0, k_1} := \|m_{k_0}^{(0)} - m_{k_1}^{(1)}\|_2^2
+ \dBW^2(\Sigma_{k_0}^{(0)}, \Sigma_{k_1}^{(1)})$, where we recall the
expression of the Bures-Wasserstein distance on $\PSD$:
\begin{equation}\label{eqn:bures_wasserstein}
    \forall \Sigma, \Sigma' \in \PSD,\; \dBW(\Sigma, \Sigma') := 
    \sqrt{\Tr(\Sigma + \Sigma' - 2(\Sigma^{1/2}\Sigma'\Sigma^{1/2})^{1/2})}.
\end{equation}
We show that if the GMM parameters are sufficiently close (thanks to EM
convergence for instance), then the $\MW_2^2$ costs will also be close thanks to
the stability result from \cite{tanguy2023discrete_sw_losses}. While general
sample complexity results for the EM algorithm are not available, assuming a
certain rate of convergence towards the true parameters, this result shows that the
precision obtained using EM translates into a precision on the $\MW_2^2$
distance. This key observation is a first step towards guaranteeing the quality
of $\MW_2^2$ as a loss function with respect to the data. We formulate two
results: first, \cref{prop:sample_complexity_mw2_one_sample} quantifies the
decrease of $\MW_2^2(\hat\mu, \mu)$ when $\hat\mu$ is an estimator of $\mu$;
then, \cref{prop:sample_complexity_mw2_two_sample} quantifies the decrease of
$|\MW_2^2(\hat\mu_0, \hat\mu_1) - \MW_2^2(\mu_0, \mu_1)|$ when $\hat\mu_i$ are
estimators of $\mu_i$ for $i \in \{0, 1\}$. We defer the proofs to
\cref{sec:proofs_sample_complexity_mw2}.

\begin{prop}\label{prop:sample_complexity_mw2_one_sample} Consider
    GMM parameters $(\hat w, \hat m, \hat \Sigma) \in \simplex{K}\times
    \R^{K \times d} \times \PD^{K}$ of a GMM $\hat\mu$ as estimators of
    $(w, m, \Sigma) \in \simplex{K}\times \R^{K \times d} \times
    \PD^{K}$ which are parameters of a target GMM $\mu$. Assume
    ``convergence rates'' on the parameter estimations for
    $k \in \llbracket 1, K\rrbracket$:
    $$\E{\|w - \hat w\|_1} \leq \rho_w,\; \E{\W_2^2(\mathcal{N}(\hat m_k,
    \hat\Sigma_k), \mathcal{N}(m_k, \Sigma_k))} \leq \rho_\mathcal{N},$$ then
    the following stability bound holds:
    \begin{equation}\label{eqn:mw2_sample_complexity_one_sample}
        \E{\MW_2^2(\hat\mu, \mu)} 
        \leq \rho_{\mathcal{N}} + \frac{\rho_w}{2}
        \max_{k,\ell}\E{\W_2^2(\mathcal{N}(\hat m_k, \hat\Sigma_k), 
        \mathcal{N}(m_\ell, \Sigma_\ell))}.
    \end{equation}
\end{prop}

\begin{prop}\label{prop:sample_complexity_mw2_two_sample} For $i\in \{0, 1\}$,
    consider GMM parameters $(\hat w_i, \hat m_i, \hat \Sigma_i) \in
    \simplex{K_i}\times \R^{K_i \times d} \times \PD^{K_i}$ of a GMM $\hat\mu_i$
    as estimators of $(w_i, m_i, \Sigma_i) \in \simplex{K_i}\times \R^{K_i
    \times d} \times \PD^{K_i}$ which are parameters of a target GMM $\mu_i$.
    Assume that the means and covariances are bounded, namely that there exists
    $R_m>0,\; R_\Sigma>0$ such that:
    $$\forall i \in \{0, 1\},\; \forall k \in \llbracket 1, K_i\rrbracket,\;
    \|m_{k}^{(i)}\|_2 \leq R_m,\; \|\hat m_{k}^{(i)}\|_2 \leq R_m,\; \sqrt{\Tr
    \Sigma_{k}^{(i)}} \leq R_\Sigma,\; \sqrt{\Tr \hat\Sigma_{k}^{(i)}} \leq
    R_\Sigma.$$ Further assume ``convergence rates'' on the parameter
    estimations $k \in \llbracket 1, K_i\rrbracket$:
    $$ \forall i \in \{0, 1\},\; \E{\|w_i - \hat w_i\|_1} \leq \rho_w,\;
    \E{\|m_k^{(i)} - \hat m_k^{(i)}\|_2} \leq \rho_m,\;
    \E{\dBW\left(\Sigma_k^{(i)}, \hat\Sigma_k^{(i)}\right)} \leq \rho_\Sigma,$$
    then the following stability bound holds:
    \begin{equation}\label{eqn:sample_complexity_mw2_two_sample}
        \E{\left|\MW_2^2(\hat\mu_0, \hat\mu_1) - \MW_2^2(\mu_0, \mu_1)\right|} 
        \leq 8R_m \rho_m + 8R_\Sigma \rho_\Sigma + 8(R_m^2 + R_\Sigma^2)\rho_w.
    \end{equation}
\end{prop}

We complement the stability bounds of
\cref{prop:sample_complexity_mw2_one_sample} and
\cref{prop:sample_complexity_mw2_two_sample} with an empirical study. We fix
$K=3$, $d=2$, vary the sample size $n$ between $10^3$ and $2\cdot 10^4$, and for
each $n$ run $40$ repetitions of EM with $200$ iterations (with \emph{k-means++}
initialisation), noting the resulting estimates $\hat{\mu}_n$ and $\hat{\nu}_n$.
Three regimes of component separation (low, medium, high) are considered,
controlled by a scale parameter $\sigma$ applied to the covariances. The
mixtures $\mu$, $\nu$ are fixed, and we provide a visualisation in
\cref{sec:grad_comparison_gmms}. For each $n$, we report the median error and
interquartile range across repetitions.
\cref{fig:mw2_sample_complexity}\subref{fig:mw2-single} shows
$\MW_2^2(\hat\mu_n,\mu)$. With medium and high component separation, the error
decreases regularly with $n$, while with low separation it remains flat,
indicating that EM does not improve the estimation significantly in this regime.
\cref{fig:mw2_sample_complexity}\subref{fig:mw2-pairs} shows the relative error
$|\MW_2^2(\hat\mu_n,\hat\nu_n)-\MW_2^2(\mu,\nu)|/\MW_2^2(\mu,\nu)$. Medium and
high separation again yield decay with $n$, while low separation plateaus. The
results above translate parameter error rates of EM into rates for $\MW_2^2$,
but do not specify a universal value. In well-separated special cases, EM is
known to achieve a rate of $\O(n^{-1/2})$ (see \cite{kwon_em_2020} for spherical
covariances). When separation is weak, EM appears not to converge reliably,
hence the $\MW_2^2$ error does not decrease.

\begin{figure}[ht]
    \centering
    \noindent\hspace*{\fill}
    \begin{subfigure}[t]{0.4\linewidth}
        \centering
        \includegraphics[width=\linewidth]{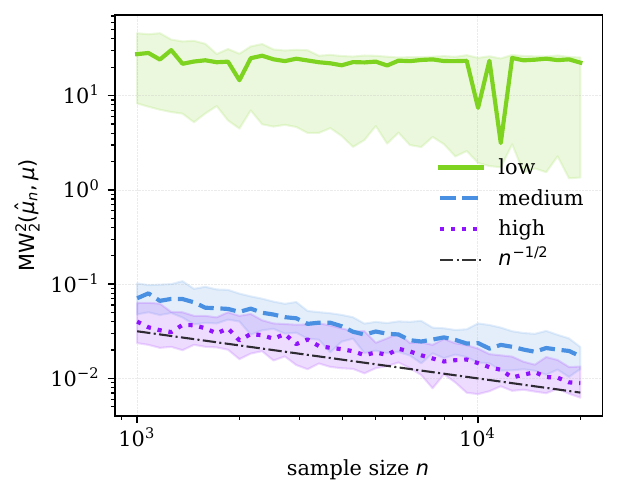}
        \caption{One-sample: $\MW_2^2(\hat\mu_n,\mu)$}
        \label{fig:mw2-single}
    \end{subfigure}
    \hfill
    \begin{subfigure}[t]{0.4\linewidth}
    \centering
    \includegraphics[width=\linewidth]{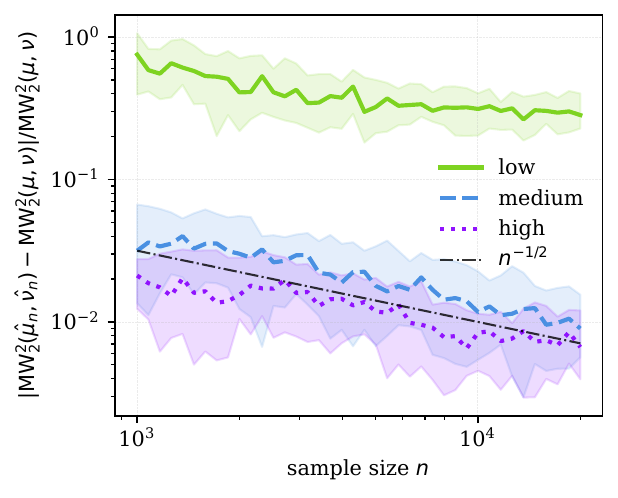}
    \caption{Two-sample: relative error}
    \label{fig:mw2-pairs}
    \end{subfigure}
    \hspace*{\fill}
    \caption{Sample complexity of $\MW_2^2$ for three specific GMMs of ``low''
    to ``high'' mode separation. Curves show the median across $40$ repetitions,
    with shaded interquartile ranges.}
    \label{fig:mw2_sample_complexity}
\end{figure}

\subsection{Minimisation of \texorpdfstring{$\EM-\MW_2^2$}{EM-MW2}: Local Optima
and Weight Fixing}\label{sec:fixweights}

In practice, the minimisation of the energy $X \longmapsto
\MW_2^2(F_X^T(\theta_0), \nu)$ for some initialisation $\theta_0 \in \Theta$ and
a target GMM $\nu$ comes with numerous challenges. The first hurdle is the
``outer'' minimisation of the $\MW_2^2$ cost. To illustrate this difficulty, we
begin with a study of the simpler energy $\mu \longmapsto \W_2^2(\mu, \nu)$ for
a fixed (discrete) measure $\nu \in \mathcal{P}_2(\R^d)$ with respect to the
weights and support of the discrete measure $\mu = \sum_{i=1}^na_i\delta_{x_i}$.
This setting corresponds to the optimisation of $\MW_2^2$ with known
covariances, and thus highlights practical bottlenecks for the minimisation of
the complete energy at stake, $\EM-\MW_2^2$. The objective of this section is to
provide a theoretical rationale for fixing the mixture weights in practical
applications, which is to say using \cref{alg:EM_fw} instead of standard EM
(\cref{alg:EM}).

\paragraph{Local Minima of the Discrete 2-Wasserstein Distance} We focus on a
particular instance of the minimisation of $\mu \longmapsto \W_2^2(\mu, \nu)$,
and show the existence of a strict local minimum. We parametrise a discrete
measure $\mu \in \mathcal{P}(\R)$ with a support size of 3 as follows:
$$\forall \alpha \in [-\tfrac{1}{6}, \tfrac{1}{6}]^2,\; \forall \eta \in
(-\tfrac{1}{2}, \tfrac{1}{2})^3,\; \mu_{\alpha, \eta} := (\tfrac{1}{6} +
\alpha_1)\delta_{\eta_1} + (\tfrac{1}{6} + \alpha_2)\delta_{\eta_2} +
(\tfrac{2}{3} - \alpha_1 - \alpha_2)\delta_{1+\eta_3}, $$ and we fix a target
measure $\nu := \frac{1}{3}(\delta_0 + \delta_{1-\varepsilon} +
\delta_{1+\varepsilon})$ for a fixed $\varepsilon \in (0, \frac{1}{2})$. The
energy to minimise is then:
\begin{equation}\label{eqn:W2_local_minimum_energy}
    \Ethree := \app{[-\frac{1}{6}, \frac{1}{6}]^2\times 
    (-\frac{1}{2}, \frac{1}{2})^3}{\R}{(\alpha, \eta)}
    {\W_2^2(\mu_{\alpha, \eta}, \nu)}.
\end{equation}
Obviously, the energy $(\alpha, \eta) \longmapsto \W_2^2(\mu_{\alpha, \eta},
\nu)$ has a global minimum with value 0 at all $(\alpha, \eta)$ such that
$\mu_{\alpha, \eta} = \nu$. However, on the region with $(\alpha, \eta) \in
[-\tfrac{1}{6}, \tfrac{1}{6}]^2 \times (-\tfrac{1}{2}, \tfrac{1}{2})^3$, we show
in \cref{sec:local_minima_w2_computations} that the energy $\Ethree$ has a
minimum at $\alpha = 0_{\R^2}$ and $\eta = 0_{\R^3}$, with value
$\Ethree(0_{\R^2}, 0_{\R^3}) > 0$. Note that for the case $n=2$ we can show that
there is a unique local minimum, see \cref{sec:unique_local_min_w2_n2}.

\paragraph{Essential Stationary Points for the
\texorpdfstring{$\EM-\MW_2^2$}{EM-MW2}
Loss} We have seen in the previous paragraph that optimising $\mu \longmapsto
\W_2^2(\mu, \nu)$ with respect to the weights and support of $\mu$ can lead to
local minima, which are not global minima. An additional difficulty arises when
optimising the energy
\begin{equation}\label{eqn:MW2_EM_loss}
    \EEMMW:= X \longmapsto \MW_2^2(\mu(F(\theta, X)), \nu),
\end{equation}
with one iteration $F$ of the EM algorithm, due to the update on the weights.
The issue is that at some problematic points $X$ to which the algorithm often
converges in practice, the gradient $\partial_X \EEMMW(X)$ becomes extremely
small, in particular when the covariances are highly localised. This leads in
practice to undesirable convergence to an essential local minimum, as
illustrated by an example in \cref{sec:xp_em_mw_flow_weights}. We provide a
theoretical explanation in a simple case in
\cref{sec:vanishing_gradients_EM_MW_computations}. To avoid these numerical
issues, we propose fixing the weights of the GMMs in the EM steps by using
\cref{alg:EM_fw}.

Our theoretical observations suggest that considering GMMs with uniform weights
and using fixed-weights EM (\cref{alg:EM_fw}) is a more stable alternative to
standard EM (\cref{alg:EM}). In practice, we believe it is also preferable to
keep the same number of components $K$ between the compared GMMs for additional
stability. Note that an identifiability issue remains with GMMs: if the means
and covariances of two modes coincide, then the GMM can also be written by
fusing both components and adding their weights. At this stage, it remains
unclear whether this phenomenon has an impact on the optimisation behaviour
(note that we never observed it in our experiments).

\subsection{Unbalanced GMM-OT}\label{sec:unbalanced}

Starting from the discrete formulation of the $\MW_2$ distance, we relax the
constraints on the transport plan $\pi$, penalising the marginal conditions
instead of enforcing them in the optimisation problem. The resulting
optimisation problem defines an unbalanced GMM-OT distance on the set
$\GMM_d^+(\infty)$ of GMMs with positive weights on $\R^d$, as in
\cite{liero2018optimal}. Given two Gaussian mixtures $ \mu =
\sum_{k_0=1}^{K_0}w_{k_0}^{(0)} g_{k_0},\quad \nu =
\sum_{k_1=1}^{K_1}w_{k_1}^{(1)} g_{k_1}\; \in\GMM_d^+(\infty),$ and
regularisation parameters $(\lambda_0, \lambda_1) \in (0, +\infty)^2$, the
unbalanced GMM-OT cost is defined as:
\begin{equation}\label{eqn:unbalanced_gmmot}
    \UMW_2^2(\mu, \nu; \lambda_0, \lambda_1) 
    := \underset{\pi \in \R_+^{K_0\times K_1}}{\min}\ 
    \sum_{k_0, k_1} \pi_{k_0, k_1} \W_2^2(g_{k_0}, g_{k_1})
    + \lambda_0 \KL(\pi\mathbf{1} | w_0)
    + \lambda_1 \KL(\pi^\top\mathbf{1} | w_1),
\end{equation}
where we recall that for $a, b \in (0, +\infty)^K$, the Kullback-Leibler
divergence is $\KL(a|b) := \sum_k a_k \log(\tfrac{a_k}{b_k})$. We have seen that
$\MW_2$ is a particular discrete Kantorovich problem, and likewise the
unbalanced GMM-OT distance $\UMW_2$ is simply an unbalanced discrete OT problem
with a particular cost matrix.

Given the numerical challenges of optimising the weights in the balanced
formulation (see the discussion in \cref{sec:fixweights}), we introduce this
variant as a possibly more stable alternative. We suspect that the underlying
geometry on the weights induced by unbalanced OT
\cite{liero2018optimal,chizat2018unbalanced} is more amenable to optimisation\sammod{: we compare balanced and unbalanced flows in
\cref{sec:unbalanced_weight_comparison} of the supplementary material.}
Furthermore, unbalanced OT has been shown by \cite{fatras2021unbalanced} to be
stable with respect to minibatch sampling, which is paramount for large-scale
machine learning applications.
\section{Illustrations and Quantitative Study of Gradient
Methods}\label{sec:diff_em_illustrations}
\subsection{Practical Implementation}

For practical implementation of the EM algorithm, some specific implementation
strategies are required to ensure numerical stability, in particular when
computing gradients. The first technique is applied in the E step, and consists
in computing the responsibilities $\gamma_{ik}(\theta_t)$ in logarithmic space
and using the so-called ``log-sum-exp trick''\footnote{for instance, see
\href{https://gregorygundersen.com/blog/2020/02/09/log-sum-exp/}{this blog post}
for an explanation of this well-known trick.} to compute the normalisation in
\cref{eqn:gamma_ik}. Furthermore, to stabilise the (differentiable) expression
of the Gaussian density $g_{m, \Sigma}(x)$, we leverage the Cholesky
decomposition of the covariance matrix $\Sigma$, which uniquely decomposes
$\Sigma = LL^\top$ where $L\in \R^{d\times d}$ is a lower triangular matrix. In
particular, the computation of the inverse is simplified by solving triangular
systems, and the determinant of $\Sigma$ is simply $\det \Sigma = (\prod_a
L_{aa})^2$ (which we compute in logarithmic space as well).

Another important implementation aspect concerns a differentiable implementation
of the matrix square root of symmetric positive semidefinite matrices. This is
required in the computation of the Bures distance \cref{eqn:bures_wasserstein}
for the $\MW_2$ distance. Unfortunately, the naive implementation using the
spectral decomposition suffers from numerical instability when eigenvalues
are very close in value\footnote{as explained in the
\href{https://docs.pytorch.org/docs/stable/generated/torch.linalg.eigh.html}{PyTorch
documentation} for \texttt{torch.linalg.eigh()}.}. Leveraging an explicit
formula for the gradient of the matrix square root (detailed in
\cref{sec:diff_matrix_sqrt}), we circumvent these numerical issues by
implementing our own differentiable square root function with an explicit
gradient.

As is done in
\href{https://scikit-learn.org/stable/modules/mixture.html}{scikit-learn's
implementation} of the EM algorithm, we have an optional regularisation term
$\varepsilon_r \geq 0$ for the covariance matrices $\Sigma_k$ to ensure
positive-definiteness. The idea is to replace the update $\Sigma_k^{(t+1)}$ with
$\Sigma_k^{(t+1)} + \varepsilon_r I_d$ to enforce a minimum eigenvalue of
$\varepsilon_r$. This regularisation was particularly crucial for numerical
stability in higher-dimensional cases where the covariances were almost
singular, which led to exploding gradients. In the larger-scale examples from
\cref{sec:diff_em_applications}, we chose a heuristic which sets $\varepsilon_r
:= 10^{-4} \times s_{\max}$, where $s_{\max}$ is the largest eigenvalue of the
covariances of a GMM fitted on the target data.

\subsection{Flow of \texorpdfstring{$\mathrm{EM}-\MW_2^2$ with Fixed Weights}{EM
MW2} in 2D}\label{sec:xp_em_mw_flow_fixed_weights}

In this section, we illustrate the use of differentiable EM for OT by
numerically computing the flow (i.e. gradient descent) of the following energy:
\begin{equation}
    \calE_T := X \in \R^{n\times 2} \longmapsto 
    \MW_2^2(\mu(\FEM^T_X(\theta_0)), \nu),
\end{equation}
for a fixed target GMM $\nu$, an initialisation $\theta_0$ and a number of EM
steps $T$. We use a variant of EM presented in \cref{alg:EM_fw} that
\textbf{fixes the mixture weights} in this experiment. We will compare three
gradient computation methods to compute (or approximate) the gradient of
$\FEM^T_X(\theta_0)$ with respect to $X$, within the gradient descent of
$\calE_T$, performed using automatic differentiation. The setup is as follows:
the initial dataset $X \in \R^{200\times 2}$ corresponds to samples of a GMM
$\mu_0$ with 3 components, and we want to displace this point cloud to match a
target GMM $\nu$ with 3 components. We represent the setup in
\cref{fig:xp_em_mw_flow_setup} and the flow for AD method in
\cref{fig:xp_em_mw_flow_full_auto}. 

The results for the AI method are both visually and quantitatively very close
\sammod{to AD}. We observe satisfactory convergence of the flow of $\calE_T$
towards the target GMM $\nu$. In many applications involving fixed EM weights
(\cref{alg:EM_fw}), we observe that particles follow rectilinear trajectories,
which is a similar behaviour to Wasserstein flows of $\W_2^2$ (see \cite[Section
5.3]{chewi2024statistical}). We interpret this phenomenon as a consequence of
the fixed weights, which translate to a Lagrangian viewpoint on the GMMs. In
simple cases, the $\MW_2^2$-optimal plans between the GMMs may not change during
the flow, and thus the particles are moved along the induced (rectilinear)
trajectories between each GMM component (see \cite[Proposition
4]{delon2020wasserstein}). In \cref{fig:xp_em_mw_flow_OS}, we show the flow with
the OS method, which \sammod{yields almost identical results to AD}. We also
experimented with the Warm-Start flow from \cref{alg:warm_start_EM_flow}, which
is a different minimisation method to minimise $\calE_T$, yet yielded almost
identical \sammod{final losses} to the AD. \sammod{In
\cref{sec:warm_start_multidimensional}, we investigate in further detail the
impact of warm-starting the AD method, concluding that the results are
comparable and become equal when $T\approx 10$.}

\begin{figure}[ht]
    \centering
    \begin{subfigure}[t]{0.32\textwidth}
        \centering
        \includegraphics[width=\textwidth]{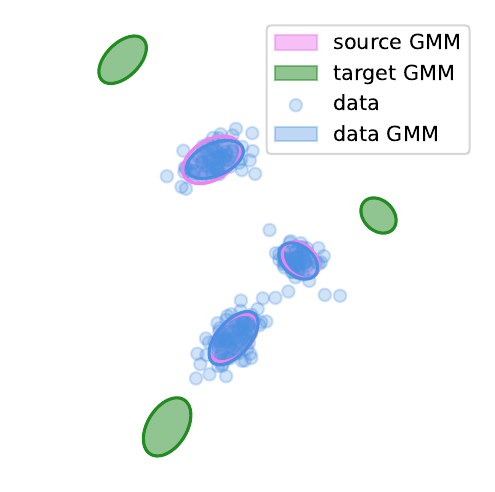}
        \caption{$\calE_T$ flow setup.}
        \label{fig:xp_em_mw_flow_setup}
    \end{subfigure}
    \hfill
    \begin{subfigure}[t]{0.32\textwidth}
        \centering
        \includegraphics[width=\textwidth]{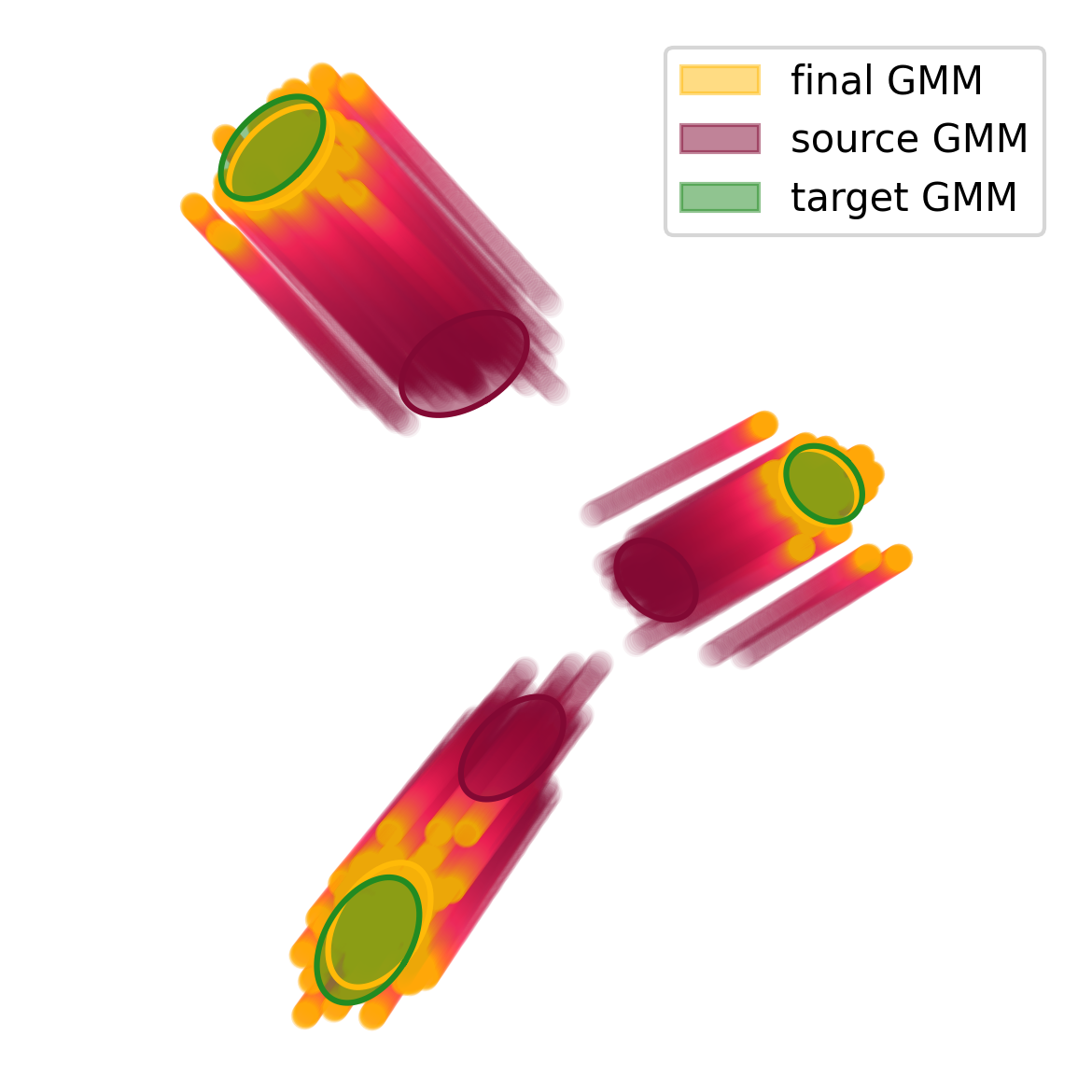}
        \caption{AD method.}
        \label{fig:xp_em_mw_flow_full_auto}
    \end{subfigure}
    \hfill
    \begin{subfigure}[t]{0.32\textwidth}
        \centering
        \includegraphics[width=\textwidth]{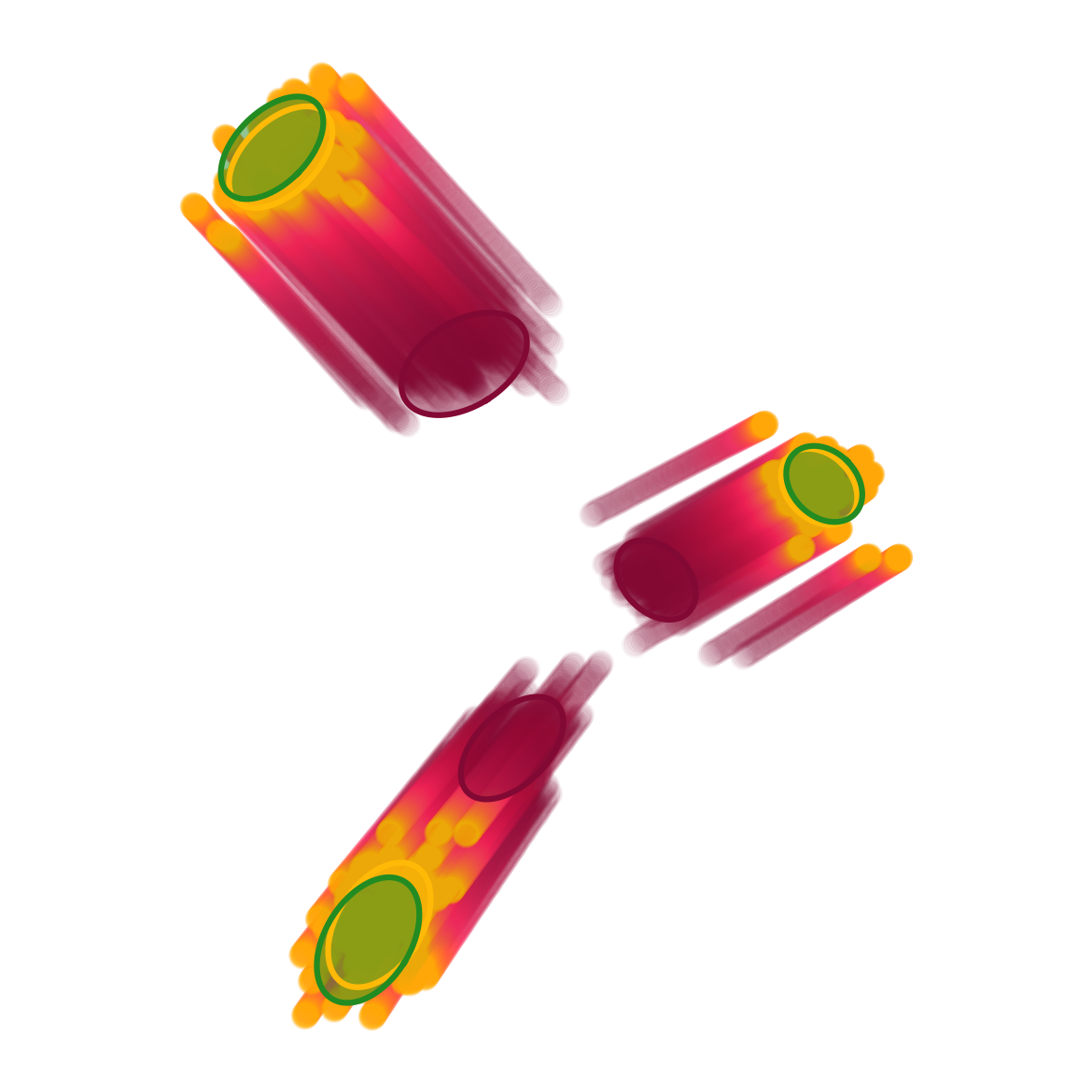}
        \caption{OS method.}
        \label{fig:xp_em_mw_flow_OS}
    \end{subfigure}
    \caption{Comparison of experimental setup and flows of $\calE_T$ using
    different methods. The dark shades of purple correspond to earlier
    iterations, and the yellow shades to later iterations. (Warm-Start and AI
    are almost identical to AD)}
\end{figure}

\subsection{Flow of \texorpdfstring{$\mathrm{EM}-\MW_2^2$}{EM MW2} in 2D: Discussion on Uniform Weights}\label{sec:xp_em_mw_flow_weights}

We now consider a similar setting to \cref{sec:xp_em_mw_flow_fixed_weights}, but
without fixing the weights in the EM algorithm, i.e. using standard EM
\cref{alg:EM}. We compare two settings: the first with non-uniform weights $w_0
:= (\tfrac{1}{5}, \tfrac{1}{5}, \tfrac{3}{5})$ for the initial GMM, and weights
$w_1 := (\tfrac{1}{2}, \tfrac{3}{10}, \tfrac{1}{5})$ for the target GMM; and the
second with uniform weights for both. In \cref{fig:EM_GD_full_auto_weights}, we
show the flow of $\calE_T$ with the AD method. We observe in
\cref{fig:non_uniform_weights_EM_flow} that the flow for non-uniform weights
converges to an unsatisfactory local minimum, with a final GMM weight of
\texttt{[0.38990256, 0.2000008, 0.41009664]} instead of the target \texttt{[0.5,
0.3, 0.2]}, as shown on the simplex in
\cref{fig:non_uniform_weights_EM_simplex}. In contrast, the flow for uniform
weights presented in \cref{fig:uniform_weights_EM_flow} converges to the target
GMM and achieves a substantially lower energy, as reported in
\cref{fig:unif_vs_non_unif_EM_loss}. The weights stay close to uniform, with a
final GMM weight of \sammod{\texttt{[0.33700094, 0.31099346, 0.35200560]}}.

The optimisation failure in the non-uniform case \sammod{is consistent with} the
essential stationary point problem illustrated in \cref{sec:fixweights}:
intuitively, to change the weights of the current GMM, the particles need to
change components, but this is \sammod{difficult} if the components are too
distant. While our framework encompasses the case of non-uniform weights, as
illustrated theoretically and experimentally, it appears that the non-uniform
weight setting is impractical. As a result, we recommend using uniform weights,
in particular using the fixed-weights EM approach (\cref{alg:EM_fw}) for speed
and stability. \sammod{Assuming fixed uniform weights in GMMs is a very
classical assumption in the machine learning literature, made for instance in
all the
recent~\cite{kim2022differentiable,lambert2022variational,huix2024theoretical,petit2026variational}.
Fixed weights do not need to be uniform and can be fixed for instance to the
target GMM weights. We investigate further the local minima that arise in this
case in \cref{sec:fixed_target_weights}.}

\begin{figure}[ht]
    \centering
    \begin{subfigure}[t]{0.24\textwidth}
        \centering
        \includegraphics[width=\textwidth]{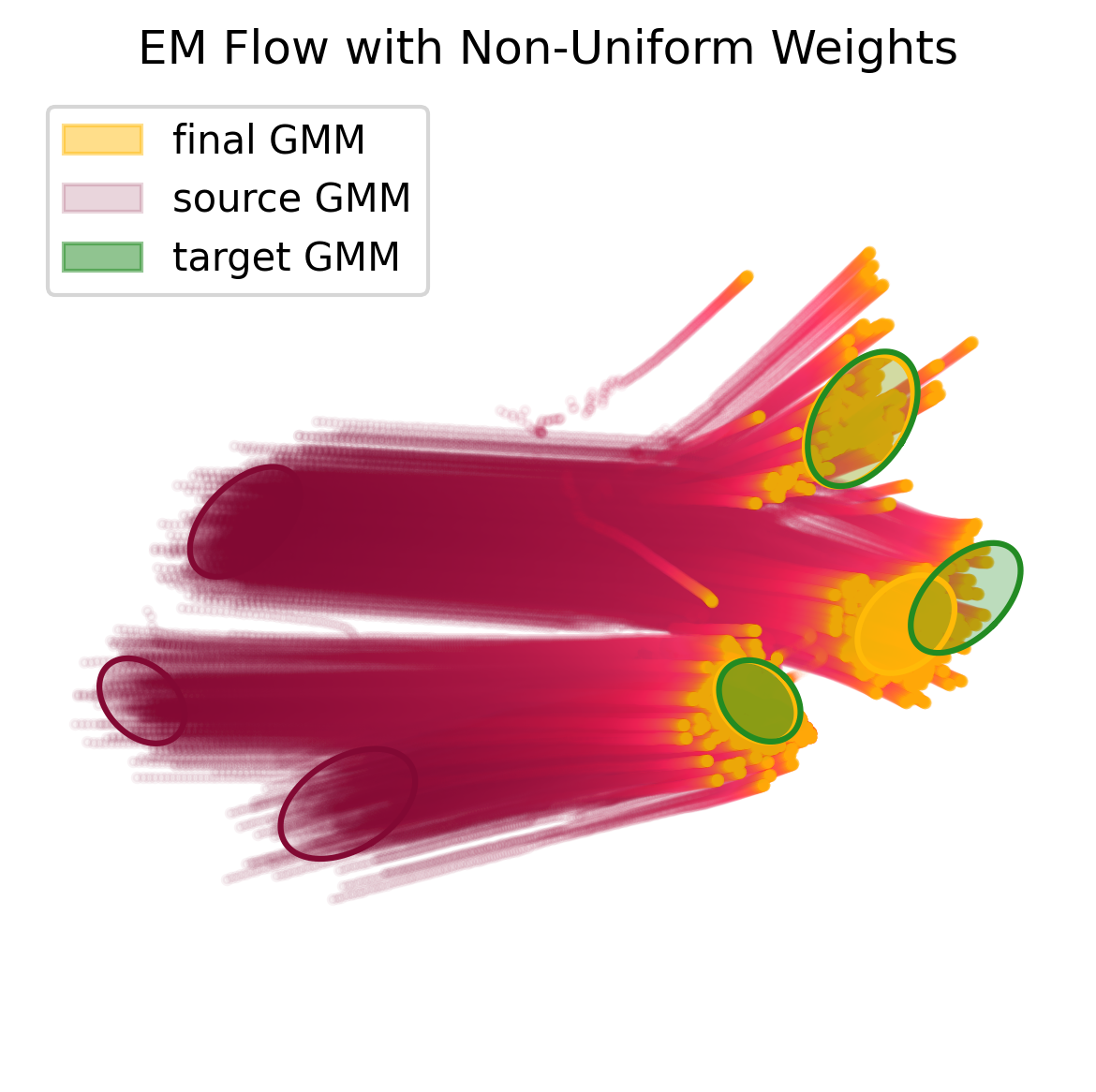}
        \caption{Particle flow (NU).}
        \label{fig:non_uniform_weights_EM_flow}
    \end{subfigure}
    \hfill
    \begin{subfigure}[t]{0.24\textwidth}
        \centering
        \includegraphics[width=\textwidth]{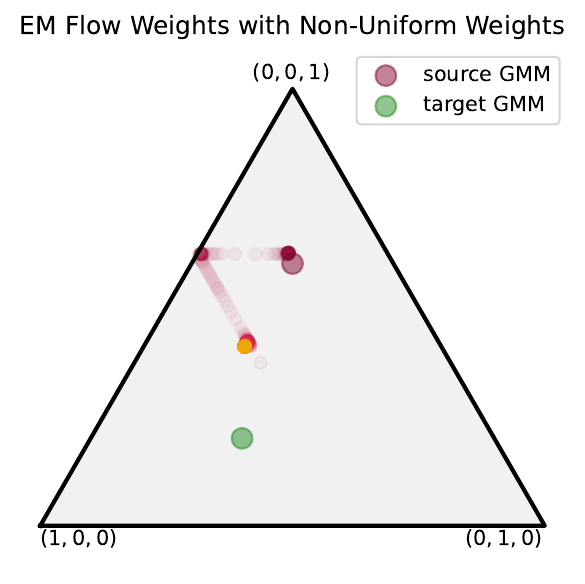}
        \caption{Weight evolution on the simplex $\simplex{3}$ (NU).}
        \label{fig:non_uniform_weights_EM_simplex}
    \end{subfigure}
    \hfill
    \begin{subfigure}[t]{0.24\textwidth}
        \centering
        \includegraphics[width=\textwidth]{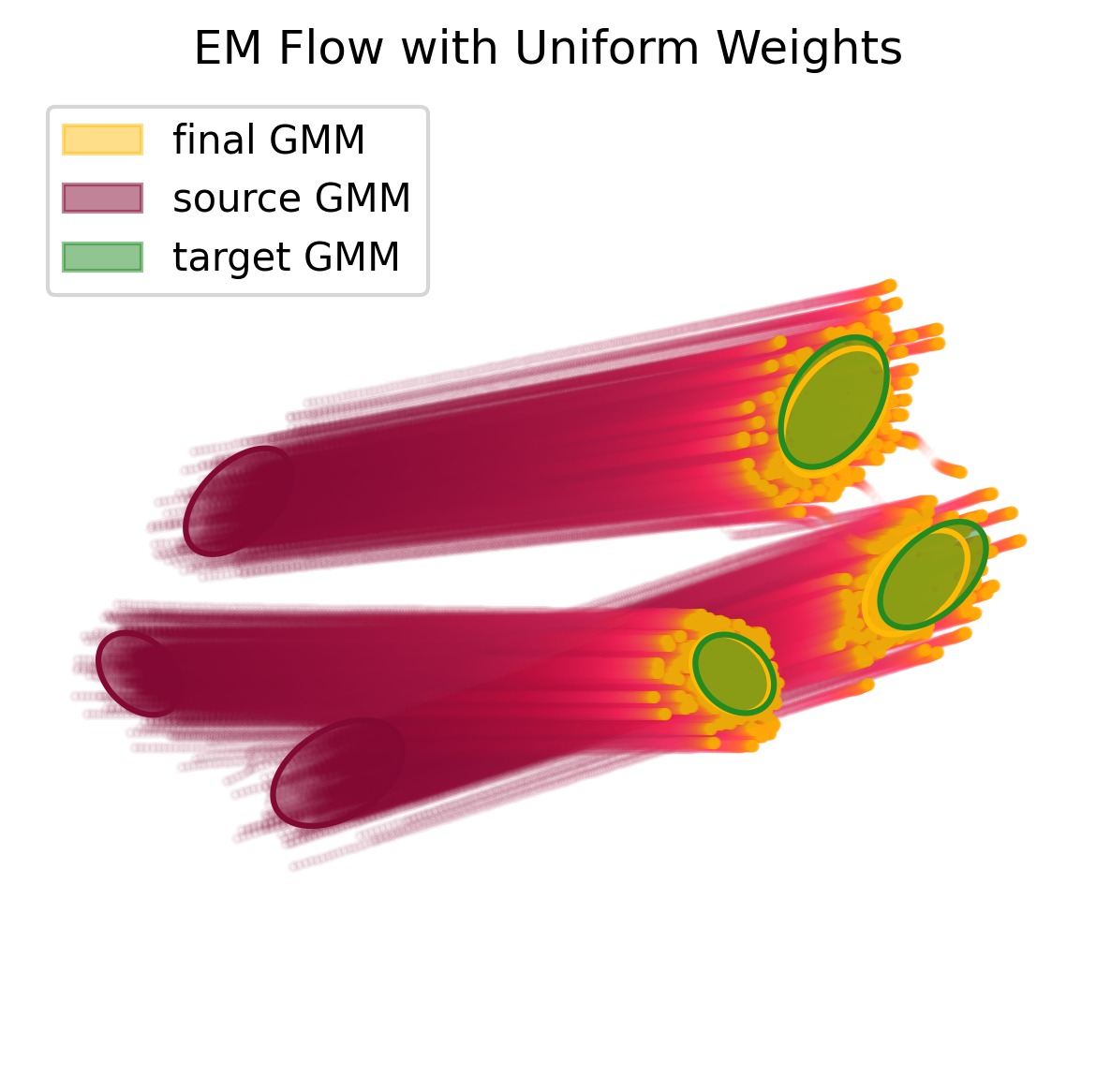}
        \caption{Particle flow (U).}
        \label{fig:uniform_weights_EM_flow}
    \end{subfigure}
    \hfill
    \begin{subfigure}[t]{0.24\textwidth}
        \centering
        \includegraphics[width=\textwidth]{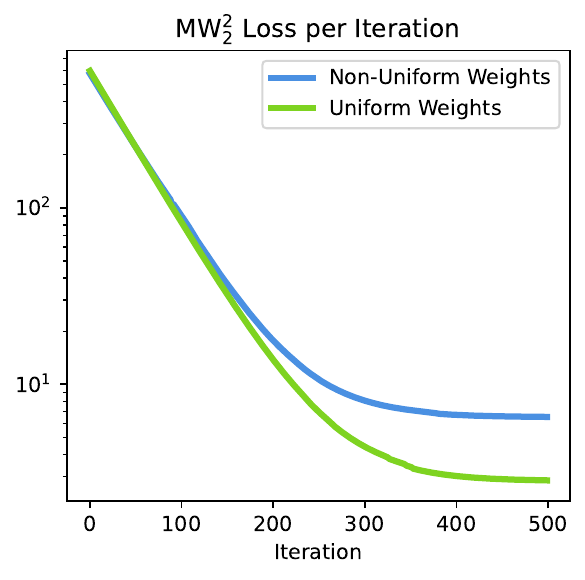}
        \caption{Energy evolutions.}
        \label{fig:unif_vs_non_unif_EM_loss}
    \end{subfigure}
    \caption{Flow of $\calE_T$ with the AD method and standard EM \cref{alg:EM}.
    We compare two settings: one with non-uniform GMM weights (NU) and one with
    uniform weights (U).}
    \label{fig:EM_GD_full_auto_weights}
\end{figure}

\subsection{Stochastic \texorpdfstring{$\mathrm{EM}-\MW_2^2$}{EM MW2} Flow with
Fixed Weights}\label{sec:xp_em_mw_sgd}

We consider a similar setting to \cref{sec:xp_em_mw_flow_fixed_weights} but
introduce stochasticity in the flow at each step by performing EM only on a
subsample of the optimised source point cloud and of the target point cloud.
While we illustrate the technique on a toy example here, this ``minibatch''
stochastic gradient descent method is useful in practice when the dataset size
is too large for simultaneous optimisation
\cite{fatras2020learning,fatras2021minibatch,fatras2021unbalanced,tong2024improving}.
The same principle is applied to the image generation task in
\cref{sec:em_mw2_gan}. We observe in \cref{fig:EM_SGD} that the general
trajectory remains similar to the deterministic case. Notice that in this
setting, the components are sufficiently close together to interact, yielding
non-rectilinear trajectories when points are influenced by multiple components.
This is amplified by the stochasticity of the method.

\begin{figure}[ht]
    \centering
    \includegraphics[width=\textwidth]{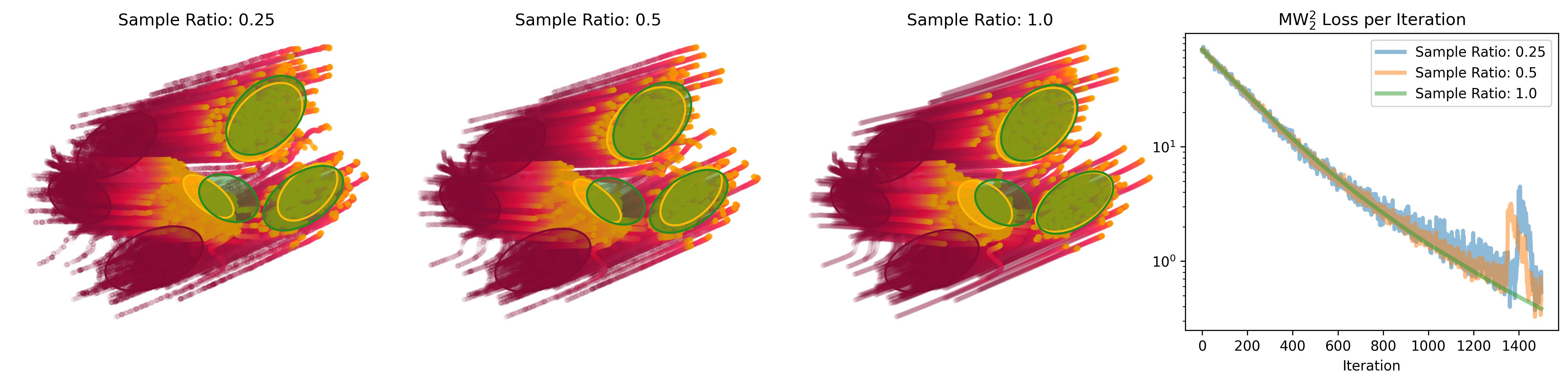}
    \caption{Stochastic Flow of $\calE_T$ for \cref{alg:EM_fw} with the full
    automatic differentiation method. We vary the sub-sampling ratio $r \in
    (0,1]$, which corresponds to performing EM on only $[r\times n]$ random
    points from the current point cloud at each step.}
    \label{fig:EM_SGD}
\end{figure}
\subsection{Quantitative Study of EM Convergence and
Gradients}\label{sec:xp_gradients_comparison}

We study the impact of the number of points $n$, components $K$ and EM
iterations $T$ on the convergence of EM iterations (to a fixed point of $\FEM$),
the local contractivity of $\FEM$ around the fixed point, and the gradient
approximation methods introduced in \cref{sec:em_grad_methods}. The experimental
setting is as follows: for each of the three parameters $n, K, T$ separately
(say $n$), we consider a range of values (say $n \in \{100, 200, 500, 1000,
1500, 2000\}$) with the others fixed. For each value of the parameter in this
range, and for three different GMMs in $\GMM_d(K)$ with $d=3$ (shown in
\cref{sec:grad_comparison_gmms}), we sample the data from $X \sim \mu_0^{\otimes
n}$ 60 times, then run the EM algorithm for $T$ iterations. We run the
experiments on three different GMMs taken with random parameters (adding a small
term $10^{-14} I_d$ to the covariances to avoid vanishing eigenvalues). Note
that GMM$\#$ 1 is better conditioned than GMM$\#$ 2, and GMM$\#$ 3 is the
worst-conditioned. We observe the mean squared error of the fixed point property
$F(\theta_T, X)\approx \theta_T$ by evaluating $\tfrac{1}{p}\|\theta_T -
F(\theta_T, X)\|_2^2$ with $p := K + Kd + Kd^2$, measuring the quality of
convergence of the EM algorithm. To study the local contractivity of $\FEM$, we
compute the spectral norm $\opn{\partial_\theta \FEM(\theta_T, X)}$: if this is
close to 0, then locally the iterated function $\FEM_X$ has a tame behaviour and
the OS method is expected to work well, while if it is close to or larger than
1, the local landscape is difficult and the OS method is expected to fail.
Finally, we compare the OS and AI gradients (from \cref{eqn:def_OS,eqn:def_AI})
to the reference AD gradient by computing the relative MSEs $\tfrac{1}{p}\|\JOS
- \JAD\|_2^2 / (\tfrac{1}{p}\|\JAD\|_2^2)$ and $\tfrac{1}{p}\|\JAI - \JAD\|_2^2
/ (\tfrac{1}{p}\|\JAD\|_2^2)$, where $\JAD$ is the AD gradient, which serves as
a baseline (see \cref{sec:em_grad_gt} for a discussion on this choice).
Concerning the impact of the number of components $K$, we defer to
\cref{sec:impact_K_grad_comp}, since the findings are less conclusive.

\paragraph{Impact of the number of samples $n$} We begin by fixing $d=3$, $K=3$
and $T=30$ and varying $n \in \{100, 200, 500, 1000, 1500, 2000\}$. The results
are shown in row 1 of \cref{fig:grad_comparison_fixed_gmm}, and we observe that
EM appears to converge to a fixed point for all $n$, albeit with a large
variance in the MSE depending on the sampled GMMs. The spectral norm of the
Jacobian is often close to 0.6 and has no clear trend with $n$, hence we expect
the OS method to be a very coarse approximation of the true gradient. The
quality of the OS gradient is relatively poor, and substantially worse than the
AI gradient, whose median MSE is much smaller, but suffers from very high
variance (in log space). Comparing GMMs shows that a precise EM convergence
leads to high precision for the AI gradient. \paragraph{Impact of the number of
EM iterations $T$} Finally, we fix $n=200, d=3$ and $K=3$, and vary $T \in \{1,
2, 5, 10, 15, 20, 30, 40\}$ in row 3 of \cref{fig:grad_comparison_fixed_gmm}.
Reassuringly, increasing the number of iterations $T$ leads to improved
convergence of the EM algorithm to better-conditioned points. The convergence
speed seems heavily dependent on the GMM, with an additional variance caused by
the dataset sampling. In the favourable settings for larger $T$, the AI
approximation substantially outperforms the OS approximation. Since the spectral
norm of the Jacobian stabilises to values of the order of 0.5, the OS method
plateaus at coarse MSEs, even for larger $T$.

\begin{figure}[ht]
    \centering
    \includegraphics[width=\textwidth]{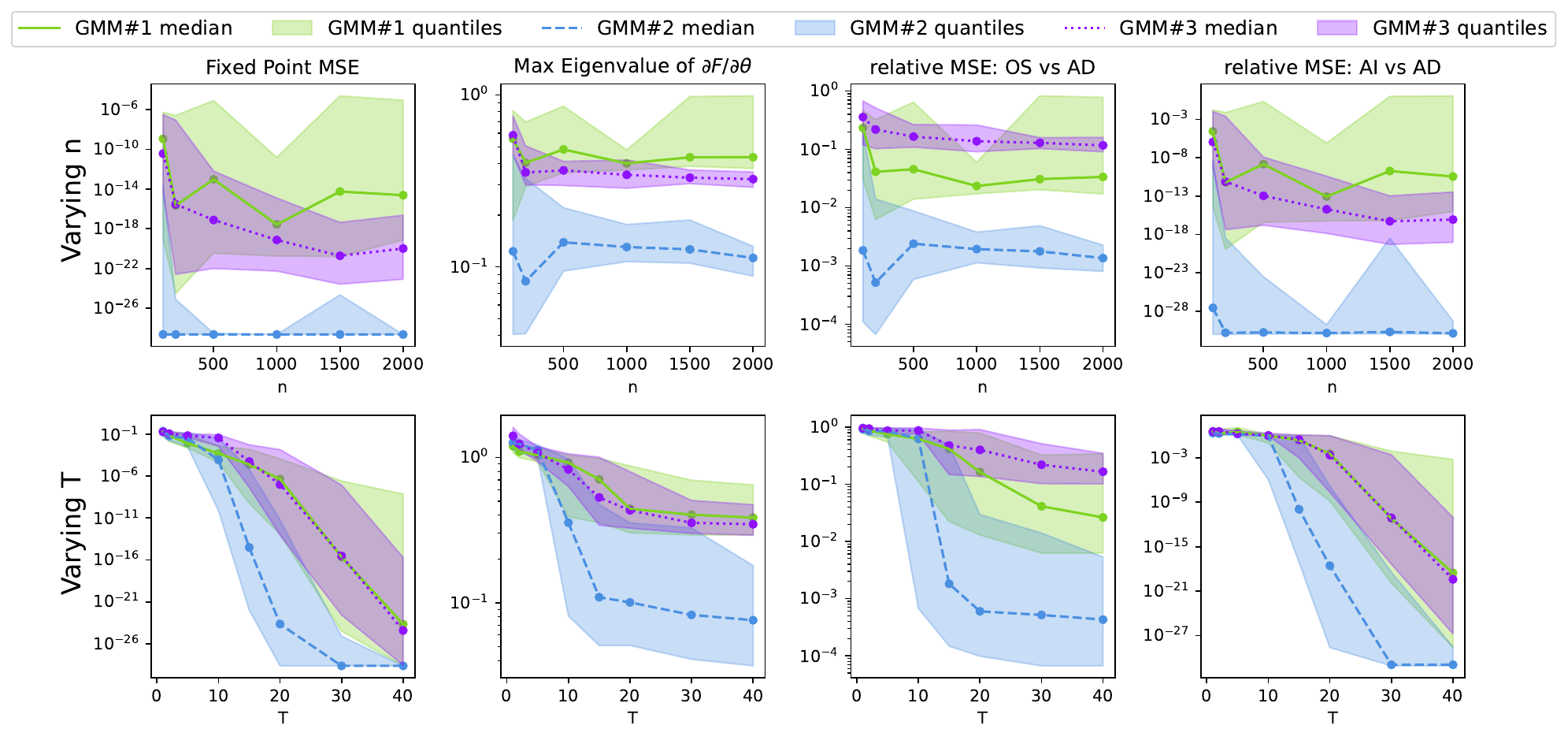}
    \caption{Varying the number of samples $n$ and the number of iterations $T$,
    we study the convergence of EM, the local contractivity of $\FEM$, and the
    MSEs of the OS and AI gradients against the AD gradient.}
    \label{fig:grad_comparison_fixed_gmm}
\end{figure}

\section{Applications of Differentiable EM}\label{sec:diff_em_applications}
In this section, we propose numerous larger-scale applications of the
differentiation of the EM algorithm presented in \cref{sec:diff_em}. Our goal is
to illustrate the versatility of the proposed approach, rather than to achieve
state-of-the-art results on these applications. In \cref{sec:em_mw2_gan}, we
also present an EM-$\MW_2^2$-regularised generative model.
\subsection{Barycentre Flow in 2D}\label{sec:bary_flow}

Wasserstein barycentres \cite{agueh2011barycenter} and their notoriously
challenging computation
\cite{cuturi14fast,alvarez2016fixed,wb_np_hard,tanguy2024computing} are active
fields of research. In this section, we illustrate the use of differentiable EM
to flow a point cloud towards a barycentre of GMMs. Given $M$ point clouds $Y_i
\in \mathbb{R}^{n\times 2}$, our goal is to optimise a point cloud $X\in
\mathbb{R}^{n\times 2}$, initialised as random normal noise, towards a
barycentre (with uniform weights) of GMMs $(\nu_i)$ fitted from $(Y_i)$.
Specifically, we solve
\[
\min_{X\in \mathbb{R}^{n\times 2}} \sum_{i=1}^M 
\MW_2^2\left(\mu(F_X^T(\theta_0)), \nu_i\right)
\]
with respect to $X$. The GMMs $\nu_i$ are fitted beforehand, and
$\mu(F_X^T(\theta_0))$ is the current EM estimate of the optimised cloud $X$.
We illustrate the results in \cref{fig:bary}, where $M:=3,\; K:=2$ and
$n:=500$. This method can be adapted to compute more general barycentres, as
presented in \cref{app:bary}.

\begin{figure}[htb]
    \centering
    \includegraphics[width=0.65\textwidth, trim=0 60 0 60, clip]{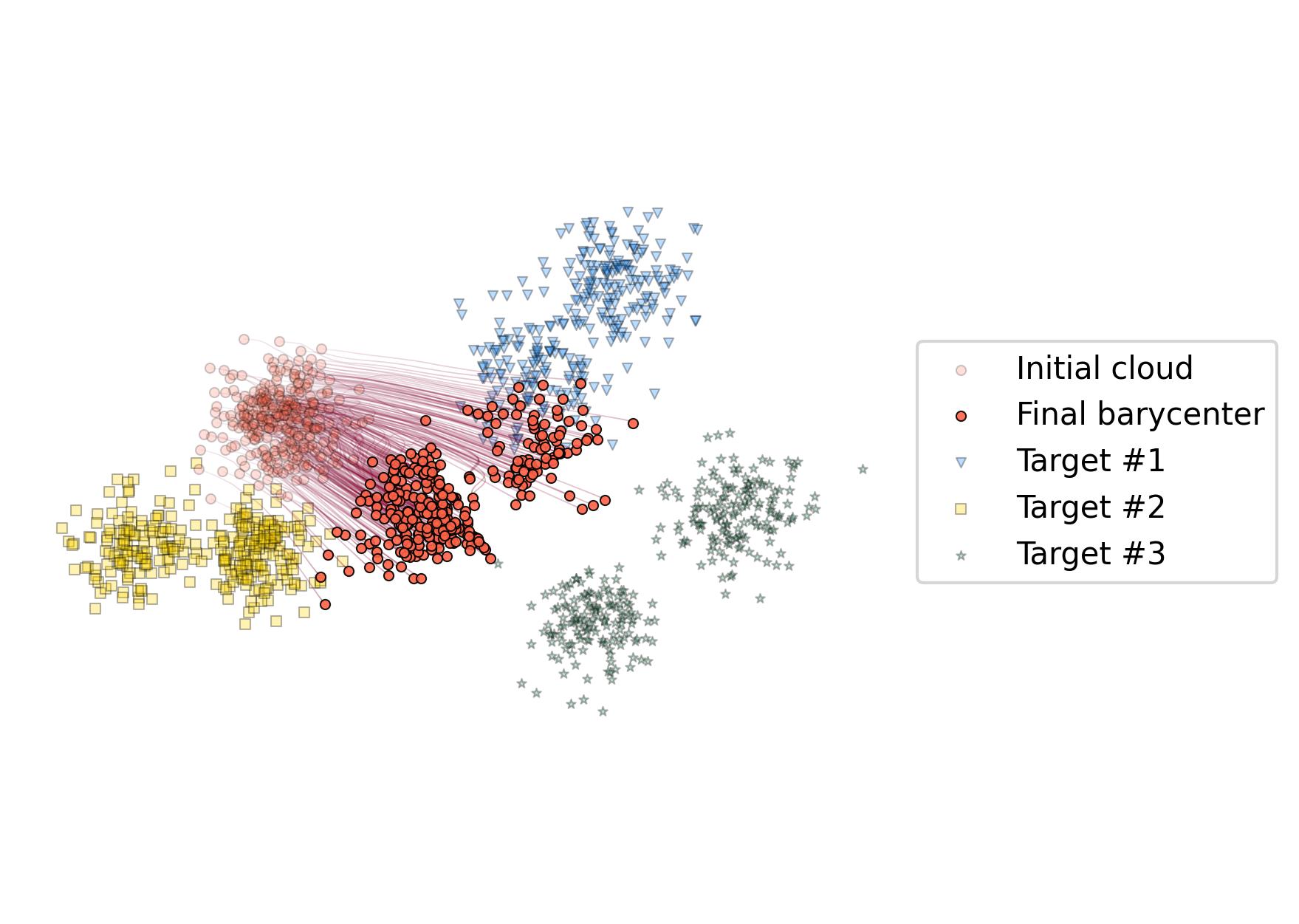}
    \caption{$\mathrm{EM}-\MW_2^2$ flow displacing particles in order to make
    their EM output approach a barycentre of three target GMMs.}
    \label{fig:bary}
\end{figure}

\subsection{Colour Transfer}\label{subsec:color_transfer}

Colour transfer is a well-known imaging task where OT techniques have been used
extensively
\cite{reinhard2002color,delon2004midway,pitie2007linear,pitie2007automated,papadakis2010variational,rabin2012wasserstein,rabin2014adaptive},
it consists in transforming the RGB colour distribution of a source image to
match that of a target image. We propose the following approach: we initialise
an image $X$ as the source image, and optimise it to minimise the $\MW_2^2$ cost
between a GMM fitted on $X$ (seen as an RGB point cloud), and a target GMM $\nu
\in \GMM_3(K)$ fitted on the colour distribution of the target image.
Specifically, we minimise $X \longmapsto \MW_2^2(\mu(F_X^T(\theta_0)), \nu)$ for
some initialisation $\theta_0$. \textit{Warm-Start EM} method from
\cref{alg:warm_start_EM_flow}, choose $K:=10$ components and use fixed uniform
GMM weights (\cref{alg:EM_fw} for EM) to avoid being trapped in a local minimum,
as seen in \cref{sec:fixweights}. We present some results in
\cref{fig:color_transfer_combined}, and provide additional discussions about the
optimisation choices in \cref{sec:colour_transfer_details}\sammod{, along with a
comparison to a (computationally expensive) Entropic OT alternative}. Even with
only $K=10$ components, the colour transfer preserves both details and global
consistency, and in the resulting colour scheme, the source image appears to
exhibit stronger contrast in this example.

Balanced OT methods may be sensitive to outliers in the distributions, leading
to artifacts in the colour transfer results if colour aberrations are present in
the target. Unbalanced OT methods \cite{liero2018optimal} can mitigate this
issue \cite{chizat2018scaling,bonet2024slicing}. We now consider the unbalanced
variant of the Mixture Loss defined in \cref{sec:unbalanced}: by relaxing the
constraints on marginals, it can ignore outliers in the input distributions.
Specifically, in \cref{fig:color_transfer_combined}, we consider an illustration
with a corrupted target image where a patch of red has been added, and observe
that the unbalanced approach is more robust to this aberration, showing no
propagation of the red artifact.

\begin{figure}[htb]
    \centering
    \noindent\hspace*{\fill}
    \begin{subfigure}{0.22\textwidth}\centering
        {\scriptsize Source}
        \includegraphics[width=\textwidth]{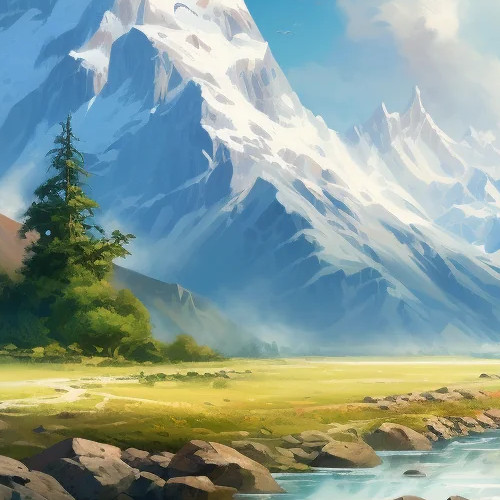}\\[-0.3em]
    \end{subfigure}
    \hfill
    \begin{subfigure}{0.22\textwidth}\centering
        {\scriptsize Result}
        \includegraphics[width=\textwidth]{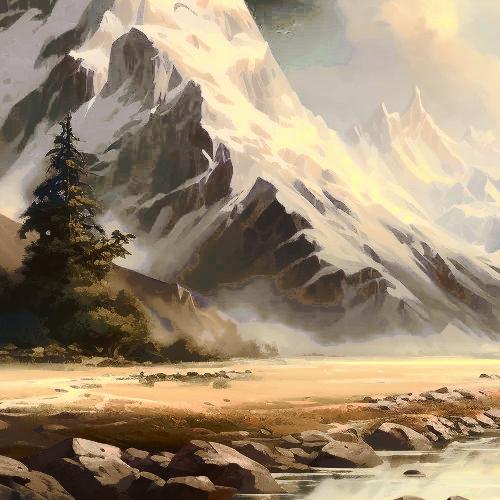}\\[-0.3em]
    \end{subfigure}
    \hfill
    \begin{subfigure}{0.22\textwidth}\centering
        {\scriptsize Target}
        \includegraphics[width=\textwidth]{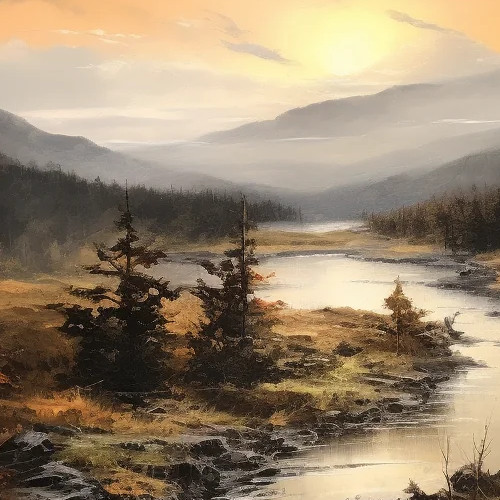}\\[-0.3em]
    \end{subfigure}
    \hspace*{\fill}

    \vspace{0.8em} %

    \noindent\hspace*{\fill}
    \begin{subfigure}{0.22\textwidth}\centering
        \includegraphics[width=\textwidth]{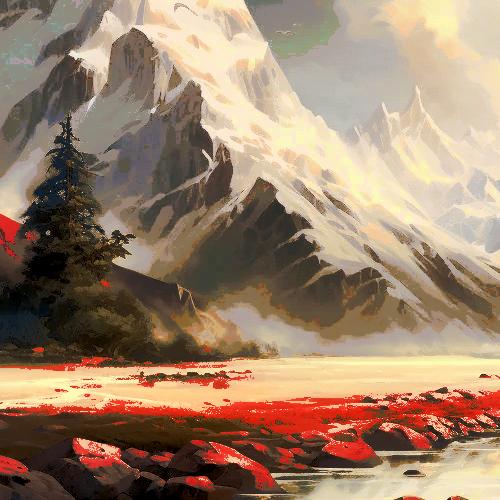}
        \\[-0.5em]
        {\scriptsize Balanced result}
    \end{subfigure}
    \hfill
    \begin{subfigure}{0.22\textwidth}\centering
        \includegraphics[width=\textwidth]{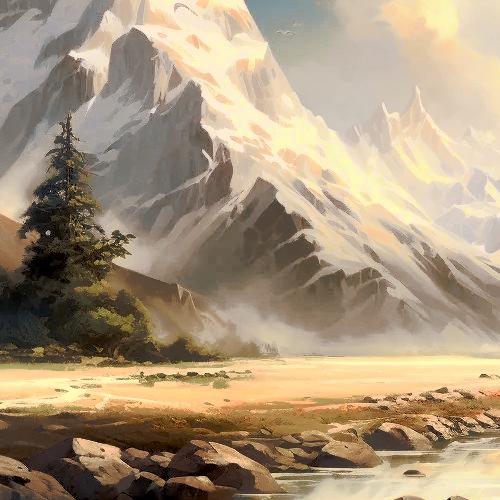}
        \\[-0.5em]
        {\scriptsize Unbalanced result}
    \end{subfigure}
    \hfill
    \begin{subfigure}{0.22\textwidth}\centering
        \includegraphics[width=\textwidth]{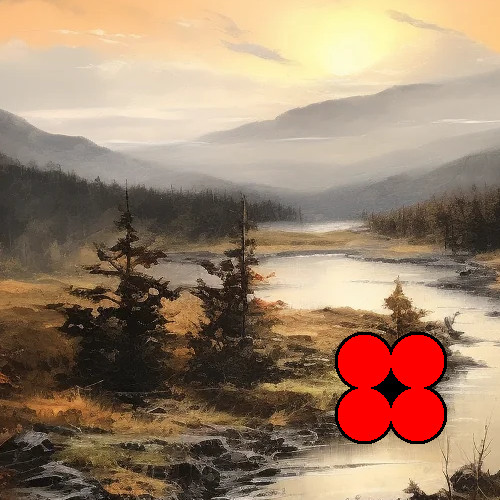}
        \\[-0.5em]
        {\scriptsize Corrupted target}
    \end{subfigure}
    \hspace*{\fill}
    \caption{Colour transfer experiments. Top row: $\EM-\MW_2^2$ from source to
    target. Bottom row: unbalanced colour transfer with regularisations
    $\lambda_1=10$ (source) and $\lambda_2=0.1$ (target).}
    \label{fig:color_transfer_combined}
\end{figure}

\subsection{Neural Style Transfer}

We apply our distance within Gatys et al.'s neural style transfer framework \cite{gatys}. The
goal is to generate an image that combines the content of one image $X_0$ with
the artistic style of another image $Y$. We rely on a pre-trained VGG-19 network
\cite{simonyan2015deep} (see \cref{fig:vgg19_encoder}) to encode our image and
extract relevant features from three specific layers $\ell$. Given an image $X$
in $\mathbb{R}^{3\times H\times W}$, we denote these features
$\text{VGG}_{1\cdots\ell}(X)$. Starting with the content image as $X_0$, we
optimise $X$ such that its features progressively match those of the style image
$Y$. The target mixtures $\overline{\mu}_{\ell}^{\text{style}}$ are fitted at
the beginning of the procedure on style features $\text{VGG}_{1\cdots\ell}(Y)$.
Notably, the target style is encoded as a low-dimensional GMM, and the reference
image is not needed during training, unlike in \cite{gatys}. Our objective is a
weighted sum of Mixture Losses between the optimised features
$\text{VGG}_{1\cdots\ell}(X)$ and the target
$\overline{\mu}_{\ell}^{\text{style}}$, for each layer $\ell$ in $\{1,2,3\}$:
we solve
\begin{equation}\label{eq:style_loss}
    \min_{X\in \mathbb{R}^{3\times H\times W}} \sum_{l=1}^{3}\lambda_{\ell}\operatorname{MW}_{2}^{2}\left(F(\theta_\text{init},\text{VGG}_{1\cdots\ell}(X)),\overline{\mu}_{\ell}^{\text{style}}\right).
\end{equation}
The weights $\lambda_{\ell}$ follow Gatys et al.'s scheme
(\cite{gatys2015texture}) of $1/d_{\ell}^{2}$ where $d_{\ell}$ is the dimension
of features in layer $\ell$. We fit $K=3$ Gaussian components at each layer. The
chosen procedure for fitting Gaussian mixtures is still \emph{Warm-Start EM}. We
optimise using 100 iterations of Adam with learning rate 0.01, which takes
approximately 20 seconds using CUDA on an RTX 4000. The example results shown in
\cref{fig:style_transfer} illustrate the ability of GMM to encode (and store) an
image style which, to the best of our knowledge, has never been shown. The
experimental setup is further detailed in \cref{sec:style_transfer_details},
along with additional examples.

\begin{figure}[ht]
    \centering
    \begin{adjustbox}{valign=c}
        \begin{subfigure}{0.3\textwidth} 
        \centering
        \includegraphics[width=\textwidth]{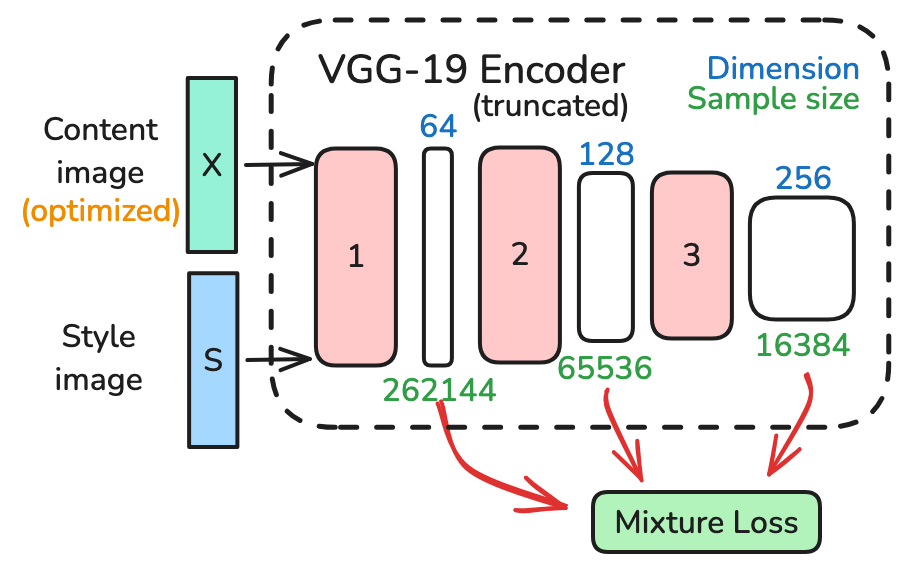}
        \caption{Model architecture}
        \label{fig:vgg19_encoder}
        \end{subfigure}
    \end{adjustbox}
    \hfill
    \begin{adjustbox}{valign=c}
        \begin{subfigure}{0.19\textwidth} 
            \centering
            \includegraphics[width=\textwidth]{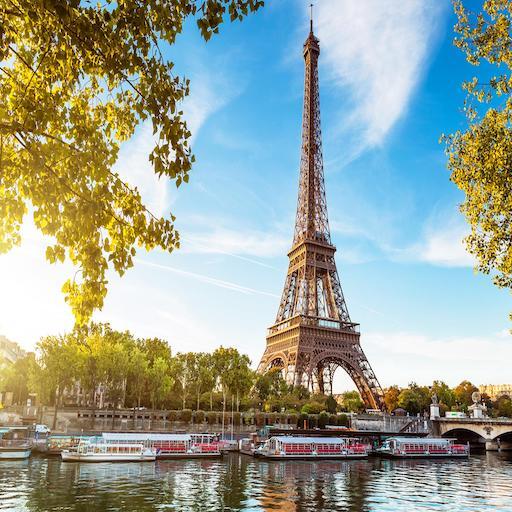}
            \caption{Source im.}
            \label{fig:style_transfer_source}
        \end{subfigure}
    \end{adjustbox}
    \hfill
    \begin{adjustbox}{valign=c}
        \begin{subfigure}{0.19\textwidth} 
            \centering
            \includegraphics[width=\textwidth]{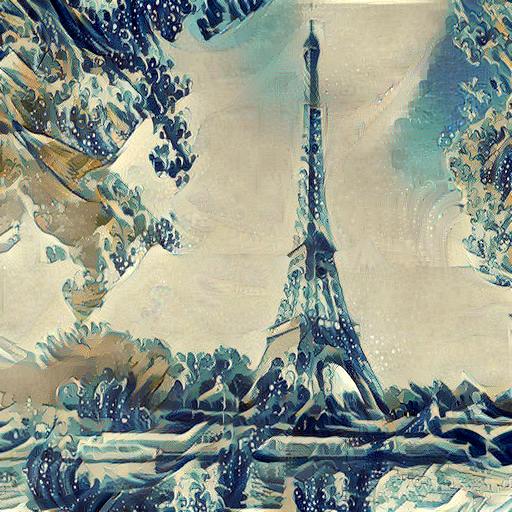}
            \caption{Generated im.}
            \label{fig:style_transfer_result}
        \end{subfigure}
    \end{adjustbox}
    \hfill
    \begin{adjustbox}{valign=c}
        \begin{subfigure}{0.19\textwidth} 
            \centering
            \includegraphics[width=\textwidth]{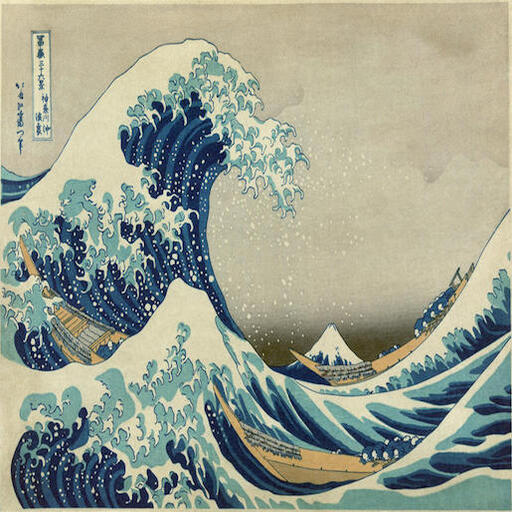}
            \caption{Style im.}
            \label{fig:style_transfer_target}
        \end{subfigure}
    \end{adjustbox}
    \caption{Style transfer method inspired by \cite{gatys2015texture}:
    setup and example result.}
    \label{fig:style_transfer}
\end{figure}

\subsection{Texture Synthesis}

We perform texture synthesis using a novel method inspired by
\cite{galerne2018texture,leclaire2023optimal,houdard2023generative}. We
initialise the synthesised texture using a stationary Gaussian field of the same
mean and covariance as the target texture. We then optimise a weighted sum of
Mixture Losses over different scales in the patch space (we refer the reader to
\cref{sec:texture_synthesis_details} for a full explanation). When doing
multi-scale synthesis, we simply choose to downscale the images by a factor
$2^s$ for $s$ between $0$ and $S$, so that the image downscaled by a factor
$2^S$ has size at least $16\times 16$. In our experiments, we choose to fit
$K=4$ Gaussian components in our mixtures. As illustrated in
\cref{fig:text_K_comp}, this corresponds (roughly) to the elbow of the model's
log-likelihood and gives more convincing results. As for patch sizes, notice
that choosing a patch size of $1$ amounts to optimising the colour intensities
directly, i.e. performing standard colour transfer. In
\cref{fig:texture_comp_patch_size}, we compare the results of taking $4\times 4$
and $8\times 8$ patches. The mono-scale variant is functional for simple
textures, as presented in \cref{sec:texture_synthesis_details}. A strength of
this approach is that it deals with multiple scales at once, whereas
\cite{leclaire2023optimal,houdard2023generative} go through each scale
successively. This provides a simpler approach that is less sensitive to the
transition method between scales.

\begin{figure}[ht]
  \centering
  \noindent\hspace*{\fill}
  \begin{subfigure}[t]{.25\linewidth}
    \centering
    \includegraphics[width=\linewidth]{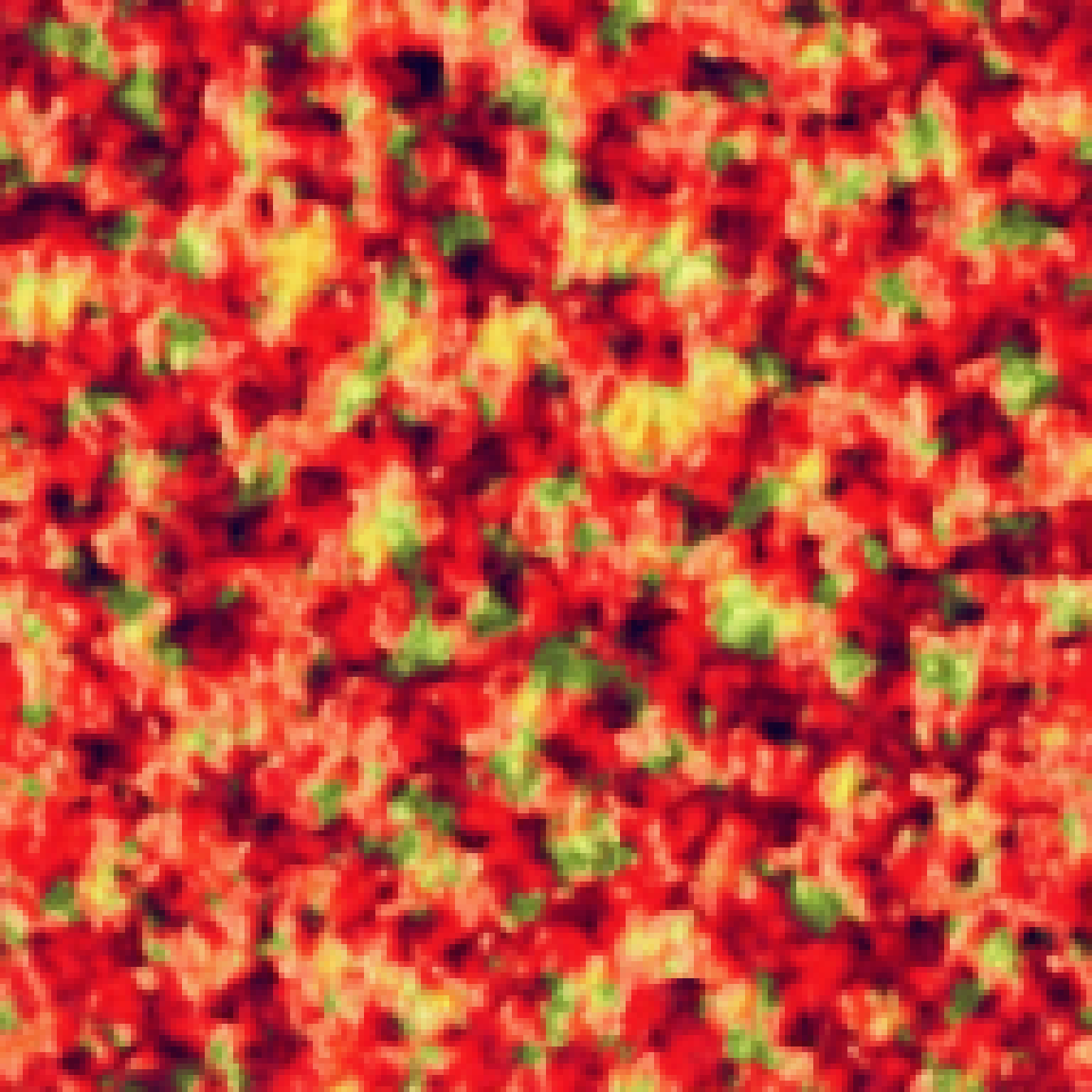}
    \caption{$K=1$ components}
  \end{subfigure}
  \hfill
  \begin{subfigure}[t]{.25\linewidth}
    \centering
    \includegraphics[width=\linewidth]{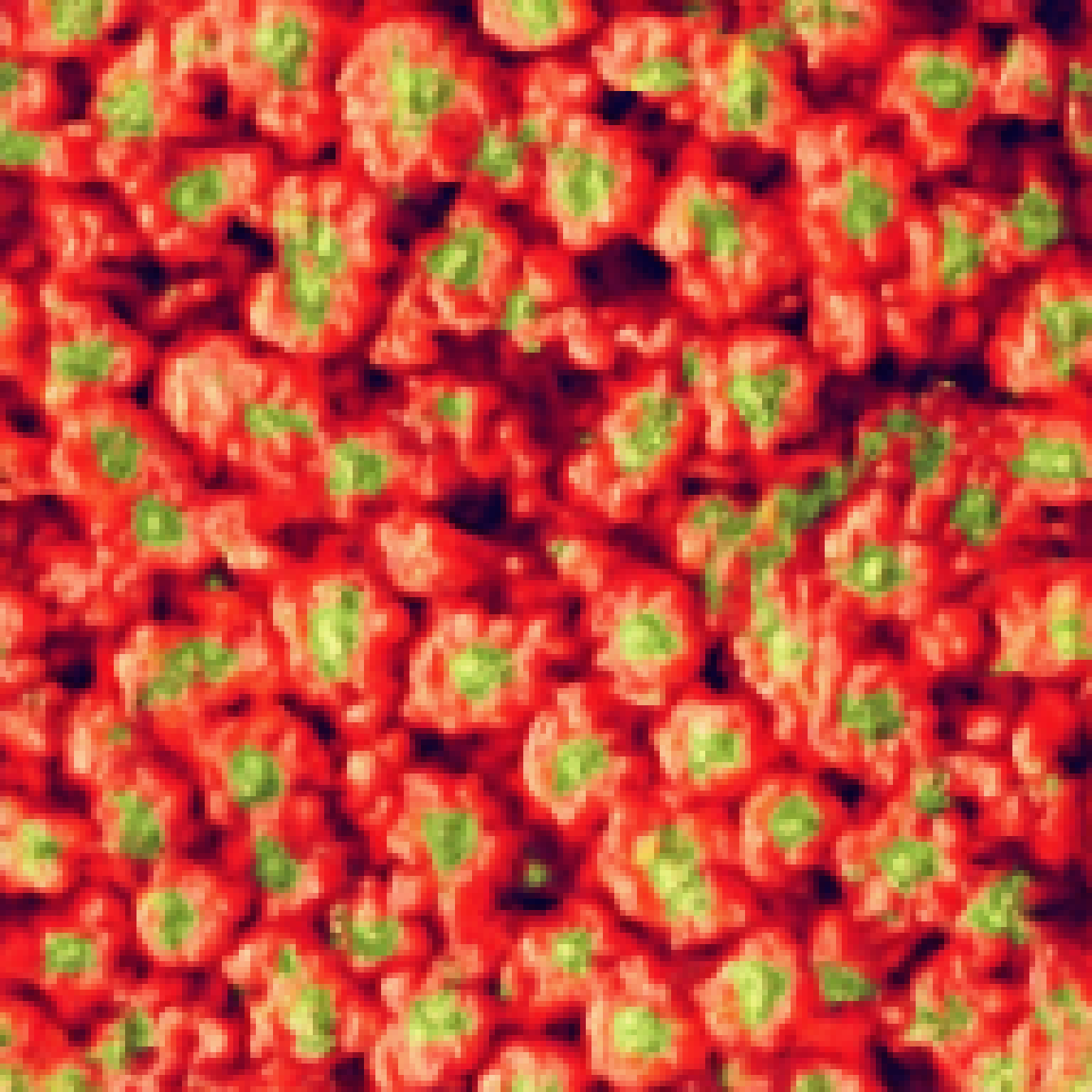}
    \caption{$K=4$ components}
  \end{subfigure}
  \hfill
  \begin{subfigure}[t]{.25\linewidth}
    \centering
    \includegraphics[width=\linewidth]{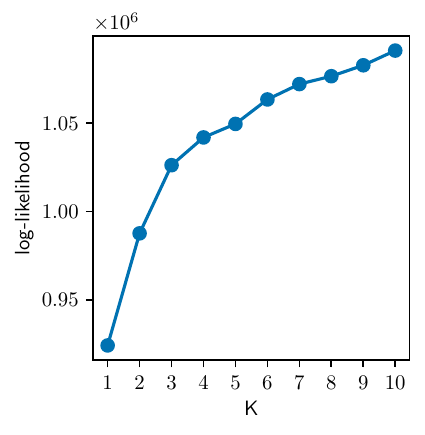}
    \caption{Log-likelihood of fit}
  \end{subfigure}
  \hspace*{\fill}
  \caption{Choosing the number of components $K$ for texture synthesis.}
  \label{fig:text_K_comp}
\end{figure}

\begin{figure}[ht]
  \centering
  \begingroup
  \setlength{\tabcolsep}{2pt}
  \renewcommand{\arraystretch}{0}
  \begin{tabular}{c c c c}
    \rotatebox{90}{\hspace{0.5cm} Reference} &
    \includegraphics[width=.18\linewidth]{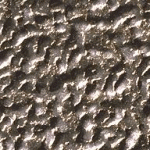} &
    \includegraphics[width=.18\linewidth]{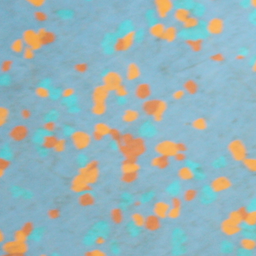} &
    \includegraphics[width=.18\linewidth]{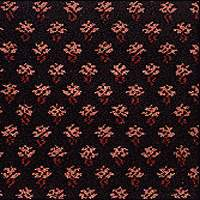} \\[2pt]
    \rotatebox{90}{Generated} &
    \includegraphics[width=.27\linewidth]{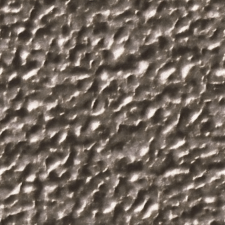} &
    \includegraphics[width=.27\linewidth]{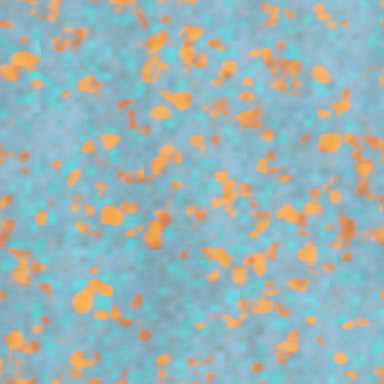} &
    \includegraphics[width=.27\linewidth]{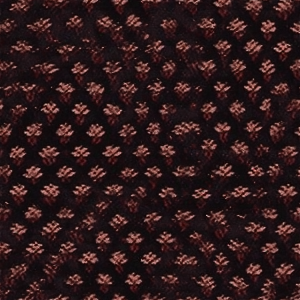}
  \end{tabular}
  \caption{Multi-scale texture synthesis with $K=4$ components for $8\times8$ 
  patches}
  \label{fig:texture_comp_patch_size}
  \endgroup
\end{figure}

\subsection*{Conclusion and Outlook}

In this work, we have provided multiple strategies to differentiate the
Expectation-Maximisation algorithm, and illustrated their use in several
applications. The overall practical message is that when possible, the
Warm-Start method is to be preferred, and when not applicable, the Automatic
Differentiation method is a good off-the-shelf solution. A core limitation of
our methods (and EM to a lesser extent) is the computational cost when the
dimension is large: our applications were limited to $d \approx 1000$. For very
high dimensional cases, one could consider sparse covariance representations
such as \cite{bouveyron2025scaling,szwagier2025parsimonious}. While we recommend
using fixed uniform GMM weights, learnable weights remain an option. Further
numerical and algorithmic tweaks could be considered to avoid issues with local optima.
Finally, we focused on Gaussian Mixtures, but any mixture with
differentiable densities could be considered. Note that for a seamless
extension, one still requires closed-form expressions for the E and M steps. 
\subsection*{Acknowledgements}

This research was funded in part by the Agence nationale de la
recherche (ANR), Grant ANR-23-CE40-0017 and by the
France 2030 program, with the reference ANR-23-PEIA-0004.

\FloatBarrier
\newpage
\printbibliography

\newpage
\appendix
\section{Supplementary Material}

\sammod{
\subsection{Discussion on the Invertibility Assumption}
\label{app:non-invertibility}

Let
$J^{\star}:=\partial_{\theta}\FEM(\theta^{\star},X_{0})$.
A sufficient condition for the invertibility assumption in
\cref{ass:EM_cv} is local contractivity of the EM map at the limiting fixed
point, i.e. $\lVert J^{\star}\rVert<1$. Moreover, since $\FEM$ is continuously
differentiable,
$I-\partial_{\theta}\FEM(\theta,X)$ remains invertible in a neighbourhood of
$(\theta^{\star},X_{0})$. It is therefore invertible at
$(\theta_{T},X)$ for sufficiently large $T$ when
$\theta_{T}\to\theta^{\star}$ and $X$ remains in this neighbourhood. The AI
method nevertheless requires the corresponding linear system to be
nonsingular at every pair $(\theta_{T},X)$ where it is evaluated.

The link with identifiability can be seen in a symmetric two-component
example, where the component means are $\pm\mu$ and the weights and
covariances are fixed. The EM update then reduces to
$\mu_{t+1}=\FEM(\mu_t)$. When the modes are well separated,
\cref{fig:reviewer1_em_high_separation} shows that $\FEM$ is strongly
contracting around $\mu^{\star}$ and that the iterates converge rapidly. When
the modes overlap, \cref{fig:reviewer1_em_low_separation} shows that
$\partial_{\mu}\FEM(\mu^{\star})$ approaches $1$: EM converges slowly and
$1-\partial_{\mu}\FEM(\mu^{\star})$ becomes poorly conditioned.

\begin{figure}[ht]
    \centering
    \begin{subfigure}[t]{0.48\textwidth}
        \centering
        \includegraphics[width=\textwidth]{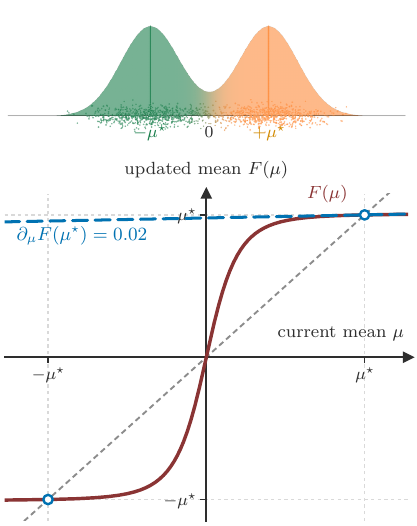}
        \caption{Well-separated modes.}
        \label{fig:reviewer1_em_high_separation}
    \end{subfigure}
    \hfill
    \begin{subfigure}[t]{0.48\textwidth}
        \centering
        \includegraphics[width=\textwidth]{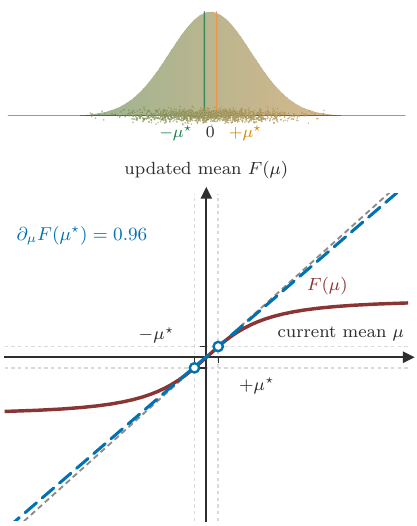}
        \caption{Overlapping modes.}
        \label{fig:reviewer1_em_low_separation}
    \end{subfigure}
    \caption{Local behaviour of the EM map in a symmetric two-component
    example. The grey dashed line is the identity map.}
    \label{fig:reviewer1_em_identifiability}
\end{figure}

The invertibility assumption does not hold in specific pathological cases. For instance, let
$\mathcal{I}\subset\mathbb{R}$ be an interval containing $0$, and suppose that
a differentiable curve
$s\in\mathcal{I}\mapsto\theta(s)$ of fixed points satisfies
$\theta(0)=\theta^{\star}(\theta_{0},X_{0})$ and $\theta'(0)\ne0$.
Differentiating
$\FEM(\theta(s),X_{0})=\theta(s)$ at $s=0$ gives
\(
(I-\partial_{\theta}\FEM(
\theta^{\star}(\theta_{0},X_{0}),X_{0}))\theta'(0)=0,
\)
so this matrix is singular. For example, consider the dataset
$X_{0}=(-1,1)$ and the Gaussian mixture
$(\frac12-s)\mathcal{N}(0,1)+(\frac12+s)\mathcal{N}(0,1)$. Then
$s\in(-\frac12,\frac12)\mapsto\theta(s) := (w(s) := \frac{1}{2} - s, \mu(s) := 0, \Sigma(s) : =1)$ is such a curve of fixed points.}

\subsection{Postponed Proofs}\label{sec:proofs_sample_complexity_mw2}

\paragraph{Proof of \cref{prop:sample_complexity_mw2_one_sample}} Consider the
cost matrices $C, \hat C \in \R_+^{K_0\times K_1}$ defined by their entries at
$(k_0, k_1) \in \llbracket 1, K_0 \rrbracket \times \llbracket 1, K_1
\rrbracket$:
$$C_{k_0, k_1} := \|m_{k_0}^{(0)} - m_{k_1}^{(1)}\|_2^2 +
\dBW^2(\Sigma_{k_0}^{(0)}, \Sigma_{k_1}^{(1)}),\; \hat C_{k_0, k_1} := \|\hat
m_{k_0}^{(0)} - \hat m_{k_1}^{(1)}\|_2^2 + \dBW^2(\hat\Sigma_{k_0}^{(0)},
\hat\Sigma_{k_1}^{(1)}),$$ and apply \cite[Lemma
2.1]{tanguy2023discrete_sw_losses} with weights $(w_i, \oll{w_i}) := (w_i, \hat
w_i)$ for $i\in \{0, 1\}$, and cost matrices $(M, \oll{M}) := (C, \hat C)$: we
obtain:
$$\left|\MW_2^2(\hat\mu_1, \hat\mu_2) - \MW_2^2(\mu_1, \mu_2)\right| \leq \|C -
\hat C\|_\infty + \|C\|_\infty\left(\|w_0 - \hat w_0\|_1 + \|w_1 - \hat
w_1\|_1\right). $$ First, noticing that $\forall \Sigma \in \PSD,\; \dBW(0,
\Sigma) = \sqrt{\Tr(\Sigma)}$, we apply the triangle inequality on $\|\cdot\|_2$
and $\dBW$ to obtain $\|C\|_\infty \leq 4(R_m^2 + R_\Sigma^2)$. We now turn to
upper-bounding $\|C - \hat C\|_\infty$. We will use the inequality $\forall (s,
t) \in \R^2,\; |s^2 - t^2|\leq 2\max(|s|, |t|)|s-t|$. First, for $x, \hat x, y,
\hat y \in B_{\R^d}(0, R_m)$, we have by the triangle inequality:
$$\left|\|x - y\|_2^2 + \|\hat x - \hat y\|_2^2\right| \leq 2\max(\|x-y\|_2,
\|\hat x - \hat y\|_2)\|(x-y)-(\hat x - \hat y)\|_2 \leq 4R_m \left(\|x-\hat
x\|_2 + \|y-\hat y\|_2\right).$$ Similarly, for $A, \hat A, B, \hat B \in \{S
\in \PSD : \sqrt{\Tr(S)} \leq R_\sigma\}$, we compute:
$$\left|\dBW(A, B) - \dBW(\hat A, \hat B)\right| \leq 4R_\Sigma\left(\dBW(A,
\hat A) + \dBW(B, \hat B)\right).$$ Applying our two inequalities to the entries
of $C - \hat C$ and combining with the inequality on $\MW_2^2$ gives
\cref{eqn:sample_complexity_mw2_two_sample} in expectation.

\paragraph{Proof of \cref{prop:sample_complexity_mw2_two_sample}} For notational
simplicity, introduce $g_k := \mathcal{N}(m_k, \Sigma_k)$ and $\hat g_k :=
\mathcal{N}(m_k, \hat\Sigma_k)$. Let $\eta_k := \min(\hat w_k, w_k)$ and the
diagonal matrix $D := \diag(\eta)$. Our objective is to complete the
non-negative matrix $D$ into a coupling $P \in \Pi(\hat w, w)$. First, if
$\sum_k \eta_k = 1$, then $w=\hat w$ and already $D \in \Pi(\hat w, w)$, so we
can set $P := D$. If $\sum_k \eta_k \neq 1$, we introduce $Q \in \R_+^{K \times
K}$ such that $Q_{k,\ell} = (\hat w_k - \eta_k)(w_\ell - \eta_\ell) / (1 -
\sum_{p}\eta_p)$ for all $(k,\ell)$. The matrix $Q$ is the independent coupling
between the non-negative vectors $\hat w - \eta$ and $w - \eta$. We set $P := D
+ Q$, straightforward computation shows that $P \in \Pi(\hat w, w)$. We denote
the pairwise cost matrix $C := (\W_2^2(\hat g_k, g_\ell))_{k,\ell}$ and use $P$
as a candidate coupling in the definition of $\MW_2^2(\hat\mu, \mu)$:
\[
\MW_2^2(\hat\mu, \mu) \leq \langle P, C \rangle 
= \langle D, C \rangle + \langle Q, C \rangle
\leq \max_{k}\W_2^2(\hat g_k, g_k) (\sum_\ell \eta_\ell) 
+ \|C\|_\infty \|Q\|_1.
\]
(Note that in the case $\sum_k \eta_k = 1$, the second term vanishes.) Then
since $\eta \leq w$ entry-wise, $\sum\eta_k \leq 1$, and we compute $\|Q\|_1
= 1 - \sum_k \eta_k$, and further re-write:
\[
\sum_{k=1}^K\eta_k = \sum_{k=1}^K\min(\hat w_k, w_k) 
= \frac{1}{2}\sum_{k=1}^K(\hat w_k + w_k - |\hat w_k - w_k|) 
= 1 - \frac{\|\hat w - w\|_1}{2},
\]
which finally yields:
\[
\MW_2^2(\hat\mu, \mu) \leq  \max_{k}\W_2^2(\hat g_k, g_k)
+ \frac{\|\hat w - w\|_1}{2}\max_{k,\ell}\W_2^2(\hat g_k, g_\ell).
\]
Taking expectations and using the assumptions on the parameter estimators
concludes the proof.

\subsection{Specific GMMs Used in \texorpdfstring{\cref{sec:sample_complexity}}{sample-complexity} and \texorpdfstring{\cref{sec:xp_gradients_comparison}}{xp-grad-comp}}\label{sec:grad_comparison_gmms}

In \cref{sec:sample_complexity}, we use two fixed mixtures $\mu$ and $\nu$ with
$K=3$ components in $d=2$, shown in \cref{fig:gmm_pairs_plot}. Separation is
controlled by the covariance scaling parameter $\sigma$: ``low'' uses
$\sigma=10$, ``medium'' $\sigma=0.5$, and ``high'' $\sigma=0.1$. Larger values
create strong overlap between components, while smaller values yield
well-separated mixtures. The same pairs $(\mu,\nu)$ are used across all
repetitions.

In \cref{fig:grad_comparison_gmms}, we show the three different GMMs used in our
experiments, which were sampled randomly with fixed seeds. We observe that GMM
$\#$1 has relatively large variances, allowing for less numerical difficulty,
while GMM $\#$2 and (to an even larger extent) GMM $\#$3 have some
almost-degenerate covariances, leading to increased numerical instability. Note
that GMM $\#2$ appears to be the best separated, which can yield better EM
convergence.

\begin{figure}[ht]
    \centering
    \includegraphics[width=.8\textwidth]{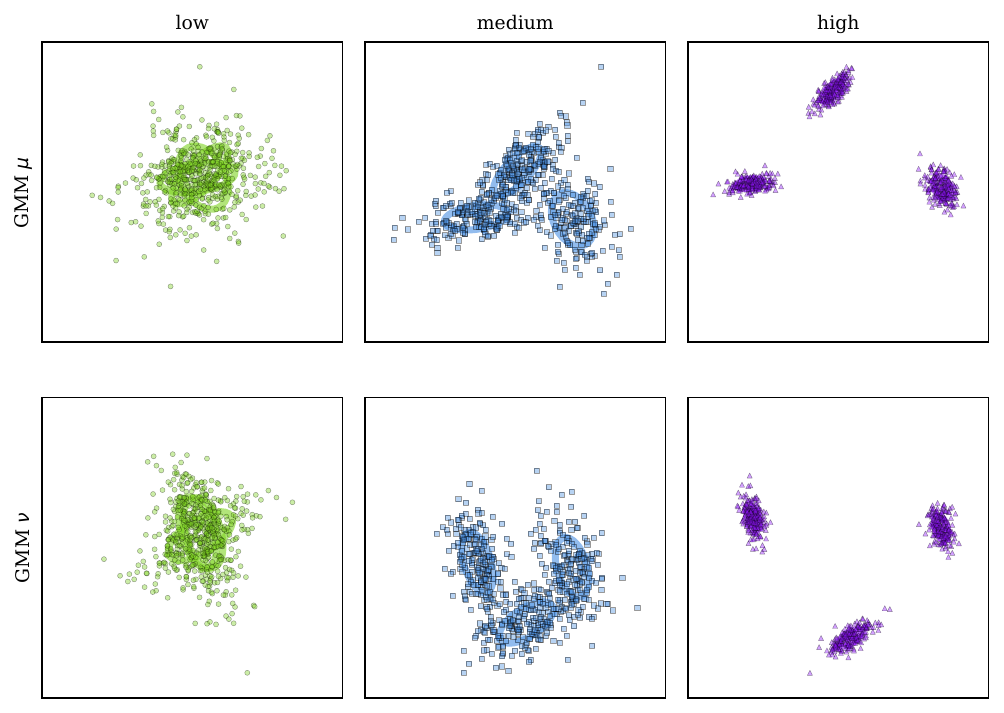}
    \caption{Two fixed GMMs $\mu$ (top) and $\nu$ (bottom) used in
    \cref{sec:sample_complexity}. Columns correspond to low ($\sigma=10$),
    medium ($\sigma=0.5$), and high ($\sigma=0.1$) separation.}
    \label{fig:gmm_pairs_plot}
\end{figure}

\begin{figure}[ht]
    \centering
    \includegraphics[width=.7\textwidth]{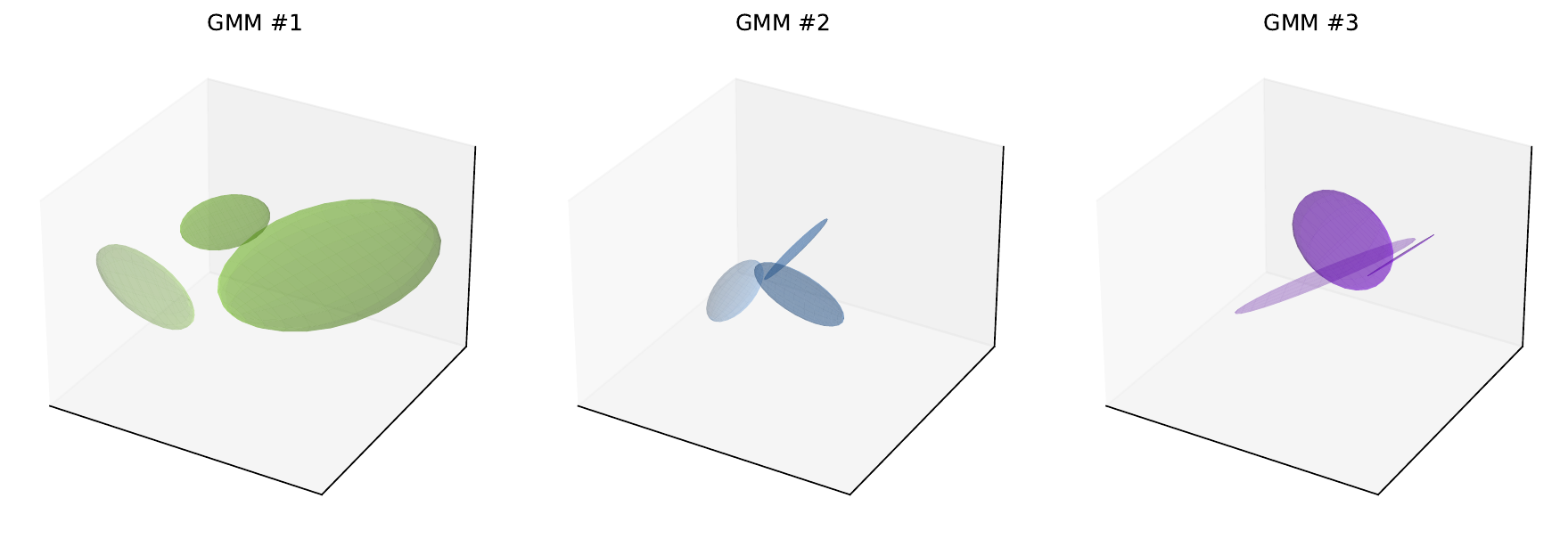}
    \caption{3D representation of the three GMMs used in
    \cref{sec:xp_gradients_comparison} to compare EM gradient methods.}
    \label{fig:grad_comparison_gmms}
\end{figure}

\subsection{Discussion on Gradient Ground Truths}\label{sec:em_grad_gt}

As discussed in \cref{sec:em_grad_methods}, the most natural baseline for the
computation of a gradient of EM is the full automatic differentiation method
(AD). Unfortunately, this strategy is only theoretically sound under strong
assumptions \cite{gilbert1992automatic,beck1994automatic,mehmood2020automatic},
and may be prone to the propagation of numerical errors across iterations. Due
to this latter issue in particular, one may not claim that the AD gradient is
exact. In order to discuss its use as a baseline (in particular in
\cref{sec:xp_gradients_comparison}), we compare it to the gradient computed
using the finite differences approximation (FD) with a step size $\FDstep$,
implemented using \texttt{torch.autograd.gradcheck.\_get\_numerical\_jacobian}
in \href{https://pytorch.org/}{Pytorch} \cite{pytorch}. The experimental setting
is as in \cref{sec:xp_gradients_comparison}: in
\cref{fig:multiple_eps_n300_d3_K3_T30}, we observe the relative MSEs $\|\JFD -
\JAD\|_2^2 / \|\JAD\|_2^2$ varying the FD step size $\FDstep \in \{10^{-5},
10^{-6}, 10^{-7}\}$ on three different GMMs, taking each time 60 samples for the
data $X$ with $n=300$, $d=3$ and running EM with $K=3$ and $T=30$. Note that the
FD approximation is extremely intensive numerically, prohibiting its use in
practice and even for extensive testing. First, we observe variability between
data samples, indicating sensitivity of the methods to the algorithm
initialisation (and in turn, convergence). \sammod{Fortunately, the AD method is
close to the AD method at numerical precision in all three cases. Given this
observation, the AD method appears to be an adequate baseline.}

\begin{figure}[ht]
    \centering
    \includegraphics[width=.4\textwidth]{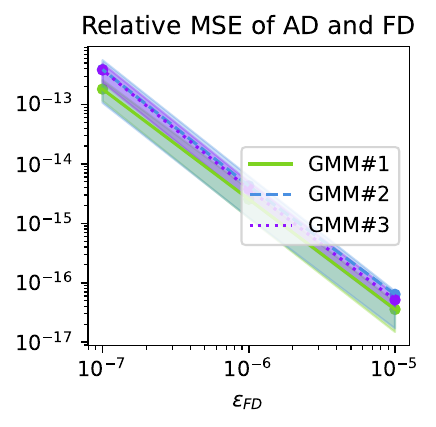}
    \caption{Relative MSE of the full automatic differentiation gradient (AD)
    against the finite differences approximation (FD) for three different GMMs
    and varying the FD step size $\FDstep$.}
    \label{fig:multiple_eps_n300_d3_K3_T30}
\end{figure}

\sammod{
\subsection{Numerical Comparison of Warm and Cold Initialisation}
\label{sec:warm_start_multidimensional}

We compare warm and cold initialisation in a flow between two fixed random
GMMs with $K=3$ components in dimension $d=20$ (\cref{tab:warm_start_multidimensional}).
For each of two samples of
$n=200$ points from the same source GMM, we run $300$ outer SGD updates with
learning rate $0.05$. Both methods differentiate through exactly
$T\in\{1,2,5,10\}$ EM steps. Their only difference is that, after the first
outer iteration, the warm run is initialised with the preceding EM fit whereas
the cold run retains the same seeded initialisation at every iteration. The GMM
weights are fixed and all other settings are identical.
For both samples, the warm and cold runs give the same recorded objective and
fitted parameters at $T=10$, and the warm run with $T=1$ gives almost the same
objective as the ones with larger $T$.

\begin{table}[ht]
    \centering
    \begin{tabular}{ccrrrrr}
        \toprule
        Sample & $T$ & $\MW_2^2$ (warm) & $\MW_2^2$ (cold)
        & Difference \\
        \midrule
        \#1 & 1  & 676.5595 & 662.2502 & 14.3093 \\
        \#1 & 2  & 674.8875 & 670.1658 &  4.7217 \\
        \#1 & 5  & 674.6132 & 674.6144 &  0.0012 \\
        \#1 & 10 & 674.6132 & 674.6132 &  0      \\
        \#2 & 1  & 680.5627 & 622.3420 & 58.2207 \\
        \#2 & 2  & 680.4759 & 629.5430 & 50.9329 \\
        \#2 & 5  & 680.4014 & 643.9372 & 36.4642 \\
        \#2 & 10 & 680.4597 & 680.4597 &  0      \\
        \bottomrule
    \end{tabular}
    \caption{Comparison between warm-start and cold-start runs.}
    \label{tab:warm_start_multidimensional}
\end{table}

}

\subsection{Explicit Differential
Expressions}\label{sec:explicit_euclidean_differentials}

\subsubsection{Differential of Gaussian Density}

First, we compute the explicit differential of
\begin{equation}\label{eqn:gaussian_density_g_def}
    g:\app{\R^d\times\PD\times\X}{\R^n}{(m, \Sigma, X)}{\left(\cfrac{1}{(2\pi)^{d/2}\det(\Sigma)^{1/2}}\exp\left(-\cfrac{1}{2}(x_i-m)^\top\Sigma^{-1}(x_i-m)\right)\right)_{i=1}^n,}
\end{equation}
where we wrote $X = (x_1, \cdots, x_n)$. We first compute the differential with
respect to $m$:
\begin{equation}\label{eqn:dgdm}
    \cfrac{\partial g}{\partial m}(m, \Sigma, X) \in \L(\R^d, \R^n)\simeq\R^{n\times d}:\; \left[\cfrac{\partial g}{\partial m}(m, \Sigma, X)\right]_{i,\cdot} = \Sigma^{-1}(x_i-m)g_{m, \Sigma}(x_i).
\end{equation}
For the differential with respect to $\Sigma$, we remark that
$T_{\Sigma}\PD=\Sym$, and use \cite[Equation 49]{matrix_cookbook} stating that
$\partial_A\det A = (\det A )A^{-\top}$ and \cite[Equation 61]{matrix_cookbook}
which states $\partial_A a^\top A^{-1}b = -A^{-\top}ab^\top A^{-\top}$, which
yields:
\begin{align}
    \cfrac{\partial g}{\partial \Sigma}(m, \Sigma, X) &\in \L(\Sym, \R^n)\hookrightarrow\R^{n\times d\times d}:\\
    \left[\cfrac{\partial g}{\partial \Sigma}(m, \Sigma, X)\right]_{i,\cdot,\cdot} &= \cfrac{g_{m, \Sigma}(x_i)}{2}\left(-\Sigma^{-1}+\Sigma^{-1}(x_i-m)(x_i-m)^\top\Sigma^{-1}\right). \label{eqn:dgdSigma}
\end{align}
Finally, the differential with respect to the data $X$ reads
\begin{equation}\label{eqn:dgdu}
    \cfrac{\partial g}{\partial X}(m, \Sigma, X) \in \L(\R^{n\times d}, \R^n)\simeq\R^{n\times n\times d}:\; \left[\cfrac{\partial g}{\partial \Sigma}(m, \Sigma, X)\right]_{i,j,\cdot}=-\delta_{i,j}g_{m, \Sigma}(x_i)\Sigma^{-1}(x_i-m).
\end{equation}

\subsubsection{Differential of \texorpdfstring{$\gamma$}{gamma}} We now compute
the differential of the function
\begin{equation}\label{eqn:gamma_def}
    \gamma:\app{\underbrace{\simplex{K}\times(\R^d)^K\times(\PD)^K}_{\Theta}\times\X}{\R^{n\times K}}{(\underbrace{w, (m_k), (\Sigma_k)}_{\theta}, X)}{\Biggl(\cfrac{w_k g_{m_k, \Sigma_k}(x_i)}{\sum_l w_l g_{m_l, \Sigma_l}(x_i)}\Biggr)_{i,k}}.  
\end{equation}
For notational convenience, we introduce
\begin{equation}\label{eqn:shorthands_g_Z}
    \forall (i,k) \in \llbracket 1, n \rrbracket\times \llbracket 1, K \rrbracket,\; g_{i,k} := g_{m_k, \Sigma_k}(x_i),\; Z_i := \sum_{k'} w_{k'} g_{i,k'}.
\end{equation}
First, the tangent space of the simplex is the same everywhere: $T_w\simplex{K}
= \simplex{K}^0 := (\mathbbold{1}_K)^\perp$, and the differential with respect
to $w$ is
\begin{equation}\label{eqn:dgammadw}
    \cfrac{\partial \gamma}{\partial w}(\theta, X) \in \L\left(\simplex{K}^0, \R^{n\times K}\right)\hookrightarrow \R^{n\times K \times K}:\; \left[\cfrac{\partial \gamma}{\partial w}(\theta, X)\right]_{i,k,l} = \cfrac{g_{i,k}}{Z_i}\left(\delta_{k,l} - \cfrac{w_kg_{i,l}}{Z_i}\right).
\end{equation}
Using \cref{eqn:dgdm}, we compute the differential of $\gamma$ w.r.t.
$\mathbf{m}=(m_k)$:
\begin{align}
    \cfrac{\partial \gamma}{\partial \mathbf{m}}(\theta, X) &\in \L\left((\R^{d})^K, \R^{n\times K}\right) \simeq \R^{n\times K \times K \times d}:\\
    \left[\cfrac{\partial \gamma}{\partial \mathbf{m}}(\theta, X)\right]_{i,k,l,\cdot} &= \cfrac{w_kg_{i,k}}{Z_i}\left(\delta_{k,l}-\cfrac{w_lg_{i,l}}{Z_i}\right)\Sigma_l^{-1}(x_i -m_l).\label{eqn:dgammadm}
\end{align}
Similarly, using \cref{eqn:dgdSigma}, we compute the differential of $\gamma$
w.r.t. $\mathbf{\Sigma}=(\Sigma_k)$:
\begin{align}
    \cfrac{\partial \gamma}{\partial \mathbf{\Sigma}}(\theta, X) &\in \L\left((\Sym)^K, \R^{n\times K}\right) \hookrightarrow \R^{n\times K \times K \times d\times d}: \\
    \left[\cfrac{\partial \gamma}{\partial \mathbf{\Sigma}}(\theta, X)\right]_{i,k,l,\cdot, \cdot} &= \cfrac{w_kg_{i,k}}{2Z_i}\left(\delta_{k,l}-\cfrac{w_lg_{i,l}}{Z_i}\right)\left(-\Sigma_l^{-1}+\Sigma_l^{-1}(x_i-m_l)(x_i-m_l)^\top\Sigma_l^{-1}\right). \label{eqn:dgammadSigma}
\end{align}
Finally, \cref{eqn:dgdu} allows the computation of the differential of $\gamma$
w.r.t.$X$:
\begin{align}
    \cfrac{\partial \gamma}{\partial X}(\theta, X) &\in \L\left(\R^{n\times d}, \R^{n\times K}\right) \simeq \R^{n\times K \times n \times d}: \\
    \left[\cfrac{\partial \gamma}{\partial X}(\theta, X)\right]_{i,k, j,\cdot} &= \cfrac{\delta_{i,j}w_kg_{i,k}}{Z_i}\left(-\Sigma_k^{-1}(x_i-m_k) + \cfrac{1}{Z_i}\sum_{l=1}^K w_lg_{i,l}\Sigma_l^{-1}(x_i-m_l)\right). \label{eqn:dgammadu}
\end{align}

\subsubsection{Differential of \texorpdfstring{$\FEM$}{F}}

We introduce the coordinate notation 
$$\FEM(\theta, X) = (\FEM_w(\theta, X), \FEM_{\mathbf{m}}(\theta, X),
\FEM_{\mathbf{\Sigma}}(\theta, X)) \in \simplex{K} \times (\R^d)^K \times
(\PD)^K,$$ hence the differentials of $\FEM$ with respect to $\theta$ and $X$
can be written ``block-wise'' as follows:
$$\cfrac{\partial \FEM}{\partial \theta} = \left(\begin{array}{ccc}
    \cfrac{\partial \FEM_w}{\partial w} & \cfrac{\partial \FEM_w}{\partial
    \mathbf{m}} & \cfrac{\partial \FEM_w}{\partial \mathbf{\Sigma}} \\
    \cfrac{\partial \FEM_{\mathbf{m}}}{\partial w} & \cfrac{\partial
    \FEM_{\mathbf{m}}}{\partial \mathbf{m}} & \cfrac{\partial
    \FEM_{\mathbf{m}}}{\partial \mathbf{\Sigma}} \\
    \cfrac{\partial \FEM_{\mathbf{\Sigma}}}{\partial w} & \cfrac{\partial
    \FEM_{\mathbf{\Sigma}}}{\partial \mathbf{m}} & \cfrac{\partial
    \FEM_{\mathbf{\Sigma}}}{\partial \mathbf{\Sigma}} \end{array}\right), \quad
    \cfrac{\partial \FEM}{\partial X} = \left(\begin{array}{c} \cfrac{\partial
    \FEM_w}{\partial X} \\
    \cfrac{\partial \FEM_{\mathbf{m}}}{\partial X} \\
    \cfrac{\partial \FEM_{\mathbf{\Sigma}}}{\partial X} \\
\end{array}\right).$$

\paragraph{Differentials of $\FEM_w$} We now compute the differentials of
$\FEM_w$ whose expression we remind from \cref{eqn:theta_next}:
\begin{equation}\label{eqn:def_F_w}
    \FEM_w: \app{\Theta\times\X}{\simplex{K}}{(\theta, X)}{\left(\cfrac{1}{n}\sum_{i=1}^n\gamma_{i,k}(\theta, X)\right)_{k=1}^K}.
\end{equation}
By \cref{eqn:dgammadw}, we have
\begin{equation}\label{eqn:dFwdw}
    \cfrac{\partial \FEM_w}{\partial w}(\theta, X) \in \L\left(\simplex{K}^0, \simplex{K}^0\right) \hookrightarrow \R^{K\times K}:\; \left[\cfrac{\partial \FEM_w}{\partial w}(\theta, X)\right]_{k,l} = \cfrac{1}{n}\sum_{i=1}^n\cfrac{g_{i,k}}{Z_i}\left(\delta_{k,l}-\cfrac{w_kg_{i,l}}{Z_i}\right).
\end{equation}
Likewise, \cref{eqn:dgammadm} yields
\begin{align}
    \cfrac{\partial \FEM_w}{\partial \mathbf{m}}(\theta, X) &\in \L\left((\R^d)^K, \simplex{K}^0\right) \hookrightarrow \R^{K\times K\times d}:\\
    \left[\cfrac{\partial \FEM_w}{\partial \mathbf{m}}(\theta, X)\right]_{k,l,\cdot} &= \cfrac{1}{n}\sum_{i=1}^n\cfrac{w_kg_{i,k}}{Z_i}\left(\delta_{k,l}-\cfrac{w_lg_{i,l}}{Z_i}\right)\Sigma_l^{-1}(x_i-m_l). \label{eqn:dFwdm}
\end{align}
\cref{eqn:dgammadSigma} gives
\begin{align}
    \cfrac{\partial \FEM_w}{\partial \mathbf{\Sigma}}(\theta, X) &\in \L\left((\Sym)^K, \simplex{K}^0\right) \hookrightarrow \R^{K\times K\times d\times d}:\\
    \left[\cfrac{\partial \FEM_w}{\partial \mathbf{\Sigma}}(\theta, X)\right]_{k,l,\cdot, \cdot} &= \cfrac{1}{n}\sum_{i=1}^n\cfrac{w_kg_{i,k}}{2Z_i}\left(\delta_{k,l}-\cfrac{w_lg_{i,l}}{Z_i}\right)\left(-\Sigma_l^{-1}+\Sigma_l^{-1}(x_i-m_l)(x_i-m_l)^\top\Sigma_l^{-1}\right). \label{eqn:dFwdSigma}
\end{align}
Finally, with \cref{eqn:dgammadu}, we obtain
\begin{align}
    \cfrac{\partial \FEM_w}{\partial X}(\theta, X) &\in \L\left(\R^{n\times d}, \simplex{K}^0\right) \hookrightarrow \R^{K\times n \times d}:\\
    \left[\cfrac{\partial \FEM_w}{\partial X}(\theta, X)\right]_{k,i,\cdot} &= \cfrac{w_kg_{i,k}}{nZ_i}\left(-\Sigma_k^{-1}(x_i-m_k) + \cfrac{1}{Z_i}\sum_{l=1}^K w_lg_{i,l}\Sigma_l^{-1}(x_i-m_l)\right). \label{eqn:dFwdu}
\end{align}

\paragraph{Differentials of $\FEM_{\mathbf{m}}$} We now turn to
$\FEM_{\mathbf{m}}$:
\begin{equation}\label{eqn:def_F_m}
    \FEM_{\mathbf{m}}: \app{\Theta\times\X}{\R^{K\times d}}{(\theta, X)}{\left(\cfrac{\sum_{i=1}^n\gamma_{i,k}(\theta, X)x_i}{\sum_{j=1}^n\gamma_{j,k}(\theta, X)}\right)_{k=1}^K}.
\end{equation}
For convenience we introduce $\Gamma_k := \sum_{i=1}^n\gamma_{i,k}$. The chain
rule gives
\begin{align}
    \cfrac{\partial \FEM_\mathbf{m}}{\partial w}(\theta, X) &\in \L\left(\simplex{K}^0, \R^{K\times d}\right) \hookrightarrow \R^{K\times d\times K}:\\
    \left[\cfrac{\partial \FEM_\mathbf{m}}{\partial w}(\theta, X)\right]_{k,\cdot, l} &= \cfrac{1}{\Gamma_k}\left(\sum_{i=1}^n\left[\cfrac{\partial \gamma}{\partial w}\right]_{i,k,l}x_i - \cfrac{1}{\Gamma_k}\sum_{j=1}^n\left[\cfrac{\partial \gamma}{\partial w}\right]_{j,k,l}\sum_{i=1}^n\gamma_{i,k}x_i\right). \label{eqn:dFmdw}
\end{align}
We continue with the differential with respect to $\mathbf{m}$:
\begin{align}
    \cfrac{\partial \FEM_\mathbf{m}}{\partial \mathbf{m}}(\theta, X) &\in \L\left(\R^{K\times d}, \R^{K\times d}\right) \simeq \R^{K\times d\times K\times d}:\\
    \left[\cfrac{\partial \FEM_\mathbf{m}}{\partial \mathbf{m}}(\theta, X)\right]_{k,\cdot, l,\cdot} &= \cfrac{1}{\Gamma_k}\left(\sum_{i=1}^nx_i\left[\cfrac{\partial \gamma}{\partial \mathbf{m}}\right]_{i,k,l,\cdot}^\top - \cfrac{1}{\Gamma_k}\left(\sum_{i=1}^n\gamma_{i,k}x_i\right)\sum_{j=1}^n\left[\cfrac{\partial \gamma}{\partial \mathbf{m}}\right]_{j,k,l,\cdot}^\top\right). \label{eqn:dFmdm}
\end{align}
For the differential with respect to $\mathbf{\Sigma}$, we require the notion of
\textit{outer product} between two tensors, which we define below:
\begin{equation}\label{eqn:outer_prod}
    \forall A \in \R^{n_1\times \cdots\times n_N},\; B \in \R^{m_1\times \cdots\times m_M},\; A\otimes B := \left(A_{i_1, \cdots, i_N}B_{j_1, \cdots, j_M}\right) \in \R^{n_1\times\cdots\times n_N\times m_1 \times \cdots \times m_M}.
\end{equation}
Note that in \cref{eqn:dFmdm}, we could have written
$x_i\otimes\left[\cfrac{\partial \gamma}{\partial
\mathbf{m}}\right]_{i,k,l,\cdot}$. We will need the following differentiation
rule for the outer product:
\begin{lemma}\label{lemma:outer_prod_diff} Let $A: \R^{p_1\times\cdots\times
    p_P} \longrightarrow \R^{n_1\times\cdots\times n_N}$ and $B:
    \R^{k_1\times\cdots\times k_P} \longrightarrow \R^{m_1\times\cdots\times
    m_M}$ be differentiable. Then the differential of $C := A\otimes B$ is:
    \begin{align}
        \cfrac{\partial C}{\partial X}(X) \in \L(\R^{p_1\times\cdots\times p_P},\; &\R^{n_1\times\cdots\times n_N \times m_1 \times \cdots \times m_M}) \simeq \R^{n_1\times\cdots\times n_N \times m_1 \times \cdots \times m_M \times p_1\times\cdots\times p_P}: \\
        \left[\cfrac{\partial C}{\partial X}(X)\right]_{i_1, \cdots, i_N, j_1, \cdots, j_M, k_1,\cdots, k_P} &= \left[\cfrac{\partial A}{\partial X}(X)\right]_{i_1, \cdots, i_N, k_1,\cdots, k_P}B_{j_1, \cdots, j_M}(X) \nonumber\\
        &+ A_{i_1, \cdots, i_N}\left[\cfrac{\partial B}{\partial X}(X)\right]_{j_1, \cdots, j_M, k_1,\cdots, k_P} \label{eqn:outer_prod_diff}, \\
        \cfrac{\partial C}{\partial X}(X) &= \tau\left(\cfrac{\partial A}{\partial X}(X)\otimes B(X)\right) + A(X)\otimes \cfrac{\partial B}{\partial X}(X), \label{eqn:outer_prod_diff_tensor}
    \end{align}
    where $\tau$ is the transposition $\left(\tau(T)\right)_{i_1, \cdots, i_N,
    j_1, \cdots, j_m, k_1,\cdots, k_P} = T_{i_1, \cdots, i_N, k_1, \cdots, k_P,
    j_1,\cdots, j_M}$. If $B = A$, then the formula can be written as
    \begin{equation}
        \cfrac{\partial C}{\partial X}(X) = \tau\left(A(X)\otimes\cfrac{\partial A}{\partial X}(X)\right) + A(X)\otimes\cfrac{\partial A}{\partial X}(X). \label{eqn:outer_prod_diff_same}
    \end{equation}
\end{lemma}
Note that in particular, the intuitive formula $\partial(A\otimes B) = (\partial
A) \otimes B + A \otimes (\partial B)$ does not hold. We now compute the
differential with respect to $\mathbf{\Sigma}$:
\begin{align}
    \cfrac{\partial \FEM_\mathbf{m}}{\partial \mathbf{\Sigma}}(\theta, X) &\in \L\left((\Sym)^K, \R^{K\times d}\right) \hookrightarrow \R^{K\times d\times K\times d\times d}:\\
    \left[\cfrac{\partial \FEM_\mathbf{m}}{\partial\mathbf{\Sigma}}(\theta, X)\right]_{k,\cdot, l,\cdot, \cdot} &= \cfrac{1}{\Gamma_k}\left(\sum_{i=1}^nx_i\otimes\left[\cfrac{\partial \gamma}{\partial \mathbf{\Sigma}}\right]_{i,k,l,\cdot, \cdot} - \cfrac{1}{\Gamma_k}\left(\sum_{i=1}^n\gamma_{i,k}x_i\right)\otimes\sum_{j=1}^n\left[\cfrac{\partial \gamma}{\partial \mathbf{\Sigma}}\right]_{j,k,l,\cdot, \cdot}\right). \label{eqn:dFmdSigma}
\end{align}
The differential with respect to $X$ is slightly different due to the product
with $x_i$ in \cref{eqn:def_F_m}:
\begin{align}
    \cfrac{\partial \FEM_\mathbf{m}}{\partial X}(\theta, X) &\in \L\left(\R^{n\times d}, \R^{K\times d}\right) \simeq \R^{K\times d\times n\times d}:\\
    \left[\cfrac{\partial \FEM_\mathbf{m}}{\partial X}(\theta, X)\right]_{k,\cdot, i,\cdot} &= \cfrac{1}{\Gamma_k}\left(\gamma_{i,k}I_d+\sum_{h=1}^nx_h\left[\cfrac{\partial \gamma}{\partial X}\right]_{h,k,i,\cdot}^\top - \cfrac{1}{\Gamma_k}\left(\sum_{h=1}^n\gamma_{h,k}x_h\right)\sum_{j=1}^n\left[\cfrac{\partial \gamma}{\partial X}\right]_{j,k,i,\cdot}^\top\right). \label{eqn:dFmdu}
\end{align}

\paragraph{Differential of $\FEM_{\mathbf{\Sigma}}$} We finish with the
computation of the differentials of
\begin{equation}\label{eqn:def_F_Sigma}
    \FEM_{\mathbf{\Sigma}}: \app{\Theta\times\X}{(\PD)^K}{(\theta, X)}{\left(\cfrac{\sum_{i=1}^n\gamma_{i,k}(\theta, X)(x_i-\FEM_{\mathbf{m}}(\theta, X)_k)(x_i-\FEM_{\mathbf{m}}(\theta, X)_k)^\top}{\sum_{j=1}^n\gamma_{j,k}(\theta, X)}\right)_{k=1}^K}.
\end{equation}
For convenience, let $\FEM_{m_k} := \left[\FEM_{\mathbf{m}}(\theta,
X)\right]_{k, \cdot}$. We first compute the differential with respect to $w$:
\begin{align}
    \cfrac{\partial \FEM_\mathbf{\Sigma}}{\partial w}(\theta, X) &\in \L\left(\simplex{K}^0, (\Sym)^K\right) \hookrightarrow \R^{K\times d\times d\times K}:\\
    \left[\cfrac{\partial \FEM_\mathbf{\Sigma}}{\partial w}(\theta, X)\right]_{k,\cdot, \cdot, l} &= \cfrac{1}{\Gamma_k}\sum_{i=1}^n\left[\cfrac{\partial \gamma}{\partial w}\right]_{i,k,l}(x_i - \FEM_{m_k})(x_i - \FEM_{m_k})^\top \nonumber\\
    &- \cfrac{1}{\Gamma_k^2}\sum_{j=1}^n\left[\cfrac{\partial \gamma}{\partial w}\right]_{j,k,l}\sum_{i=1}^n\gamma_{i,k}(x_i - \FEM_{m_k})(x_i - \FEM_{m_k})^\top \nonumber\\
    &-\cfrac{1}{\Gamma_k}\sum_{i=1}^n\gamma_{i,k}\left(\left[\cfrac{\partial \FEM_{\mathbf{m}}}{\partial w}\right]_{k,\cdot,l}(x_i - \FEM_{m_k})^\top + (x_i - \FEM_{m_k})\left[\cfrac{\partial \FEM_{\mathbf{m}}}{\partial w}\right]_{k,\cdot,l}^\top\right). \label{eqn:dFSigmadw}
\end{align}
For the differential with respect to $\mathbf{m}$, we will use
\cref{lemma:outer_prod_diff}:
\begin{align}
    \cfrac{\partial \FEM_\mathbf{\Sigma}}{\partial \mathbf{m}}(\theta, X) &\in \L\left(\R^{K\times d}, (\Sym)^K\right) \hookrightarrow \R^{K\times d\times d\times K\times d}:\\
    \left[\cfrac{\partial \FEM_\mathbf{\Sigma}}{\partial \mathbf{m}}(\theta, X)\right]_{k,\cdot, \cdot, l, \cdot} &= \cfrac{1}{\Gamma_k}\sum_{i=1}^n(x_i - \FEM_{m_k})(x_i - \FEM_{m_k})^\top \otimes \left[\cfrac{\partial \gamma}{\partial \mathbf{m}}\right]_{i,k,l,\cdot} \nonumber\\
    &- \cfrac{1}{\Gamma_k^2}\left(\sum_{i=1}^n\gamma_{i,k}(x_i - \FEM_{m_k})(x_i - \FEM_{m_k})^\top\right)\otimes \sum_{j=1}^n\left[\cfrac{\partial \gamma}{\partial \mathbf{m}}\right]_{j,k,l,\cdot} \nonumber\\
    &-\cfrac{1}{\Gamma_k}\sum_{i=1}^n\gamma_{i,k}\left\{\tau_{1,2}\left((x_i - \FEM_{m_k})\otimes \left[\cfrac{\partial \FEM_{\mathbf{m}}}{\partial \mathbf{m}}\right]_{k,\cdot,l,\cdot}\right) + (x_i - \FEM_{m_k})\otimes\left[\cfrac{\partial \FEM_{\mathbf{m}}}{\partial \mathbf{m}}\right]_{k,\cdot,l,\cdot}\right\}. \label{eqn:dFSigmadm}
\end{align}
In \cref{eqn:dFSigmadm}, the transposed term comes from the differential of the
outer product \cref{eqn:outer_prod_diff_same}, where $\tau_{1,2}(A)_{i,j,\cdots}
= A_{j,i,\cdots}$. For the differential with respect to $\mathbf{\Sigma}$, we
use the same method as in \cref{eqn:dFSigmadm}:
\begin{align}
    \cfrac{\partial \FEM_\mathbf{\Sigma}}{\partial \mathbf{\Sigma}}(\theta, X) &\in \L\left((\Sym)^K, (\Sym)^K\right) \hookrightarrow \R^{K\times d\times d\times K\times d\times d}:\\
    \left[\cfrac{\partial \FEM_\mathbf{\Sigma}}{\partial \mathbf{\Sigma}}(\theta, X)\right]_{k,\cdot, \cdot, l, \cdot, \cdot} &= \cfrac{1}{\Gamma_k}\sum_{i=1}^n(x_i - \FEM_{m_k})(x_i - \FEM_{m_k})^\top \otimes \left[\cfrac{\partial \gamma}{\partial \mathbf{\Sigma}}\right]_{i,k,l,\cdot, \cdot} \nonumber\\
    &- \cfrac{1}{\Gamma_k^2}\left(\sum_{i=1}^n\gamma_{i,k}(x_i - \FEM_{m_k})(x_i - \FEM_{m_k})^\top\right)\otimes \sum_{j=1}^n\left[\cfrac{\partial \gamma}{\partial \mathbf{\Sigma}}\right]_{j,k,l,\cdot, \cdot} \nonumber\\
    &-\cfrac{1}{\Gamma_k}\sum_{i=1}^n\gamma_{i,k}\left\{\tau_{1,2}\left((x_i - \FEM_{m_k})\otimes \left[\cfrac{\partial \FEM_{\mathbf{m}}}{\partial \mathbf{\Sigma}}\right]_{k,\cdot,l,\cdot, \cdot}\right) + (x_i - \FEM_{m_k})\otimes\left[\cfrac{\partial \FEM_{\mathbf{m}}}{\partial \mathbf{\Sigma}}\right]_{k,\cdot,l,\cdot, \cdot}\right\}. \label{eqn:dFSigmadSigma}
\end{align}
Finally, for the differential with respect to $X$, the method is almost
identical, with an additional term due to the presence of $x_i$ in $(x_i -
\FEM_{m_k})(x_i - \FEM_{m_k})^\top$.
\begin{align}
    \cfrac{\partial \FEM_\mathbf{\Sigma}}{\partial X}(\theta, X) &\in \L\left(\R^{n\times d}, (\Sym)^K\right) \hookrightarrow \R^{K\times d\times d\times n\times d}:\\
    \left[\cfrac{\partial \FEM_\mathbf{\Sigma}}{\partial X}(\theta, X)\right]_{k,\cdot, \cdot, i, \cdot} &= \cfrac{1}{\Gamma_k}\sum_{h=1}^n(x_h - \FEM_{m_k})(x_h - \FEM_{m_k})^\top \otimes \left[\cfrac{\partial \gamma}{\partial X}\right]_{h,k,i,\cdot} \nonumber\\
    &- \cfrac{1}{\Gamma_k^2}\left(\sum_{h=1}^n\gamma_{h,k}(x_h - \FEM_{m_k})(x_h - \FEM_{m_k})^\top\right)\otimes \sum_{j=1}^n\left[\cfrac{\partial \gamma}{\partial X}\right]_{j,k,i,\cdot} \nonumber\\
    &-\cfrac{1}{\Gamma_k}\sum_{h=1}^n\gamma_{h,k}\left\{\tau_{1,2}\left((x_h - \FEM_{m_k})\otimes \left[\cfrac{\partial \FEM_{\mathbf{m}}}{\partial X}\right]_{k,\cdot,i,\cdot}\right) + (x_h - \FEM_{m_k})\otimes\left[\cfrac{\partial \FEM_{\mathbf{m}}}{\partial X}\right]_{k,\cdot,i,\cdot}\right\}\nonumber \\
    &+\cfrac{\gamma_{i,k}}{\Gamma_k}\Biggl\{\tau_{1,2}\Bigl((x_i - \FEM_{m_k})\otimes I_d\Bigr) + (x_i - \FEM_{m_k})\otimes I_d\Biggr\}. \label{eqn:dFSigmadu}
\end{align}

\subsection{Local Minima in (GMM)-OT}\label{sec:fixweights_computations}

\subsubsection{Computation of the Local Minima of the Discrete 2-Wasserstein
Distance}\label{sec:local_minima_w2_computations}

We provide an explicit local minimum of the energy $\Ethree$ defined in
\cref{eqn:W2_local_minimum_energy}. We illustrate the setup in
\cref{fig:w2_local_minimum_setup}.

\begin{figure}[ht]
    \centering
    \includegraphics[width=0.5\textwidth]{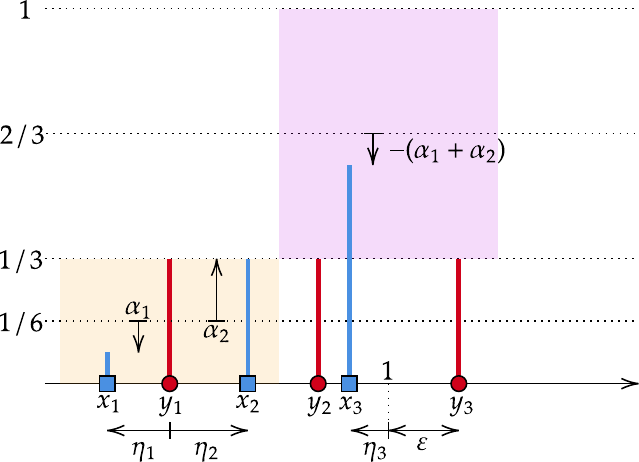}
    \caption{Setup for the local minimum of $\W_2^2(\mu_{\alpha, \eta}, \nu)$.
    The support of $\mu$ is represented with blue squares, and its weights with
    vertical blue lines. For the target $\nu$, its support is red squares and
    its weights red lines. We consider a specific region where the points $x_1 =
    \eta_1$ and $x_2 = \eta_2$ stay within $(-\frac{1}{2}, \frac{1}{2})$ and the
    weights $a_1 = \frac{1}{6} + \alpha_1$ and $a_2 = \frac{1}{6} + \alpha_2$
    stay within $[0, \frac{1}{3}]$, as represented by the orange rectangle.
    Likewise, the point $x_3 = 1+\eta_3$ must stay within $(\frac{1}{2},
    \frac{3}{2})$ and its weight $a_3 = \frac{2}{3} - (\alpha_1 + \alpha_2)$
    must stay in $[\frac{1}{3}, 1]$, as shown with the purple rectangle.}
    \label{fig:w2_local_minimum_setup}
\end{figure}

To compute the energy $\Ethree(\alpha,
\eta)$ at $[-\frac{1}{6}, \frac{1}{6}]^2\times (-\frac{1}{2}, \frac{1}{2})^3$,
we must distinguish the two cases $\alpha_1+\alpha_2 \geq 0$ and
$\alpha_1+\alpha_2 < 0$, which yield two different OT plans between
$\mu_{\alpha, \eta}$ and $\nu$. We will show separately that the energy
$\Ethree$ has a strict local minimum at $(0_{\R^2}, 0_{\R^3})$ in both
sub-regions. To break the symmetry in $(\alpha_1, \alpha_2, \eta_1, \eta_2)
\gets (\alpha_2, \alpha_1, \eta_2, \eta_1)$, we will impose the constraint
$\eta_1\leq\eta_2$ in the optimisation problem. To summarise, we split the
optimisation problem in two as follows:
\begin{equation}\label{eqn:W2_local_minimum_energy_split}
    \underset{\substack{\alpha \in [-\frac{1}{6}, \frac{1}{6}]^2 \\ 
    \eta \in (-\frac{1}{2}, \frac{1}{2})^3}}{\min}\ \Ethree(\alpha, \eta) 
    =  \min(\Ethree^+, \Ethree^-),\quad 
    \Ethree^+ := \underset{\substack{\alpha \in [-\frac{1}{6}, \frac{1}{6}]^2 \\ 
    \eta \in (-\frac{1}{2}, \frac{1}{2})^3 \\ \alpha_1 + \alpha_2 \geq 0 \\ 
    \eta_1 \leq \eta_2}}{\min}\ \Ethree(\alpha, \eta),\quad
    \Ethree^- := \underset{\substack{\alpha \in [-\frac{1}{6}, \frac{1}{6}]^2 \\ 
    \eta \in (-\frac{1}{2}, \frac{1}{2})^3 \\ \alpha_1 + \alpha_2 \leq 0 \\ 
    \eta_1 \leq \eta_2}}{\min}\ \Ethree(\alpha, \eta).
\end{equation}

\paragraph{Case $\alpha_1+\alpha_2 \geq 0$} For any $(\alpha, \eta) \in
[-\frac{1}{6}, \frac{1}{6}]^2\times (-\frac{1}{2}, \frac{1}{2})^3$ such that
$\alpha_1+\alpha_2 \geq 0$ and $\eta_1\leq\eta_2$, the optimal plan between
$\mu_{\alpha, \eta}$ and $\nu$ is given by:
$$\pi^+(\alpha, \eta) = \left(\begin{array}{ccc} \frac{1}{6} + \alpha_1 & 0 & 0
    \\
    \frac{1}{6} - \alpha_1 & \alpha_1 + \alpha_2 & 0 \\
    0 & \frac{1}{3} -\alpha_1 - \alpha_2 & \frac{1}{3} \end{array}\right),$$ and
thus the energy $\Ethree(\alpha, \eta)$ is given by:
\begin{equation}\label{eqn:EW_plus}
    \Ethree(\alpha, \eta) = \left(\tfrac{1}{6} + \alpha_1\right)\eta_1^2 
    + \left(\tfrac{1}{6} - \alpha_1\right)\eta_2^2 
    + \left(\alpha_1+\alpha_2\right)\left(1-\varepsilon-\eta_2\right)^2 
    + \left(\tfrac{1}{3} - \alpha_1 - \alpha_2\right)
    \left(\eta_3 +\varepsilon\right)^2
    + \tfrac{1}{3}(\eta_3-\varepsilon)^2.
\end{equation}
Taking dual variables $\lambda_1, \lambda_2 \in \R$, the Lagrangian of the
problem $\Ethree^+$ defined in \cref{eqn:W2_local_minimum_energy_split} is given
by:
\begin{equation}\label{eqn:W2_local_minimum_lagrangian_case1}
    \mathcal{L}^+(\alpha, \eta, \lambda_1, \lambda_2) =
    \Ethree(\alpha, \eta) + \lambda_1(-\alpha_1-\alpha_2) 
    + \lambda_2(\eta_1-\eta_2).
\end{equation}
To find solutions of the problem $\Ethree^+$, we study the necessary KKT system
(see \cite[Section 5.5.3]{boyd2004convex}, or \cite[Theorem
12.1]{nocedal2006numerical}). We begin with the stationarity condition,
expressed at the point $(\alpha, \eta) \in [-\frac{1}{6}, \frac{1}{6}]^2\times
(-\frac{1}{2}, \frac{1}{2})^3$:
\begin{align*}
    0 &= \cfrac{\partial \mathcal{L}^+}{\partial \alpha_1} 
    = \eta_1^2 - \eta_2^2 + (1-\varepsilon-\eta_2)^2 
    - (\eta_3+\varepsilon)^2 -\lambda_1;\\
    0 &= \cfrac{\partial \mathcal{L}^+}{\partial \alpha_2}
    = (1-\varepsilon-\eta_2)^2 - (\eta_3+\varepsilon)^2 - \lambda_1;\\
    0 &= \cfrac{\partial \mathcal{L}^+}{\partial \eta_1}
    = 2\left(\tfrac{1}{6} + \alpha_1\right)\eta_1 + \lambda_2;\\
    0 &= \cfrac{\partial \mathcal{L}^+}{\partial \eta_2}
    = 2\left(\tfrac{1}{6} - \alpha_1\right)\eta_2 - 2(\alpha_1+\alpha_2)
    (1-\varepsilon-\eta_2) - \lambda_2;\\
    0 &= \cfrac{\partial \mathcal{L}^+}{\partial \eta_3}
    = 2\left(\tfrac{1}{3} - \alpha_1 - \alpha_2\right)(\eta_3+\varepsilon) 
    + \tfrac{2}{3}(\eta_3-\varepsilon).
\end{align*}
The remaining KKT conditions are the primal / dual feasibility conditions, which
are given by:
$$\alpha_1+\alpha_2 \geq 0,\; \eta_1\leq \eta_2,\; \lambda_1 \geq 0,\;
\lambda_2\geq 0,$$ and the complementary slackness conditions:
$$\lambda_1(-\alpha_1-\alpha_2) = 0,\; \lambda_2(\eta_1-\eta_2) = 0.$$ It is
easy to see that the point $(\alpha, \eta, \lambda_1, \lambda_2) = (0_{\R^2},
0_{\R^3}, 1-2\varepsilon, 0)$ is a solution of the KKT system (note that
$\varepsilon < \tfrac{1}{2}$).

We now check that our critical point is a local minimum. For $\alpha \in
[-\frac{1}{6}, \frac{1}{6}]^2,\; \eta \in (-\frac{1}{2}, \frac{1}{2})^3$
verifying $\alpha_1+\alpha_2\geq 0$, we compute the Hessian of $\Ethree$:
\begin{equation}
    H^+(\alpha, \eta) = \left(\begin{matrix}0 & 0 & 2 \eta_{1} & 2 \varepsilon - 2 & - 2 \eta_{3} - 2 \varepsilon\\0 & 0 & 0 & 2 \eta_{2} + 2 \varepsilon - 2 & - 2 \eta_{3} - 2 \varepsilon\\2 \eta_{1} & 0 & 2 \alpha_{1} + \frac{1}{3} & 0 & 0\\2 \varepsilon - 2 & 2 \eta_{2} + 2 \varepsilon - 2 & 0 & 2 \alpha_{2} + \frac{1}{3} & 0\\- 2 \eta_{3} - 2 \varepsilon & - 2 \eta_{3} - 2 \varepsilon & 0 & 0 & - 2 \alpha_{1} - 2 \alpha_{2} + \frac{4}{3}\end{matrix}\right).
\end{equation}
Using numerical solvers, we obtain that at $(\alpha, \eta) = (0_{\R^2},
0_{\R^3})$, the Hessian verifies the following property:
\begin{equation}\label{eqn:prop_H_plus_spectrum}
    H^+(0_{\R^2}, 0_{\R^3}) v = 0,\; v:= (1, -1, 0, 0, 0),\; 
    \forall w \in v^\perp,\; w^\top H^+(0_{\R^2}, 0_{\R^3}) w > 0.
\end{equation}
Adding the vector $tv$ to $(\alpha, \eta)$ for $|t|$ small enough corresponds to
adding $t$ to $\alpha_1$ and subtracting $t$ from $\alpha_2$, and we notice that
at $(\alpha, \eta) = (0_{\R^2}, 0_{\R^3})$, since $\eta_1=\eta_2$, this
operation does not change the cost: $\Ethree((t, -t), 0_{\R^3}) =
\Ethree(0_{\R^2}, 0_{\R^3})$. We conclude that $(0_{\R^2}, 0_{\R^3})$ is a local
minimum for the problem $\Ethree^+$ defined in
\cref{eqn:W2_local_minimum_energy_split}.

\paragraph{Case $\alpha_1+\alpha_2 \leq 0$} For any $(\alpha, \eta) \in
[-\frac{1}{6}, \frac{1}{6}]^2\times (-\frac{1}{2}, \frac{1}{2})^3$ such that
$\alpha_1+\alpha_2 \leq 0$ and $\eta_1\leq\eta_2$, the optimal plan between
$\mu_{\alpha, \eta}$ and $\nu$ and the associated energy $\Ethree$ are given by:
$$\pi^-(\alpha, \eta) = \left(\begin{array}{ccc} \frac{1}{6} + \alpha_1 & 0 & 0
    \\
    \frac{1}{6} + \alpha_2 & 0 & 0 \\
    -\alpha_1-\alpha_2 & \frac{1}{3} & \frac{1}{3} \end{array}\right),$$
\begin{equation*}
    \Ethree(\alpha, \eta) = \left(\tfrac{1}{6} + \alpha_1\right)\eta_1^2 
    + \left(\tfrac{1}{6} + \alpha_2\right)\eta_2^2 
    - \left(\alpha_1+\alpha_2\right)\left(1+\eta_3\right)^2 
    + \tfrac{1}{3}\left(\eta_3 +\varepsilon\right)^2
    + \tfrac{1}{3}(\eta_3-\varepsilon)^2.
\end{equation*}
We deduce the expression of the Lagrangian of the problem $\Ethree^-$ defined in
\cref{eqn:W2_local_minimum_energy_split}:
\begin{equation}\label{eqn:W2_local_minimum_lagrangian_case2}
    \mathcal{L}^-(\alpha, \eta, \lambda_1, \lambda_2) =
    \Ethree(\alpha, \eta) + \lambda_1(\alpha_1+\alpha_2)
    + \lambda_2(\eta_1-\eta_2).
\end{equation}
The KKT stationarity condition writes then:
\begin{align*}
    0 &= \cfrac{\partial \mathcal{L}^-}{\partial \alpha_1} 
    = \eta_1^2 - (1+\eta_3)^2 +\lambda_1;\\
    0 &= \cfrac{\partial \mathcal{L}^-}{\partial \alpha_2}
    = \eta_2^2 - (1+\eta_3)^2 +\lambda_1;\\
    0 &= \cfrac{\partial \mathcal{L}^-}{\partial \eta_1}
    = 2\left(\tfrac{1}{6} + \alpha_1\right)\eta_1 + \lambda_2;\\
    0 &= \cfrac{\partial \mathcal{L}^-}{\partial \eta_2}
    = 2\left(\tfrac{1}{6} - \alpha_1\right)\eta_2 - \lambda_2;\\
    0 &= \cfrac{\partial \mathcal{L}^-}{\partial \eta_3}
    = -2\left(\alpha_1 + \alpha_2\right)(1+\eta_3) 
    + \tfrac{4}{3}\eta_3.
\end{align*}
Again the primal / dual feasibility conditions read:
$$\alpha_1+\alpha_2 \leq 0,\; \eta_1\leq \eta_2,\; \lambda_1 \geq 0,\;
\lambda_2\geq 0,$$ and the complementary slackness conditions:
$$\lambda_1(\alpha_1+\alpha_2) = 0,\; \lambda_2(\eta_1-\eta_2) = 0.$$ Likewise,
the point $(\alpha, \eta, \lambda_1, \lambda_2) = (0_{\R^2}, 0_{\R^3}, 1, 0)$ is
a solution of the KKT system.

To prove local minimality, we compute the Hessian of $\Ethree$ at $\alpha \in
[-\frac{1}{6}, \frac{1}{6}]^2,\; \eta \in (-\frac{1}{2}, \frac{1}{2})^3$
verifying $\alpha_1+\alpha_2\leq 0$:
\begin{equation*}
    H^-(\alpha, \eta) = \left(\begin{matrix}0 & 0 & 2 \eta_{1} & 0 & - 2 \eta_{3} - 2\\0 & 0 & 0 & 2 \eta_{2} & - 2 \eta_{3} - 2\\2 \eta_{1} & 0 & 2 \alpha_{1} + \frac{1}{3} & 0 & 0\\0 & 2 \eta_{2} & 0 & 2 \alpha_{2} + \frac{1}{3} & 0\\- 2 \eta_{3} - 2 & - 2 \eta_{3} - 2 & 0 & 0 & - 2 \alpha_{1} - 2 \alpha_{2} + \frac{4}{3}\end{matrix}\right).
\end{equation*}
The property of $H^+(0_{\R^2}, 0_{\R^3})$ from \cref{eqn:prop_H_plus_spectrum}
is also satisfied by $H^-(0_{\R^2}, 0_{\R^3})$, and we conclude likewise that
$(0_{\R^2}, 0_{\R^3})$ is a local minimum for the problem $\Ethree^-$ defined in
\cref{eqn:W2_local_minimum_energy_split}.

We conclude that $(\alpha, \eta) = (0_{\R^2}, 0_{\R^3})$ is a minimum of
$\Ethree$ on the set $[-\frac{1}{6}, \frac{1}{6}]^2\times (-\frac{1}{2},
\frac{1}{2})^3$, with value $\Ethree(0_{\R^2}, 0_{\R^3}) =
\tfrac{2}{3}\varepsilon^2 > 0$ (use \cref{eqn:EW_plus} for example).

\subsubsection{Discrete 2-Wasserstein Distance for \texorpdfstring{$n=2$}{n=2}:
Proof of Unique Local Minimum}\label{sec:unique_local_minimum_w2_computations}

\paragraph{Unique Local Minimum of the Discrete 2-Wasserstein Distance for
\texorpdfstring{$n=2$}{n=2}}\label{sec:unique_local_min_w2_n2} We consider the
minimisation of the quadratic Wasserstein cost $\mu\longmapsto\W_2^{2}(\mu,\nu)$
when $\mu$ is restricted to two Dirac atoms,
\[
\mu_{x, \alpha}=\alpha\delta_{x_1}+(1-\alpha)\delta_{x_2},
\qquad (x,\alpha)\in\R^{2d}\times [0,1],
\]
against the fixed target $\nu=\gamma\delta_{y_1}+(1-\gamma)\delta_{y_2}$ with
$\gamma\in (0, 1)$ and $y_1\neq y_2$.  
Denoting
\begin{equation}\label{eqn:W2_2_minima_energy}
    \Etwo(x,\alpha):=\W_2^{2}(\mu_{x, \alpha},\nu),
\end{equation}
we show  that inside the domain $\mathcal{D}:=\{(x,\alpha):x_1\neq
x_2,\;\alpha\in(0,1)\}$ every local minimiser of $\Etwo$ is such that $ \mu_{x,
\alpha} = \nu$, and is thus a global minimiser. We show this result more
generally, for costs $c(x, y):=\phi(x-y)$ where $\phi$ is convex,
$\mathcal{C}^1$, $\phi(v)=0\iff v=0$ and $\nabla\phi(v)=0\iff v=0$.
Additionally, boundary stationary points can occur when atoms merge, i.e.
$\alpha\in\{0,1\}$ or $x_1=x_2$; they are local but not global minima. When
these weights arise from EM, we can prevent such degeneracies by fixing weights
or imposing a hard lower bound on them throughout the iterations.

We now determine all the local minimisers of the energy $\Etwo$ defined in
\cref{eqn:W2_2_minima_energy} in the domain $\mathcal{D}:=\{(x,\alpha):x_1\neq
x_2,\;\alpha\in(0,1)\}$. Fix $\gamma\in(0, 1)$, and let $x_{1}$, $x_{2}$,
$y_{1}$ and $y_{2}$ in $\mathbb{R}^{d}$, with $y_{1}\ne y_{2}$. For $\alpha \in
(0, 1)$, we consider the measures:
\[
\mu(x_1, x_2, \alpha) := \alpha\delta_{x_{1}}+(1-\alpha)\delta_{x_{2}},
\quad\nu:=\gamma\delta_{y_{1}}+(1-\gamma)\delta_{y_{2}}.
\]
Consider the ground cost $c := (x, y) \in \R^d\times\R^d \longmapsto \phi(x-y)$,
with:
\begin{itemize}
\item (H1) $\phi:\mathbb{R}^{d}\rightarrow\mathbb{R}_{+}$ convex and
$\mathcal{C}^{1}$,
\item (H2) $\phi(v)=0\iff v=0$,
\item (H3) $\nabla\phi(v)=0\iff v=0$.
\end{itemize}
For instance, $\phi$ could be $\lVert\cdot\rVert_{p}^{p}$ with $p>1$. Every
feasible coupling between $\mu(x_1, x_2, \alpha)$ and $\nu$ can be written under
the form:
\[
\pi(\alpha, t) := \begin{pmatrix}\alpha-t & t\\
\gamma-\alpha+t & 1-\gamma-t
\end{pmatrix},\; t\in\left[t_-(\alpha),t_+(\alpha)\right],\; t_-(\alpha) := \max(0, \alpha-\gamma),\; t_+(\alpha) := \min(\alpha,1-\gamma).
\]
The interval $\left[t_-(\alpha),t_+(\alpha)\right]$ is non-empty for all $\alpha
\in (0, 1)$, since $\gamma \in (0, 1)$. The transport cost associated to a plan
$\pi(\alpha, t)$ writes:
\begin{equation}
    \mathcal{F}(x, \alpha, t) := (\alpha - t)c(x_1, y_2) 
    + (\gamma - \alpha + t) c(x_2, y_2) + tc(x_1, y_2) 
    + (1-\gamma-t)c(x_2, y_2).
\end{equation}
Since $\Etwo(x_1, x_2, \alpha) = \min_{t \in [t_-(\alpha),
t_+(\alpha)]}\mathcal{F}(x_1, x_2, \alpha, t)$ and that $\mathcal{F}$ is linear
in $t$, we obtain:
\begin{equation}
    \begin{gathered}
        \Etwo(x_1, x_2, \alpha) = \min\left(\mathcal{E}^{-}(x_1, x_2, \alpha), 
        \mathcal{E}^{+}(x_1, x_2, \alpha)\right), \\
        \mathcal{E}^{-}(x_1, x_2, \alpha) := 
        \mathcal{F}(x_1, x_2, \alpha, t_{-}(\alpha)),\;
        \mathcal{E}^{+}(x_1, x_2, \alpha) := 
        \mathcal{F}(x_1, x_2, \alpha, t_{+}(\alpha)).
    \end{gathered}
\end{equation}
In particular, any local optimum of $\Etwo$ in $\mathcal{D}:=\{(x_1, x_2,
\alpha) \in \R^d\times\R^d\times (0, 1):\; x_{1}\ne x_{2}\}$ must be a local
optimum of $\mathcal{E}^{-}$ or $\mathcal{E}^{+}$. Straightforward computation
yields the following symmetrical expression:
\[
\forall x_1, x_2 \in \R^d,\; \forall \alpha \in (0, 1),\; \forall t \in \R,\; 
\mathcal{F}(x_1, x_2, \alpha, t) = \mathcal{F}(x_2, x_1, 1-\alpha, 1-\gamma-t),
\]
which can be understood as exchanging the roles of $x_1$ and $x_2$. For any
$\alpha \in (0, 1)$, we have $1-\gamma-t_-(\alpha) = t_+(\alpha)$, concluding a
symmetrical relationship between $\mathcal{E}^-$ and $\mathcal{E}^+$:
\begin{equation}\label{eqn:twocomp_symmetry}
    \forall x_1, x_2 \in \R^d,\; \forall \alpha \in (0, 1),\; 
    \mathcal{E}^+(x_2, x_1, 1-\alpha) = \mathcal{E}^-(x_1, x_2, \alpha).
\end{equation}
As a consequence, any local optimum $(x_1, x_2, \alpha) \in \mathcal{D}$ of
$\mathcal{E}^+$ is such that $(x_2, x_1, 1-\alpha)$ is a local optimum of
$\mathcal{E}^-$, and conversely. To determine the local minima of $\Etwo$ in
$\mathcal{D}$, we can focus on the local minima of $\mathcal{E}^-$ in
$\mathcal{D}$. We split $\mathcal{D}$ intro three sub-regions wherein
$\mathcal{E}^-$ has an explicit expression: consider
\[
\mathcal{R}_{<}:=\{0<\alpha<\gamma\}\cap\mathcal{D},\quad 
\mathcal{R}_{=}:=\{\alpha=\gamma\}\cap\mathcal{D},\quad 
\mathcal{R}_{>}:=\{\gamma<\alpha<1\}\cap\mathcal{D},
\]
we have the following expressions for $\mathcal{E}^-$ at $(x_1, x_2, \alpha) \in
\mathcal{D}$:
\begin{equation}
    \mathcal{E}^{-}(x_1, x_2, \alpha)=
    \begin{cases}
        \alpha c(x_{1}, y_{1})+(\gamma-\alpha) c(x_{2}, y_{1})
        +(1-\gamma) c(x_{2}, y_{2}) 
        & \text{if } (x_1, x_2, \alpha) \in \mathcal{R}_< \cup \mathcal{R}_=\\
        \gamma c(x_{1}, y_{1})+(\alpha-\gamma) c(x_{1}, y_{2})
        +(1-\alpha) c(x_{2}, y_{2}) 
        & \text{if } (x_1, x_2, \alpha) \in \mathcal{R}_= \cup \mathcal{R}_>
    \end{cases}.\label{eq:twocomp}
\end{equation}

\paragraph{No local minimum in $\mathcal{R}_{<}$}

For $(x_1, x_2, \alpha) \in \mathcal{R}_{<}$, consider
\[
\mathcal{E}_{<}^{-}(x_{1}, x_{2}, \alpha):=\alpha c(x_{1}, y_{1})+(\gamma-\alpha) c(x_{2}, y_{1})+(1-\gamma) c(x_{2}, y_{2}).
\]
A local optimum $(x_1, x_2, \alpha) \in \mathcal{R}_{<}$ of $\mathcal{E}^{-}$
must satisfy the stationarity conditions:
\begin{align*}
    0 &= \frac{\partial\mathcal{E}_{<}^{-}}{\partial x_{1}}=
    \alpha \nabla \phi(x_{1} - y_{1})\qquad
    \underset{\text{(H3)}}{\implies}\qquad x_{1}=y_{1},\\
    0 &= \frac{\partial\mathcal{E}_{<}^{-}}{\partial\alpha}=
    \underbrace{c(x_{1}, y_{1})}_{=0\text{ by (H2)}}-c(x_{2}, y_{1})
    \qquad\underset{\text{(H2)}}{\implies}\qquad x_{2}=y_{1}.\\
    0 &= \frac{\partial\mathcal{E}_{<}^{-}}{\partial x_{2}}
    =(\gamma-\alpha)\underbrace{\nabla \phi(x_{2} - y_{1})}_{=0\text{ by (H2)}}
    +(1-\gamma) \nabla \phi(\underbrace{x_{2}}_{=y_{1}} - y_{2})
    =(1-\gamma) \nabla \phi(y_{1} - y_{2}).
\end{align*}
By (H3), since $y_{1}\ne y_{2}$, we have $\nabla c(y_{1}, y_{2})\ne0$, hence
$\mathcal{E}_{-}$ has no local minimum in $\mathcal{R}_{<}$.

\paragraph{No local minimum in $\mathcal{R}_{>}$}

For $(x_1, x_2, \alpha) \in \mathcal{R}_{>}$, consider
\[
\mathcal{E}_{>}^{-}(x_{1}, x_{2}, \alpha):=\gamma c (x_{1}, y_{1})+(\alpha-\gamma) c(x_{1}, y_{2})+(1-\alpha) c(x_{2}, y_{2}).
\]
Likewise, a local minimum $(x_1, x_2, \alpha) \in \mathcal{R}_{>}$ of
$\mathcal{E}^{-}$ must verify:
\begin{align*}
    0 &= \frac{\partial\mathcal{E}_{>}^{-}}{\partial x_{1}}=
    (1-\alpha) \nabla \phi(x_{2} - y_{2})\qquad
    \underset{\text{(H3)}}{\implies}\qquad x_{2}=y_{2},\\
    0 &= \frac{\partial\mathcal{E}_{>}^{-}}{\partial\alpha}=
    c(x_{1}, y_{2})-\underbrace{c(x_{2}, y_{2})}_{=0\text{ by (H2)}}
    \qquad\underset{\text{(H2)}}{\implies}\qquad x_{1}=y_{2}.\\
    0 &= \frac{\partial\mathcal{E}_{>}^{-}}{\partial x_{2}}
    =\gamma\nabla \phi(\underbrace{x_{1}}_{=y_{2}} - y_{1})
    +(\alpha-\gamma) \underbrace{\nabla \phi(x_{2} - y_{2})}_{=0\text{ by (H2)}}
    =\gamma \nabla \phi(y_{2} - y_{1}) \ne 0.
\end{align*}
As before, this system cannot hold, and thus $\mathcal{E}_{-}$ has no local
minimum in $\mathcal{R}_{>}$.

\paragraph{Analysis on $\mathcal{R}_=$} For $(x_1, x_2, \alpha) \in
\mathcal{R}_{=}$, we have:
\[
\mathcal{E}^-(x_1, x_2, \alpha) = \mathcal{E}_=^-(x_1, x_2) 
:= \gamma c(x_{1}, y_{1})+(1-\gamma)c(x_{2}, y_{2}),
\]
thus a local minimum $(x_1, x_2, \alpha) \in \mathcal{R}_{=}$ of
$\mathcal{E}^{-}$ must satisfy the stationarity conditions:
\begin{align*}
    0 &= \frac{\partial\mathcal{E}_{=}^{-}}{\partial x_{1}}=
    \gamma \nabla \phi(x_{1} - y_{1})\qquad
    \underset{\text{(H3)}}{\implies}\qquad x_{1}=y_{1},\\
    0 &= \frac{\partial\mathcal{E}_{=}^{-}}{\partial x_{2}}
    =(1-\gamma)\nabla \phi(x_{2} - y_{2})\qquad
    \underset{\text{(H3)}}{\implies}\qquad x_{2}=y_{2},
\end{align*}
Since $(x_1, x_2, \alpha) = (y_1, y_2, \gamma)$ is a global minimum of
$\mathcal{E}_{=}^{-}$, we conclude that the only local minimum of
$\mathcal{E}^-$ in $\mathcal{D}$ is $(y_1, y_2, \gamma)$.

Using \cref{eqn:twocomp_symmetry}, we deduce that the only local minimum of
$\mathcal{E}^+$ in $\mathcal{D}$ is $(y_2, y_1, 1-\gamma)$, which is a global
minimum of $\Etwo$ in $\mathcal{D}$. Finally, there are two local minima of
$\Etwo$ in $\mathcal{D}$, which are the two global minima corresponding to
$\mu(x_1, x_2, \alpha) = \nu$.

\paragraph{Local minimum for a single-Dirac source, and how to avoid it}

The constraint $(x_1, x_2, \alpha) \in \mathcal{D}$ imposes in particular that
$\mu(x_1, x_2, \alpha)$ is composed of two distinct Dirac masses. We now focus
on the pathological case where $\alpha=0$, showing the existence of a local
optimum. For simplicity, we consider $d=1$ and the cost $c(x,y):=|x-y|^2$. Set 
\[
z^{\star}:=(x_{1}^{\star}, x_{2}^{\star}, \alpha^{\star}):=(-1, 1-\gamma, 0),\; y_{1}:=0,\; y_{2}:=1,\; \gamma \in (0, 1).
\]
For $(x_1, x_2, \alpha)$ in an open vicinity of $z^{\star}$, the cost simplifies
to
\[
\mathcal{E}_2(x_{1}, x_{2}, \alpha)=\alpha x_{1}^{2}
+(\gamma-\alpha)x_{2}^{2}+(1-\gamma)(x_2-1)^{2}.
\]
To show that $z^{\star}$ is a local minimum, we compute the gradient of
$\mathcal{E}_2$ at $z^{\star}$:
\[
\frac{\partial\mathcal{E}_2}{\partial x_{1}}(z^{\star})=2\alpha^{\star}x_{1}^{\star}=0,
\]
\[
\frac{\partial\mathcal{E}_2}{\partial x_{2}}(z^{\star})=2\left(x_{2}^{\star}-(1-\gamma)\right)=0,
\]
\[
\frac{\partial\mathcal{E}_2}{\partial\alpha}(z^{\star})=x_{1}^{2}-x_{2}^{2}=1-(1-\gamma)^{2}>0.
\]
Consider a perturbation $h:=(h_{1}, h_{2}, h_{3}) \in \R^3$ with $h_{3}>0$ (as
$\alpha\ge0$). For $z:=z^{\star}+h$, we have:
\[
\mathcal{E}_2(z)=\mathcal{E}_2(z^{\star})+\underbrace{\frac{\partial\mathcal{E}_2}{\partial\alpha}(z^{\star})h_{3}}_{>0}+\O(\lVert h\rVert^{2}),
\]
so every feasible perturbation increases the objective. Thus $z^{\star}$ is a
strict local minimiser, but is not global. When $\alpha$ arises as a mixture
weight in EM, two remedies are possible: fixing uniform weights, or enforcing
$\alpha \ge\varepsilon>0$: both methods sidestep the degenerate local minimum
above.

\subsubsection{Essential Stationary Points for the MW2-EM Loss:
Computations}\label{sec:vanishing_gradients_EM_MW_computations}

We illustrate a phenomenon of points where the gradients are infinitesimally
small on a simple example with two Gaussian components, and studying a variant
of the EM algorithm that does not update the covariances for simplicity. We fix
$\varepsilon > 0$ and the following parametrised input GMM:
\[
\forall \alpha \in (0, 1),\; \forall m_1, m_2 \in \R,\; 
\mu(\alpha, m_1, m_2) :=
\alpha\mathcal{N}(m_1, \varepsilon^2) 
+ (1-\alpha)\mathcal{N}(m_2, \varepsilon^2),
\]
and for $w := \tfrac{2}{3}$, the target GMM: $\nu := w\mathcal{N}(0,
\varepsilon^2) + (1-w)\mathcal{N}(1, \varepsilon^2)$. We consider the particular
dataset $X_\varepsilon \in \R^{6\times 1}$ defined by:
\[
X_\varepsilon := (x_1, \cdots, x_6),\; x_1 := -\varepsilon,\; 
x_2 := \varepsilon,\; x_3 := x_5 := m^\star - \varepsilon,\; 
x_4 := x_6 := m^\star + \varepsilon.
\]
We define the GMM $\mu^\star := (1-w)\mathcal{N}(0, \varepsilon^2) +
w\mathcal{N}(m^\star, \varepsilon^2)$ associated to the vanishing gradient
point, and introduce $\theta^\star := ((1-w), 0, m^\star)$ its parameters. For
any $X \in \R^{6\times 1}$, we denote by $\theta(X)$ the parameters of the GMM
fitted by the EM algorithm in one iteration on $X$ with initialisation
$\theta^\star$: $\theta(X) := \FEM(\theta^\star, X)$ (we remind that in this
section, we consider a simplified EM that does not update covariances). We shall
first see that $\theta(X_\varepsilon)\approx \theta^\star$, then that the energy
\[
\EEMMW := X \longmapsto \MW_2^2(\mu(\theta(X)), \nu),\quad \mu(\theta(X)) := \alpha(X)\mathcal{N}(m_1(X), \varepsilon^2) + (1-\alpha(X))\mathcal{N}(m_2(X), \varepsilon^2),
\]
verifies $\partial_X \EEMMW(X_\varepsilon) \approx 0$. We summarise our setting
in \cref{fig:computations_vanishing_gradients_EM_MW}.

\begin{figure}[ht]
    \centering
    \includegraphics[width=0.6\textwidth]{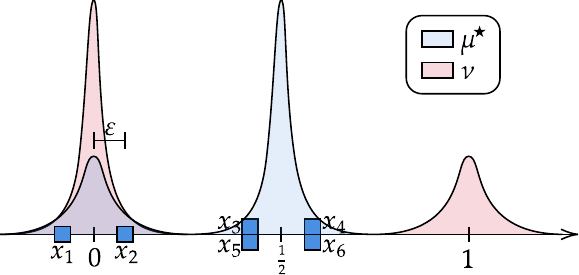}
    \caption{With the data $X_\varepsilon := (x_1, \cdots, x_6)$, one iteration
    of the EM algorithm initialised at $\theta^\star$ (corresponding to the GMM
    $\mu^\star$) yields approximately the same parameters $\theta^\star$. We
    shall see that the gradient of the energy $\EEMMW$ at $X_\varepsilon$ is
    approximately zero, illustrating the vanishing gradient phenomenon.}
    \label{fig:computations_vanishing_gradients_EM_MW}
\end{figure}

\paragraph{Showing that $\theta(X_\varepsilon) \approx \theta^\star$} Using the
expressions of the EM update from \cref{eqn:gamma_ik}, we obtain that
responsibilities $\gamma(X_\varepsilon)$ computed with initialisation
$\theta^\star$ verify:
\[
\forall i \in \{1, 2\},\; \gamma_{i, 1}(X_\varepsilon) = 1 + \O(e^{-1/\varepsilon^2}),\; \forall i \in \{3, 4, 5, 6\},\; \gamma_{i, 2}(X_\varepsilon) = 1 + \O(e^{-1/\varepsilon^2}),
\]
as $\varepsilon \longrightarrow 0^+$. This means that the two points $x_1, x_2$
are considered as belonging to the first component $\mathcal{N}(0,
\varepsilon^2)$ of $\mu^\star$, and the four points $x_3, x_4, x_5, x_6$ to the
second component $\mathcal{N}(m^\star, \varepsilon^2)$ of $\mu^\star$. As for
the weight $\alpha(X_\varepsilon)$ and means $m(X_\varepsilon)$ of the GMM
$\FEM(\theta^\star, X^\star)$, we use \cref{eqn:theta_next} to deduce that:
\begin{equation}\label{eqn:theta_vanish_gradients_EM_MW}
    \alpha(X_\varepsilon) = (1-w) + \O(e^{-1/\varepsilon^2}),\; 
    m_1(X_\varepsilon) = 0 + \O(e^{-1/\varepsilon^2}),\; 
    m_2(X_\varepsilon) = m^\star + \O(e^{-1/\varepsilon^2}).
\end{equation}
In this sense, we have $\theta(X_\varepsilon) \approx \theta^\star$ as
$\varepsilon \longrightarrow 0^+$.

\paragraph{Vanishing gradient of the energy $\EEMMW$ at $X_\varepsilon$} For
$\varepsilon > 0$ sufficiently small and $X$ in a sufficiently small open
vicinity of $X_\varepsilon$, it follows by regularity of $F$ that
$\theta(X)=(\alpha(X), m_1(X), m_2(X))$ will be sufficiently close to
$\theta(X_\varepsilon)$ to ensure that $\alpha(X) < w$ and that $m_1(X) <
m_2(X)$, which yields (by property on one-dimensional OT) the following
expression for the energy:
\[
\EEMMW(X) := \MW_2^2(\mu(\theta(X)), \nu) = \mathcal{F}(\alpha(X), m_1(X), m_2(X)),
\]
where $\mathcal{F}:(0, 1)\times\R\times\R\longrightarrow \R$ is the function
defined by:
\[
\forall (\alpha, m_1, m_2) \in (0, 1)\times\R\times\R,\;
\mathcal{F}(\alpha, m_1, m_2) := \alpha m_1^2 + (w - \alpha)m_2^2 
+ (1-w)(m_2-1)^2.
\]
To determine the gradient of $\EEMMW$ at $X_\varepsilon$, we use the chain rule:
\[
\nabla \EEMMW(X_\varepsilon) =
\cfrac{\partial \mathcal{F}}{\partial \alpha}(\theta(X_\varepsilon))
\cfrac{\partial \alpha}{\partial X}(X_\varepsilon)
+ \cfrac{\partial \mathcal{F}}{\partial m_1}(\theta(X_\varepsilon))
\cfrac{\partial m_1}{\partial X}(X_\varepsilon)
+ \cfrac{\partial \mathcal{F}}{\partial m_2}(\theta(X_\varepsilon))
\cfrac{\partial m_2}{\partial X}(X_\varepsilon).
\]
Differentiating $\mathcal{F}$ and evaluating at $\theta(X_\varepsilon)$ which
verifies the properties given in \cref{eqn:theta_vanish_gradients_EM_MW} yields:
\[
\cfrac{\partial \mathcal{F}}{\partial \alpha}(\theta(X_\varepsilon)) = -m^\star + \O(e^{-1/\varepsilon^2}) = \O(1),\;
\cfrac{\partial \mathcal{F}}{\partial m_1}(\theta(X_\varepsilon)) 
= \O(e^{-1/\varepsilon^2}),\;
\cfrac{\partial \mathcal{F}}{\partial m_2}(\theta(X_\varepsilon)) 
= \O(e^{-1/\varepsilon^2}).
\]
Using the expression of $\tfrac{\partial \alpha}{\partial X}$ from
\cref{eqn:dFwdu} with $X:=X_\varepsilon$ and $\theta := \theta^\star$ yields
$\tfrac{\partial \alpha}{\partial X}(X_\varepsilon) =
\O(\varepsilon^{-2}e^{-1/\varepsilon^2})$. 

With the expressions of $\tfrac{\partial m_1}{\partial X}$ and $\tfrac{\partial
m_2}{\partial X}$ computed in \cref{eqn:dFmdu}, we obtain $\tfrac{\partial
m_1}{\partial X}(X_\varepsilon) = \O(1)$ and $\tfrac{\partial m_2}{\partial
X}(X_\varepsilon) = \O(1)$. Putting everything together, we conclude that the
gradient vanishes:
\[
\nabla \EEMMW(X_\varepsilon) = \O(\varepsilon^{-2}e^{-1/\varepsilon^2}).
\]

\sammod{
\subsubsection{Source Weights Fixed to the Target Weights}
\label{sec:fixed_target_weights}

When the source and target GMMs have the same number of components, a natural
idea is to 
keep the weights fixed to the target weights in the source EM steps. However,
doing so can result in unwanted local minima, as shown by the following example.
\begin{example}
    \label{ex:fixed_target_weights_local_minimum}
    Let $p\in(\frac{1}{2},1)$, set $q:=1-p$ and fix $\sigma>0$. For
    $\mu_{x}:=p\mathcal{N}(x_{1},\sigma^{2})+q\mathcal{N}(x_{2},\sigma^{2})$,
    and $\nu:=p\mathcal{N}(0,\sigma^{2})+q\mathcal{N}(1,\sigma^{2})$, the map
    $x\mapsto\mathrm{MW}_{2}^{2}(\mu_{x},\nu)$ has a strict non-global local
    minimum at $x^{\star}=(\frac{q}{p},0)$, with value $\frac{q(p-q)}{p}>0$. Its
    global minimum has value zero at $x=(0,1)$.
\end{example}
\begin{proof}
    As the Gaussian components have the same variance, their pairwise squared
    $\W_{2}$ costs reduce to squared distances between their means. 
    We focus on the open region
    $\mathcal{R}_{>}:=\{(x_{1},x_{2})\in\mathbb{R}^{2},\;x_{1}>x_{2}\}$. In the
    region $\mathcal{R}_{>}$, 
    \[
    \mathrm{MW}_{2}^{2}(\mu_{x},\nu)=qx_{2}^{2}+(p-q)x_{1}^{2}+q(x_{1}-1)^{2}.
    \]
    At $x^{\star}$, this quadratic objective has gradient zero and Hessian
    $2\mathrm{diag}(p,q)\succ0$. As $x^{\star}$ lies in the interior of the
    region $\mathcal{R}_{>}$, it is a strict local minimum of this function.
    Direct substitution gives
    $\mathrm{MW}_{2}^{2}(\mu_{x^{\star}},\nu)=\frac{q(p-q)}{p}>0$, whereas
    $\mu_{(0,1)}=\nu$.
\end{proof}

\begin{figure}[ht]
    \centering
    \begin{subfigure}[b]{0.48\textwidth}
        \centering
        \resizebox{\textwidth}{!}{%
            \begin{tikzpicture}[
    font=\small,
    >=Latex,
    source/.style={
        circle,
        draw=black!65,
        fill=Tomato,
        line width=.45pt
    },
    target/.style={
        regular polygon,
        regular polygon sides=3,
        draw=black!65,
        fill=DodgerBlue,
        line width=.45pt
    },
    transport/.style={
        ->,
        draw=black!42,
        line cap=round
    },
    masslabel/.style={
        fill=white,
        fill opacity=.92,
        text opacity=1,
        rounded corners=1pt,
        inner sep=1.2pt
    }
]
\definecolor{plotink}{RGB}{45,52,58}

\begin{scope}[x=5.15cm,y=1.62cm]
    \node[anchor=east, text=Tomato!78!black, font=\bfseries]
        at (-.055,3.02) {source};
    \node[anchor=east, text=DodgerBlue!72!black, font=\bfseries]
        at (-.055,.92) {target};

    % Transport is drawn first so that the components remain visually dominant.
    \draw[transport, line width=1.35pt]
        (0,2.82) -- node[masslabel, left=2pt] {$q$} (0,1.13);
    \draw[transport, line width=.85pt]
        (.6667,2.82) to[bend right=23]
        node[masslabel, pos=.50, above left=-1pt] {$p-q$} (0,1.13);
    \draw[transport, line width=1.35pt]
        (.6667,2.82) to[bend left=22]
        node[masslabel, pos=.50, above right=-1pt] {$q$} (1,1.13);

    \node[source, minimum size=6.8mm] at (0,3.02) {};
    \node[source, minimum size=9.2mm] at (.6667,3.02) {};
    \node[above=8pt, text=Tomato!78!black] at (0,3.02) {$q$};
    \node[above=10pt, text=Tomato!78!black] at (.6667,3.02) {$p$};

    \node[target, minimum size=9.2mm] at (0,.92) {};
    \node[target, minimum size=6.8mm] at (1,.92) {};
    \node[below=9pt, text=DodgerBlue!72!black] at (0,.92) {$p$};
    \node[below=9pt, text=DodgerBlue!72!black] at (1,.92) {$q$};

    \draw[->, draw=plotink, line width=.55pt] (-.13,.08) -- (1.15,.08);
    \foreach \x/\lab in {0/0,.6667/{q/p},1/1}{
        \draw[draw=plotink, line width=.45pt] (\x,.01) -- (\x,.15);
        \node[below=2pt, text=plotink] at (\x,.01) {$\lab$};
    }
    \node[anchor=west, text=plotink] at (1.15,.08) {$x$};
\end{scope}
\end{tikzpicture}%
        }
        \subcaption{Configuration at the local minimum}
    \end{subfigure}
    \hfill
    \begin{subfigure}[b]{0.48\textwidth}
        \centering
        \includegraphics[width=\textwidth]{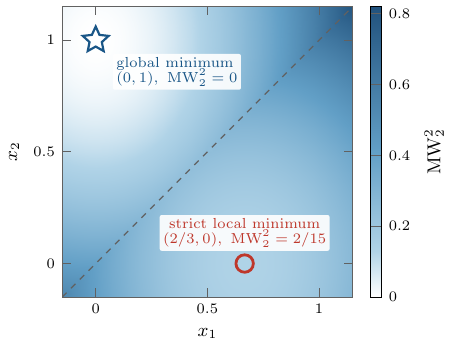}
        \subcaption{Objective landscape}
    \end{subfigure}
    \caption{Local minimum in
    \cref{ex:fixed_target_weights_local_minimum} with fixed non-uniform weights.}
    \label{fig:fixed_target_weights_local_minimum}
\end{figure}

We compare the impact of fixing source weights to target weights or to uniform
weights in the colour-transfer experiment in
\cref{fig:colour_transfer_fixed_weights_comparison}. In this experiment, the
colour transfer appears to be slightly more faithful to the target when the
source weights are fixed to uniform weights.

\begin{figure}[ht]
    \centering
    \begin{subfigure}[t]{0.24\textwidth}
        \centering
        \includegraphics[width=\textwidth]{figures/color_transfer_source.jpg}
        \subcaption{Source image}
    \end{subfigure}
    \begin{subfigure}[t]{0.24\textwidth}
        \centering
        \includegraphics[width=\textwidth]{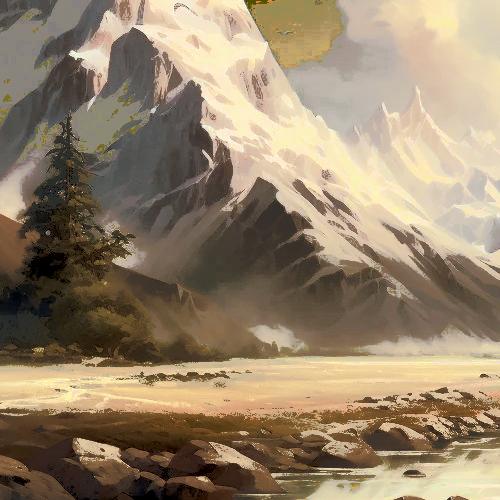}
        \subcaption{Fixed to target weights}
    \end{subfigure}
    \begin{subfigure}[t]{0.24\textwidth}
        \centering
        \includegraphics[width=\textwidth]{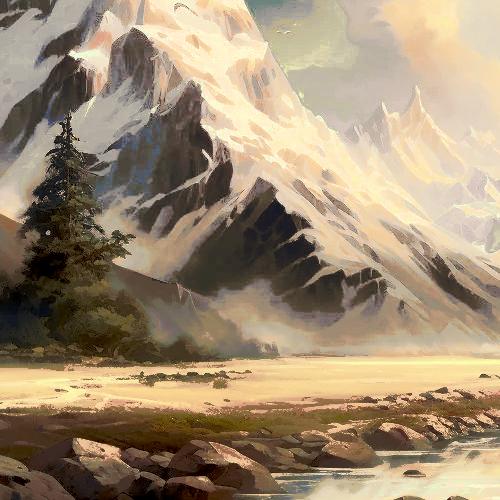}
        \subcaption{Fixed uniform weights}
    \end{subfigure}
    \begin{subfigure}[t]{0.24\textwidth}
        \centering
        \includegraphics[width=\textwidth]{figures/color_transfer_target.jpg}
        \subcaption{Target image}
    \end{subfigure}
    \caption{Impact of the treatment of source weights in colour transfer.}
    \label{fig:colour_transfer_fixed_weights_comparison}
\end{figure}
}

\subsection{Differentiating the Matrix Square Root}\label{sec:diff_matrix_sqrt}

We consider the (differentiable) matrix square root function:
\begin{equation*}
    R := \app{\PD}{\PD}{A}{\sqrt{A}},
\end{equation*}
and provide a formula for its differential which is useful for numerical
automatic differentiation.
\begin{prop}\label{eqn:diff_matrix_sqrt} Let $A\in\PD$ and $H\in\Sym$. Then the
    differential of the matrix square root at $A$ in the direction $H$ is given
    by the following matrix in $\Sym$:
    \begin{equation*}
        \dd_A R(H) = PGP^\top,\quad [G]_{i,j} 
        := \cfrac{[P^\top H P]_{i,j}}{\sqrt{\lambda_i} + \sqrt{\lambda_j}},
    \end{equation*}
    where the orthonormal decomposition of $A$ is given by $A =
    P\diag(\lambda_1, \cdots, \lambda_d)P^\top$.    
\end{prop}
\begin{proof}
    For any $A\in \PD$, we have by definition $R(A)R(A) = A$, and
    differentiating this identity at $A$ in the direction $H$ yields that $\dd_A
    R(H)$ is a solution of the following Sylvester equation:
    \begin{equation}\label{eqn:sylvester_sqrt}
        R(A) X + X R(A) = H.
    \end{equation}
    By \cite[Theorem VII.2.1]{bhatia2013matrix}, \cref{eqn:sylvester_sqrt} has a
    unique solution in $\R^{d\times d}$, which is therefore $\dd_A R(H)$. Using
    the notation of the result statement, we consider the symmetric matrix $X :=
    PGP^\top$ and notice that by definition of $G$, we have:
    \[
    \forall i,j \in \llbracket 1, d \rrbracket,\; 
    \sqrt{\lambda_i} G_{i,j} + G_{i,j}\sqrt{\lambda_j} = [P^\top H P]_{i,j},
    \]
    introducing $D := \diag(\lambda_1, \cdots, \lambda_d)$, we deduce that
    $R(D)G + GR(D) = P^\top H P$ and then that $X$ is indeed a solution of
    \cref{eqn:sylvester_sqrt}, hence by uniqueness $X = \dd_A R(H)$.
\end{proof}

Below, we provide a \href{https://pytorch.org/}{PyTorch} \cite{pytorch}
implementation of this gradient computation, allowing for automatic gradient
propagation.

\begin{python}[caption={PyTorch implementation of the matrix square root gradient}, label={lst:diff_matrix_sqrt}]
import torch

class MatrixSquareRoot(torch.autograd.Function):
    @staticmethod
    def forward(ctx, A):
        A_sym = .5 * (A + A.transpose(-2, -1))
        L, P = torch.linalg.eigh(A_sym)
        S = L.clamp_min(0).sqrt()
        R = (P * S.unsqueeze(-2)) @ P.transpose(-2,-1)
        ctx.save_for_backward(P, S)
        return R

    @staticmethod
    def backward(ctx, H):
        P, S = ctx.saved_tensors
        H_sym = .5 * (H + H.transpose(-2, -1))
        D = S.unsqueeze(-1) + S.unsqueeze(-2)
        G = (P.transpose(-2,-1) @ H_sym @ P) / D
        G = G.masked_fill(D == 0, 0)
        return P @ G @ P.transpose(-2, -1)
\end{python}

\subsection{Experimental Details and Additional Results}\label{sec:exp_details}

\sammod{
\subsubsection{Unbalanced GMM-OT}\label{sec:unbalanced_weight_comparison}

The difficulties of optimising non-uniform mixture weights described in
\cref{sec:fixweights} motivate relaxing the marginal constraints.
Balanced GMM-OT requires the transport plan to match the mixture weights
exactly. The unbalanced formulation of \cref{sec:unbalanced}
penalises marginal discrepancies: increasing the penalty progressively restores
the marginal constraints, providing a transition from the relaxed problem
to balanced OT.

We illustrate this approach using the non-uniform source and target GMMs
of \cref{sec:xp_em_mw_flow_weights}. We compare flows with learnable
weights under balanced OT, unbalanced OT with a fixed penalty, and an
increasing penalty followed by balanced OT. In the unbalanced formulation
\cref{eqn:unbalanced_gmmot}, we take equal penalties
$\lambda_0=\lambda_1=\lambda$.
Since the optimisation objectives differ, all curves report the balanced
cost \cref{eq:mw2}. \Cref{fig:appendix_uot_comparison} shows a lower
final balanced cost for both strategies involving unbalanced OT, with the
smallest value obtained by the strategy where we progressively increase the penalty.

\begin{figure}[ht]
    \centering
    \includegraphics[width=\textwidth]{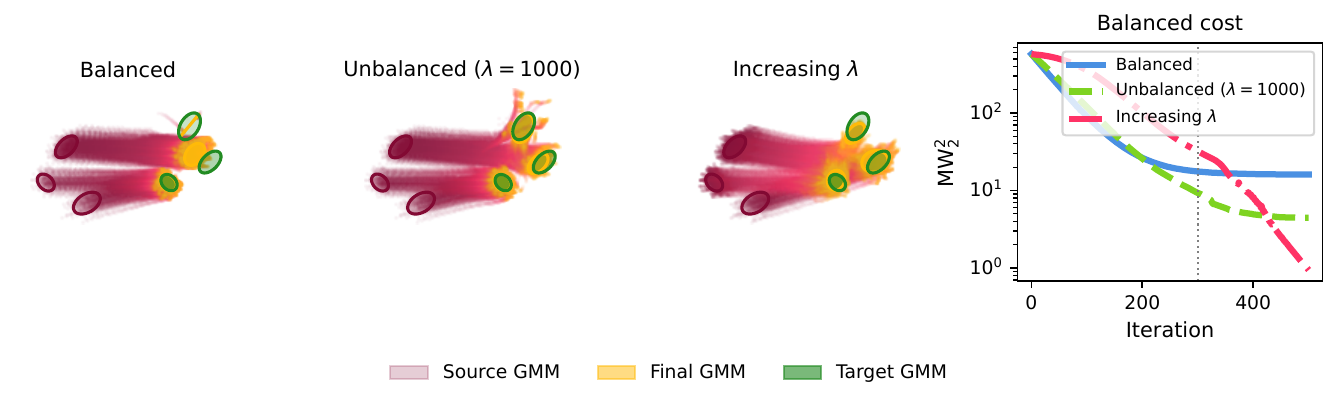}
    \caption{Comparison of balanced and unbalanced particle flows for
    non-uniform weights.}
    \label{fig:appendix_uot_comparison}
\end{figure}
}

\subsubsection{Impact of the Number of Components for the Gradient
Methods}\label{sec:impact_K_grad_comp}

In this section, we present the impact of the parameter $K$ in the setting of
\cref{sec:xp_gradients_comparison}. We fix $n=2000, d=3$ and $T=30$, varying $K
\in \{3, 5, 8, 10, 12, 15\}$ in \cref{fig:impact_K_grad_comp}. The results
suggest that the number of components $K$ greatly impacts the difficulty of the
EM algorithm to converge to a ``stable'' fixed point, incurring worsening
gradient approximations. In certain settings, the spectral norm of the Jacobian
can be \sammod{of the order of} 1, completely invalidating the OS method.
\sammod{The AI method is successul for low values of $K$, however for larger
$K$, neither OS nor AI approximate the AD method satisfactorily.} In this
experiment, we required a regularisation of $10^{-8} I_d$ for the covariances
for numerical stability.

\begin{figure}[ht]
    \centering
    \includegraphics[width=\textwidth]{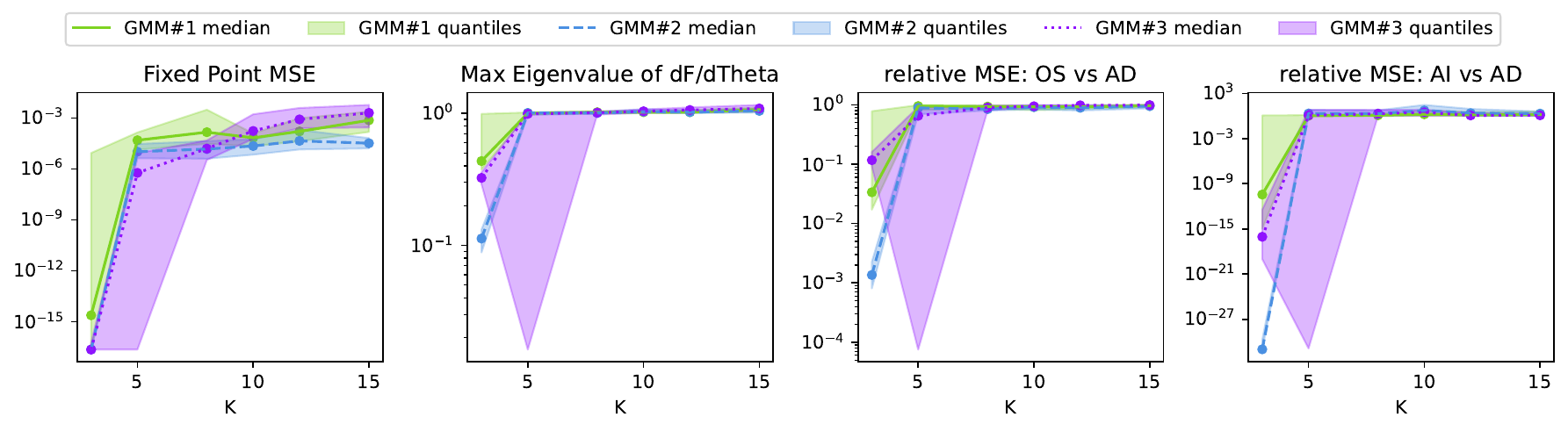}
    \caption{Varying the number of components $K$, we study the convergence of
    EM, the local contractivity of $\FEM$, and the MSEs of the OS and AI
    gradients.}
    \label{fig:impact_K_grad_comp}
\end{figure}

\subsubsection{Barycentres}\label{app:bary}

We consider a barycentre problem similar to \cref{sec:bary_flow}, with more
complex datasets: we take three two-dimensional images $I_1$, $I_2$ and $I_3$,
we randomly sample $n=500$ points $Y_i \in \mathbb{R}^{n\times 2}$ from each. In
\cref{fig:bary_image}, we flow a point cloud initialised as random normal noise,
towards a barycentre of $K=15$ GMMs fitted from $(Y_i)$.

\begin{figure}[htb]
    \centering
    \begin{adjustbox}{valign=c}
        \begin{subfigure}[t]{0.3\textwidth}
            \centering
            \includegraphics[width=\textwidth]{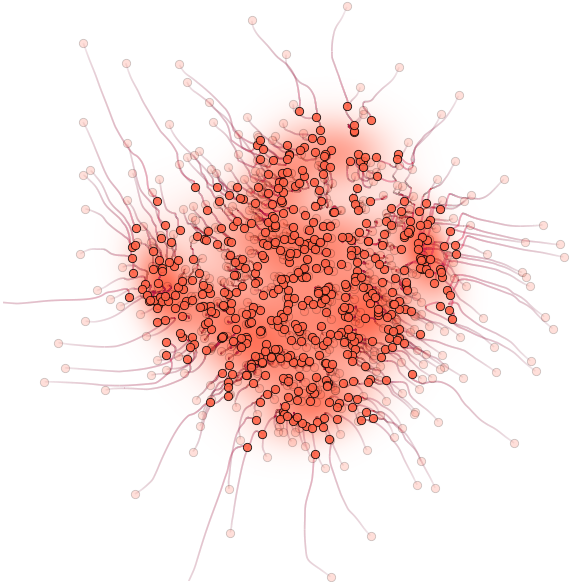}
        \end{subfigure}
    \end{adjustbox}
    \hfill
    \begin{adjustbox}{valign=c}
        \begin{subfigure}[t]{0.24\textwidth}
            \centering
            \includegraphics[width=\textwidth]{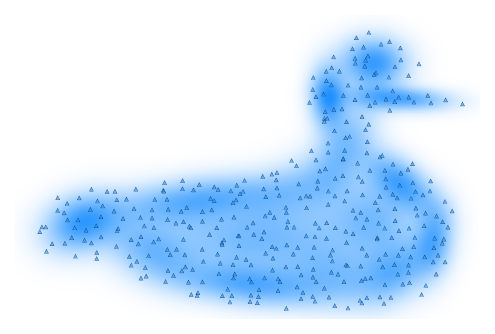}
        \end{subfigure}
    \end{adjustbox}
    \hfill
    \begin{adjustbox}{valign=c}
        \begin{subfigure}[t]{0.22\textwidth}
            \centering
            \includegraphics[width=\textwidth]{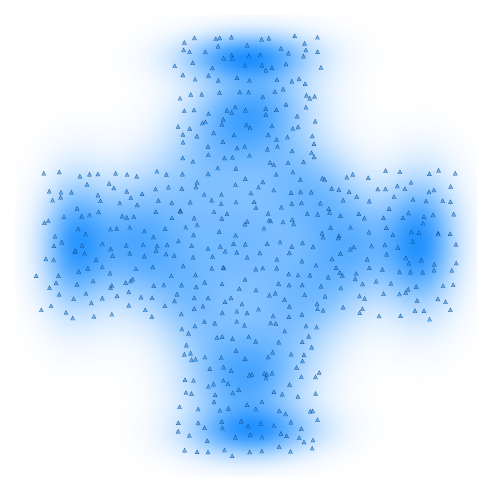}
        \end{subfigure}
    \end{adjustbox}
    \hfill
    \begin{adjustbox}{valign=c}
        \begin{subfigure}[t]{0.20\textwidth}
            \centering
            \includegraphics[width=\textwidth]{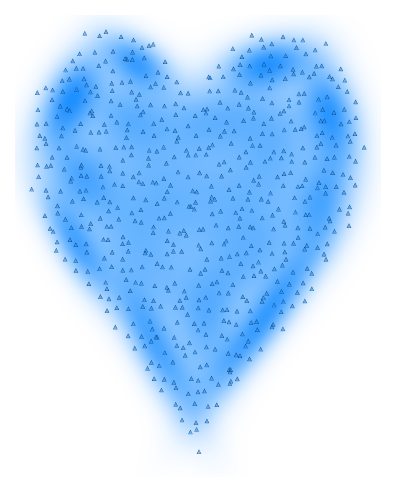}
        \end{subfigure}
    \end{adjustbox}
    \caption{Left: $\mathrm{EM}-\MW_2^2$-Barycentre flow; right: input point
    clouds.}
    \label{fig:bary_image}
\end{figure}

We can also compute a generalised barycentre $X$ in $\mathbb{R}^{n\times 3}$,
such that every projection $P_i(X)\in \mathbb{R}^{n\times 2}$ orthogonal to the
canonical direction $e_i$ coincides with $Y_i$. Specifically, we solve
\[
\min_{X\in\mathbb{R}^{n\times 2}} \sum_{i=1}^3 \MW_2^2\left(P_i\#\mu(F_X^T(\theta_0)),
\nu_i\right).
\]
The GMMs $(\nu_i) \in \GMM_2(40)^3$ are fitted beforehand with the point clouds
$(Y_i)$, and ${\mu}_X$ is the running EM estimation of optimised cloud $X$. The
projected GMM $P_i\#\mu(F_X^T(\theta_0))$ is defined by projecting the means and
covariances of the three-dimensional GMM $\mu(F_X^T(\theta_0))$. We sample
$n=1000$ points and fit $K=100$ Gaussian components in each cloud. Example
results are shown in \cref{fig:gbary}. We obtained very similar results with an
alternative loss which estimates three GMMs in $\R^2$ instead of one in $\R^3$.
In all our barycenter experiments, we set $\varepsilon_r = 10^{-3}$.

\begin{figure}[htb]
    \centering
    \begin{adjustbox}{valign=c}
        \begin{subfigure}[t]{0.3\textwidth}
            \centering
            \includegraphics[height=6cm]{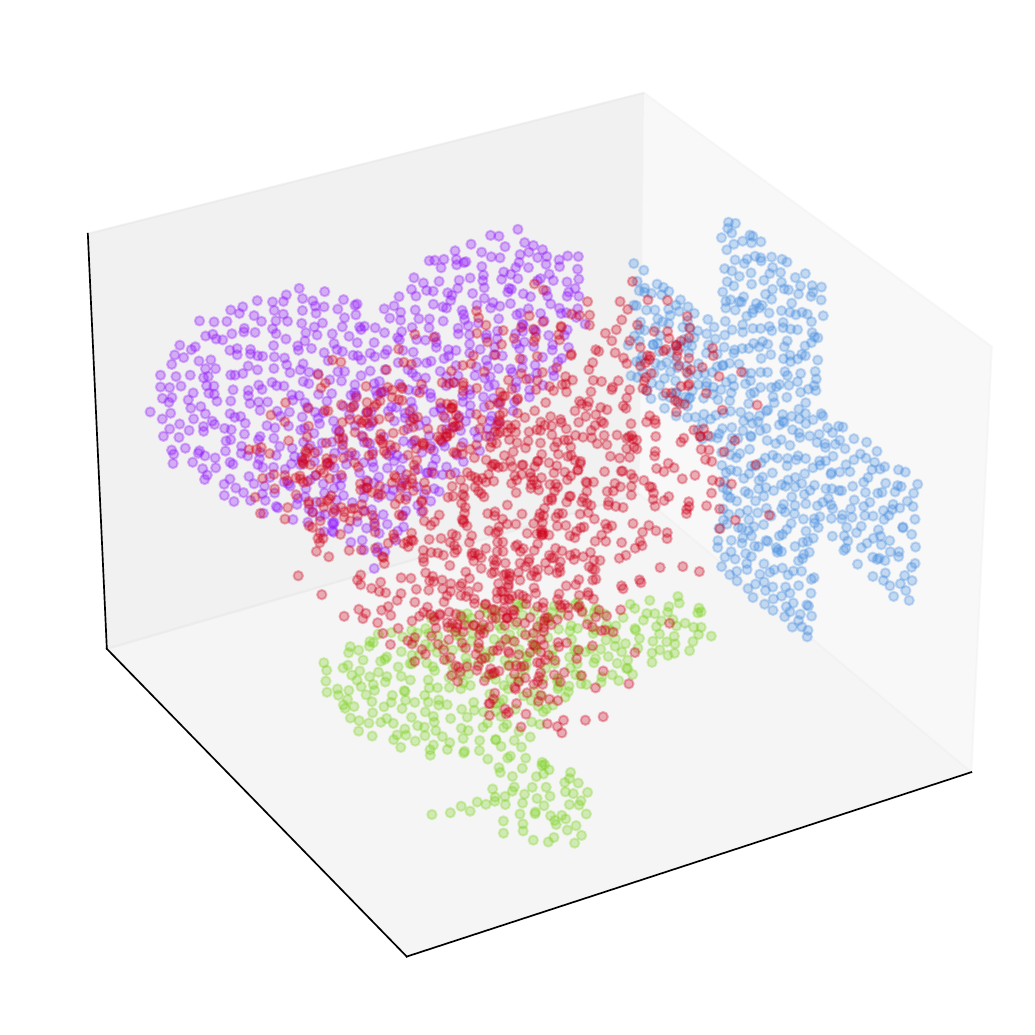}
        \end{subfigure}
    \end{adjustbox}
    \hfill
    \begin{adjustbox}{valign=c}
        \begin{subfigure}[t]{0.65\textwidth}
            \centering
            \includegraphics[height=6cm]{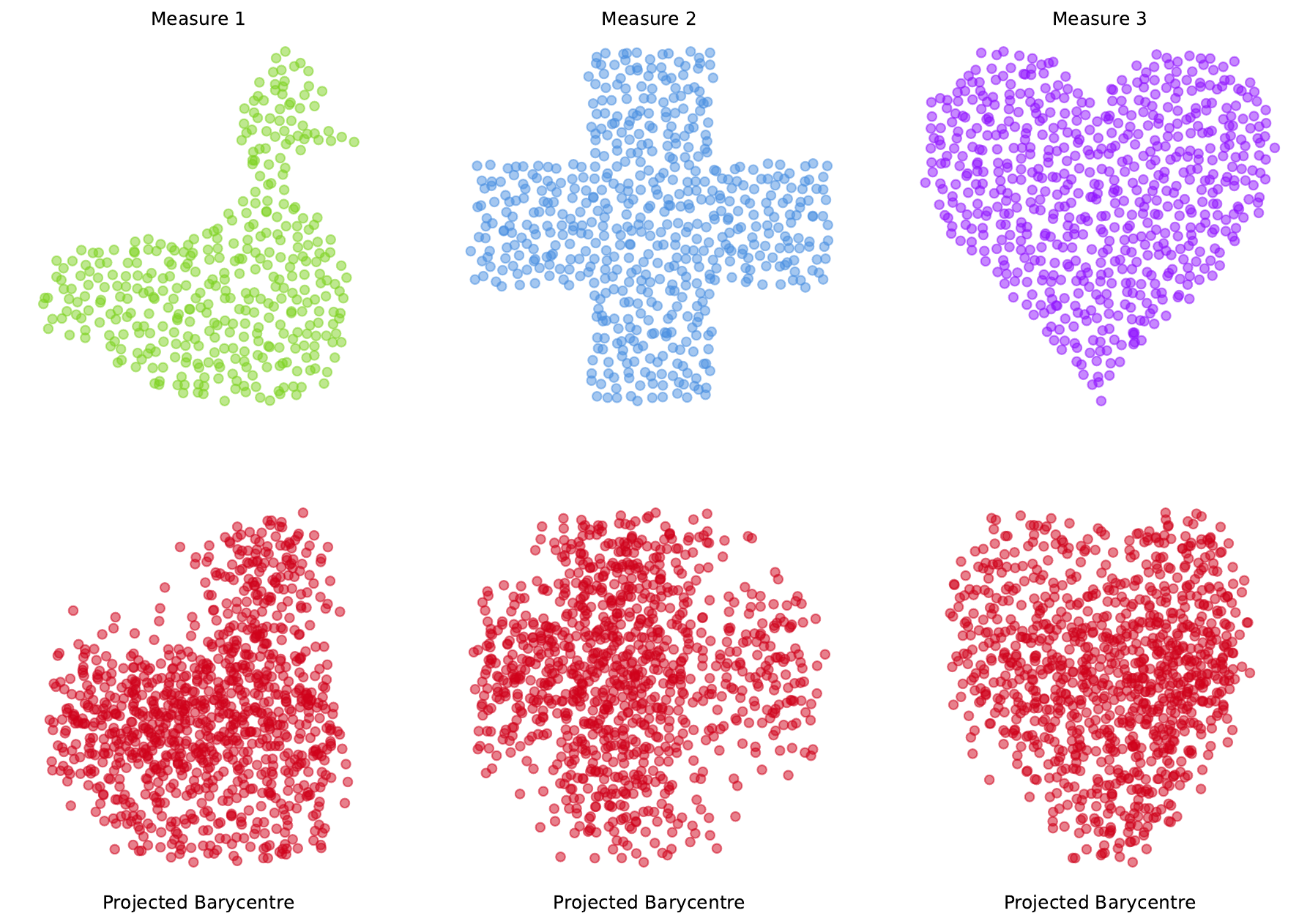}
        \end{subfigure}
    \end{adjustbox}
    \caption{3D barycentre (left) and its projections (right).}
    \label{fig:gbary}
\end{figure}

\subsubsection{Colour Transfer}\label{sec:colour_transfer_details}

\paragraph{Optimiser choice}
We use gradient descent with fixed step size. Indeed, optimisers such as Adam
have per-parameter learning rates that are adapted dynamically. As a result,
points are treated differently and move at different speeds. For instance,
imagine that we independently sample points $x_1,\cdots,x_n$ in $\mathbb{R}^{d}$
following the law $\mathcal{N}(0_d, I_d)$ and want to match them to
$\mathcal{N}(0_d, 2 I_d)$ using the Mixture Loss. Optimisers like Adam might
focus on moving the points on the boundary to make the cloud's variance larger.
Fixed-step gradient descent, on the contrary, will treat each point equally, and
will rescale the cloud as a whole.

\paragraph{Fixing mixture weights}
As discussed in \cref{sec:fixweights}, fixing the mixture weights to be uniform
(using \cref{alg:EM_fw}) avoids local minima, as illustrated in
\cref{fig:fixweights}. For a colour transfer task with $K=6$ components, when
weights are allowed to vary, some components of the optimised mixture (in red)
are trapped between two target components (in blue).

\begin{figure}[ht]
    \centering
    \begin{subfigure}{0.45\textwidth} \centering
    \includegraphics[width=\textwidth, trim=50 20 20 20,
    clip]{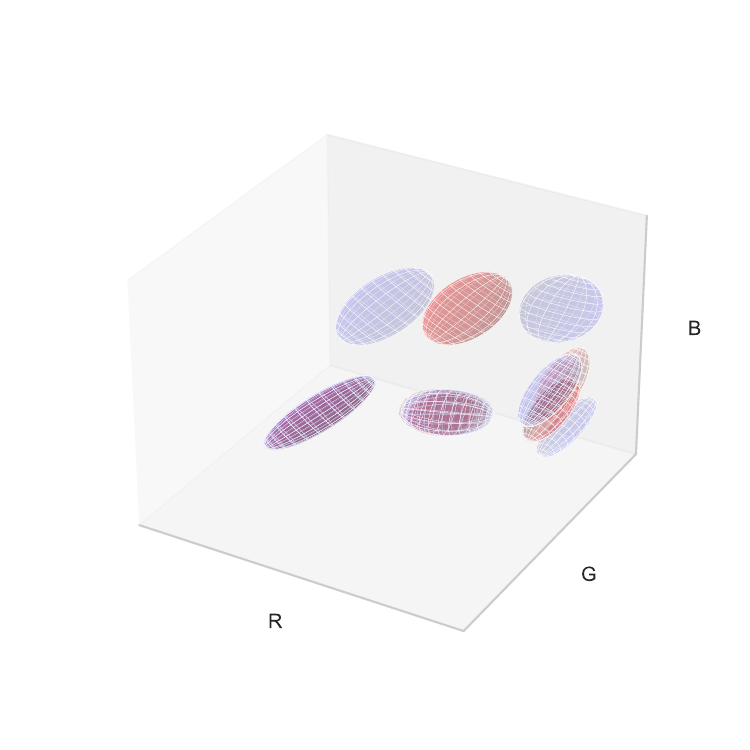}\caption{} \end{subfigure}
    \hfill
    \begin{subfigure}{0.45\textwidth} \centering
    \includegraphics[width=\textwidth, trim=50 20 20 20,
    clip]{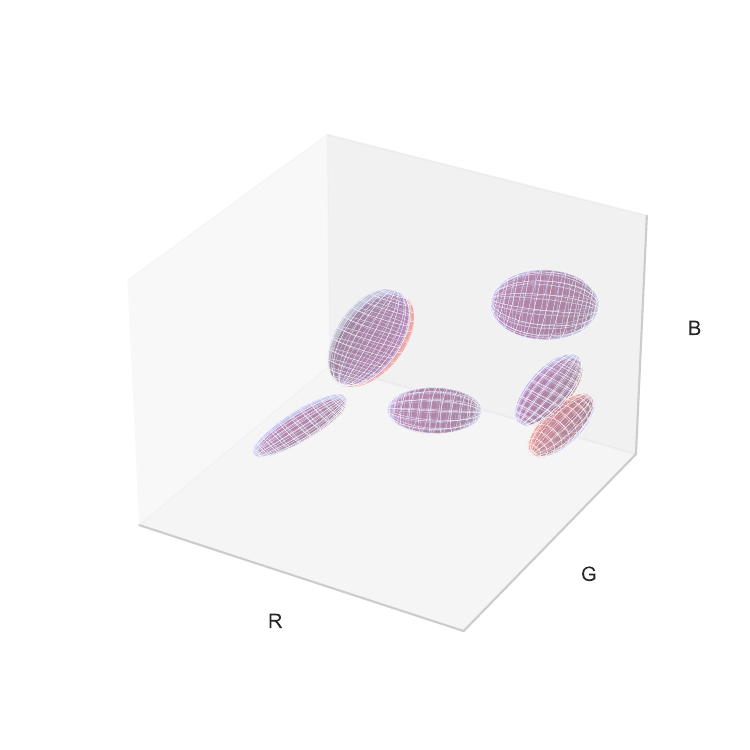}\caption{} \end{subfigure}
    \caption{Final GMMs in RGB space with (a) variable weights and (b) fixed
    weights. Target is in blue and optimised mixture is in red. In (a) we are
    stuck in a local minimum, in (b) we converged.}
    \label{fig:fixweights}
\end{figure}

\paragraph{Choice of the number of components $K$}
We choose $K=10$ Gaussian components in our colour transfer experiments. Taking
$K=1$, i.e. one Gaussian per source and target, is equivalent to applying the
following affine map to source pixels:
\[
T: x\in\mathbb{R}^d \longmapsto m_t + A(x-m_s),
\]
where $m_s$ and $m_t$ are the empirical means of source and target colour
distributions, and where
\[
A = \Sigma_s^{-1/2} \left(\Sigma_s^{1/2} \Sigma_t \Sigma_s^{1/2}\right)^{1/2} \Sigma_s^{-1/2} = A^\top.
\]
This map performs a coarser colour transfer than the optimisation with $K=10$,
as compared in \cref{fig:compare_K}: the contrast is better preserved with
higher values of $K$. We always set $\varepsilon_r = 10^{-3}$ in this
experiment.

\begin{figure}[ht]
  \centering
  \begin{subfigure}{0.32\textwidth} \centering
  \includegraphics[width=\textwidth]{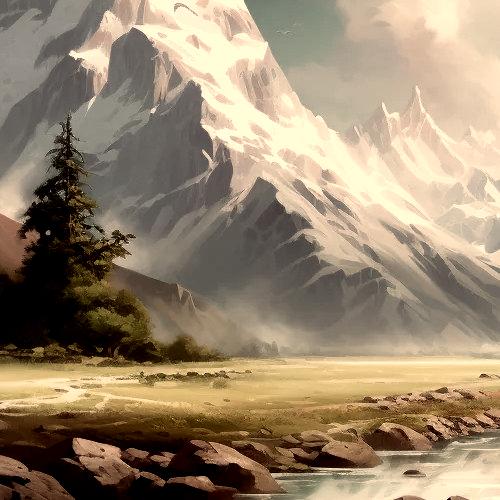}
  \caption{Result with $K=1$}\end{subfigure}
  \hfill
  \begin{subfigure}{0.32\textwidth} \centering
  \includegraphics[width=\textwidth]{figures/color_transfer_result.jpg}
  \caption{Result with $K=10$}\end{subfigure}
  \hfill
  \begin{subfigure}{0.32\textwidth} \centering
  \includegraphics[width=\textwidth]{figures/color_transfer_target.jpg}
  \caption{Target}\end{subfigure}
  \caption{Colour transfer with (a) $K=1$ components and (b) $K=10$ components.}
  \label{fig:compare_K}
\end{figure}

\sammod{\paragraph{Comparison with entropic OT transfer} The advantage of our
method is that it is in $O(n)$, where $n$ is the number of pixels, whereas
Sinkhorn's algorithm is in $O(n^2)$. We can use entropic OT for transferring the
colours of small enough images, for instance using the barycentric transport map
$T(x_i)=\sum_j P_{ij}^{\varepsilon}y_j/\sum_j P_{ij}^{\varepsilon}$ (which can
be evaluated lazily with POT's GeomLoss backend and
KeOps~\cite{charlier2021kernel}, removing the need to store the full transport
matrix). Using every pixel with regularisation $\varepsilon=0.01$  gives a
convincing transfer in $24.5$ seconds on one RTX A4000, compared with $55$
seconds for $\mathrm{MW}_2$ (see Fig.~\ref{fig:colour-transfer-entropic}).
However, for larger images, at about one million pixels, the times we measure on
an H100 are $17.3$ seconds for $\mathrm{MW}_2$ and $98.7$ seconds for entropic
transfer. Entropic barycentric colour transfer with $4096$ quantised colours was
already studied in~\cite[Section~6.1]{blondel2018smooth}. A related
colour-transfer method using the gradient of the Sinkhorn divergence (which
differs from projecting one transport plan) is illustrated
in~\cite[Figure~3.33]{feydy2020geometric}.}

\begin{figure}[ht]
\centering
\begin{tabular}{ccc}
Source & $\mathrm{MW}_2$ & Target \\
\includegraphics[width=0.3\textwidth]{figures/color_transfer_source.jpg} &
\includegraphics[width=0.3\textwidth]{figures/color_transfer_result.jpg} &
\includegraphics[width=0.3\textwidth]{figures/color_transfer_target.jpg} \\
$\varepsilon=1$ & $\varepsilon=0.1$ & $\varepsilon=0.01$ \\
\includegraphics[width=0.3\textwidth]{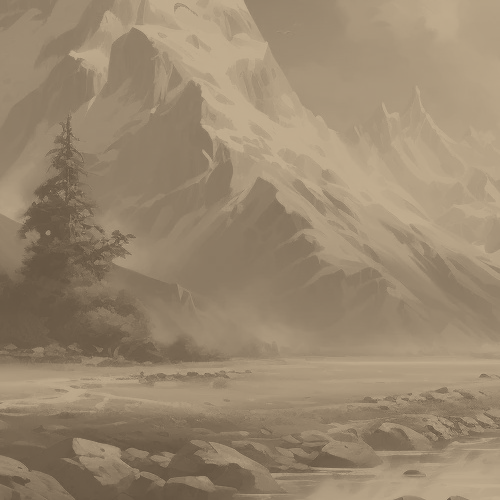} &
\includegraphics[width=0.3\textwidth]{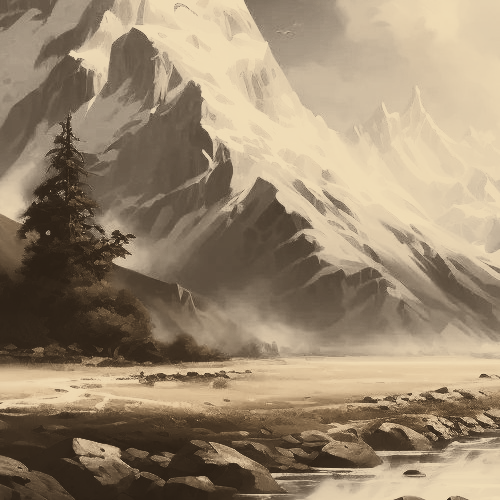} &
\includegraphics[width=0.3\textwidth]{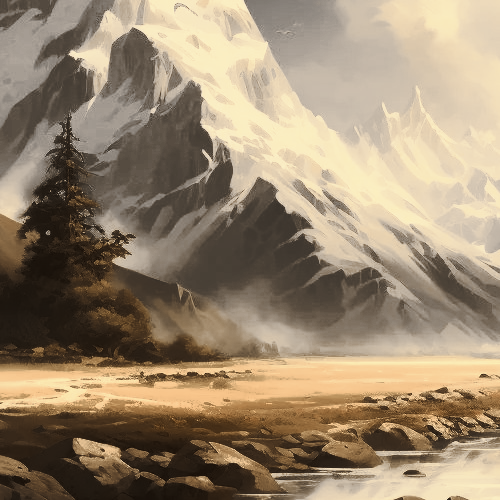}
\end{tabular}
\caption{Colour transfer: barycentric transfer with entropic OT, and our method.}
\label{fig:colour-transfer-entropic}
\end{figure}

\subsubsection{Neural Style Transfer}\label{sec:style_transfer_details}

We note that the optimiser choice and EM variant (between
\cref{alg:EM,alg:EM_fw}) do not change the results qualitatively. In our
experiments, we use Adam and use standard EM \cref{alg:EM}.

\paragraph{Choice of the number of components $K$}
We choose $K=3$ in our experiments. When using $K=1$ Gaussian component, the
style transfer objective in \cref{eq:style_loss} simplifies to
\[
\min_{X\in \mathbb{R}^{3\times H\times W}} \sum_{\ell=1}^{3}\lambda_{\ell} \left(\left\lVert m_{s,\ell}(X) - m_{t,\ell} \right\rVert_2^2 + \dBW^2\left(\Sigma_{s,\ell}(X),\Sigma_{t,\ell}\right)\right),
\]
where $m_{s,\ell}(X)$ and $\Sigma_{s,\ell}(X)$ are the empirical means and
covariances of source features $\text{VGG}_{1\cdots\ell}(X)$. We similarly
define $m_{t,\ell}$ and $\Sigma_{t,\ell}$ for the target image $Y$. As there are
only one source and target Gaussian, there is no need for an EM algorithm, we
simply estimate the means and covariance and apply Gaussian OT. To evaluate the
influence of $K$, we take two content images (the Eiffel tower and Gatys'
\cite{gatys2015texture} picture of Tuebingen), and two style images (\emph{The
Great Wave} and \emph{The Starry Night}). We compute their features
corresponding to first layers $1,\cdots,\ell$ of VGG, for
$\ell\in\{1,\cdots,3\}$. We fit Gaussian mixtures on these features with varying
number of components $K$, and we evaluate corresponding log-likelihoods as a
measure of model quality. Results are presented in \cref{fig:style_ll}: most of
them elbow around $K=3$, the value we retain in our experiments. We set
$\varepsilon_r = 0.1$ (as the dominant eigenvalue of the target covariance
matrices empirically lies between $10^2$ and $10^3$).

\begin{figure}[ht]
  \centering
  \begin{subfigure}[t]{.24\linewidth}
    \centering
    \includegraphics[width=\linewidth]{figures/eiffel.jpg}
  \end{subfigure}
  \hfill
  \begin{subfigure}[t]{.24\linewidth}
    \centering
    \includegraphics[width=\linewidth]{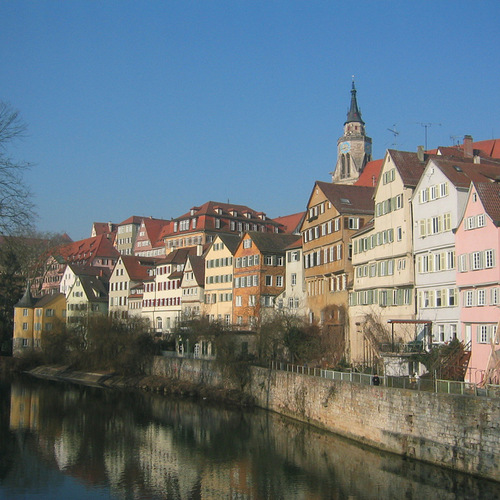}
  \end{subfigure}
  \hfill
  \begin{subfigure}[t]{.24\linewidth}
    \centering
    \includegraphics[width=\linewidth]{figures/hokusai.jpg}
  \end{subfigure}
  \hfill
  \begin{subfigure}[t]{.24\linewidth}
    \centering
    \includegraphics[width=\linewidth]{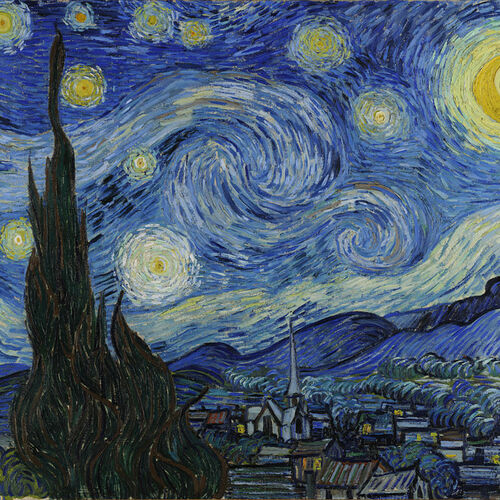}
  \end{subfigure}
  \begin{subfigure}{0.24\textwidth} \centering
  \includegraphics[width=\textwidth]{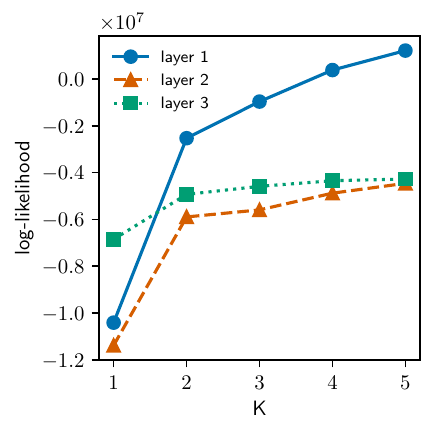} \caption{Eiffel
  Tower} \end{subfigure}
  \hfill
  \begin{subfigure}{0.24\textwidth} \centering
  \includegraphics[width=\textwidth]{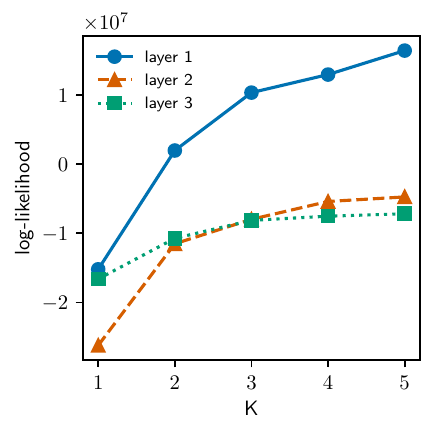}
  \caption{Tuebingen} \end{subfigure}
  \hfill
  \begin{subfigure}{0.24\textwidth} \centering
  \includegraphics[width=\textwidth]{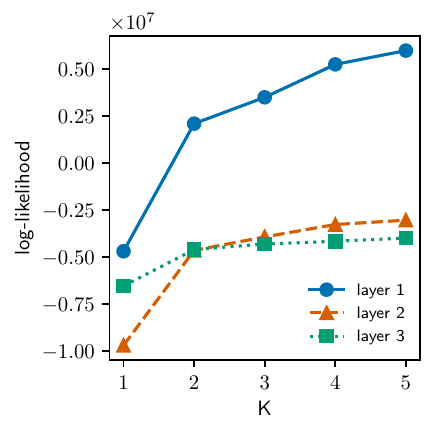} \caption{The Great
  Wave} \end{subfigure}
  \hfill
  \begin{subfigure}{0.24\textwidth} \centering
  \includegraphics[width=\textwidth]{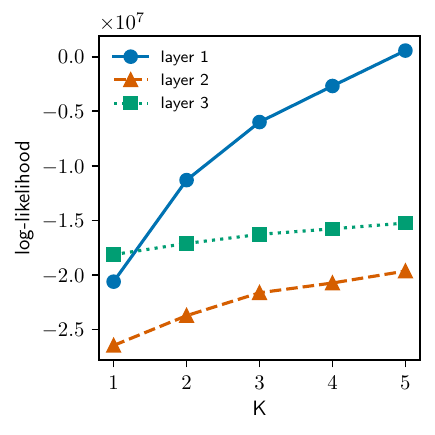} \caption{The Starry
  Night} \end{subfigure} \caption{Four images and their corresponding
  log-likelihood vs. $K$ plots, for VGG layers $\ell\in\{1,\cdots,3\}$.}
  \label{fig:style_ll}
\end{figure}

See \cref{fig:style_K_comp} for a comparison of style transfer with $K=1$ and
$K=3$. Taking higher values of $K$ does not yield significant improvement in the
results. Yet, for \cref{fig:eiffel_st}, the sky is less uniform with $K=1$ and
the bottom-left corner is more blurry. For \cref{fig:cri_st}, the sky differs a
bit between $K=1$ and $K=3$. In \cref{fig:style_film}, we illustrate some
gradient descent iterations, and observe the progressive refinement of the
stylisation. It appears that the colour correspondance is achieved quite early,
with stylistic details appearing later in the optimisation.

\begin{figure}[ht]
  \centering
  \begin{subfigure}[t]{.32\linewidth}
    \centering
    \includegraphics[width=\linewidth]{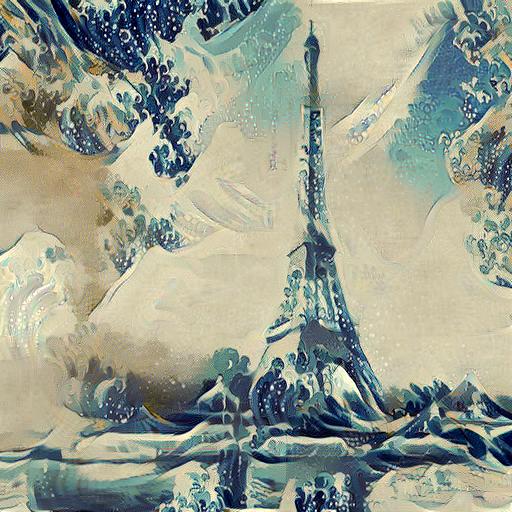}
  \end{subfigure}
  \hfill
  \begin{subfigure}[t]{.32\linewidth}
    \centering
    \includegraphics[width=\linewidth]{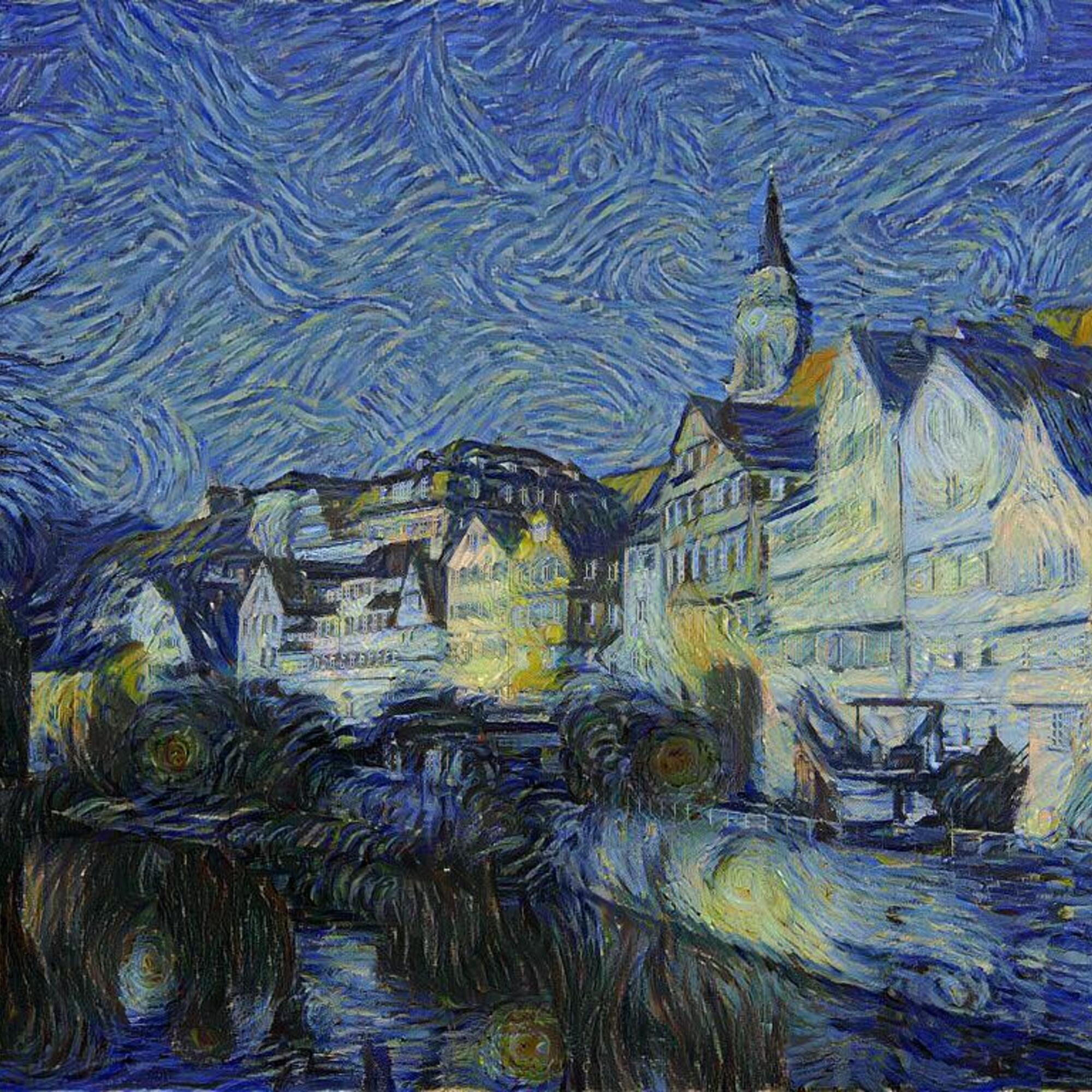}
  \end{subfigure}
  \hfill
  \begin{subfigure}[t]{.32\linewidth}
    \centering
    \includegraphics[width=\linewidth]{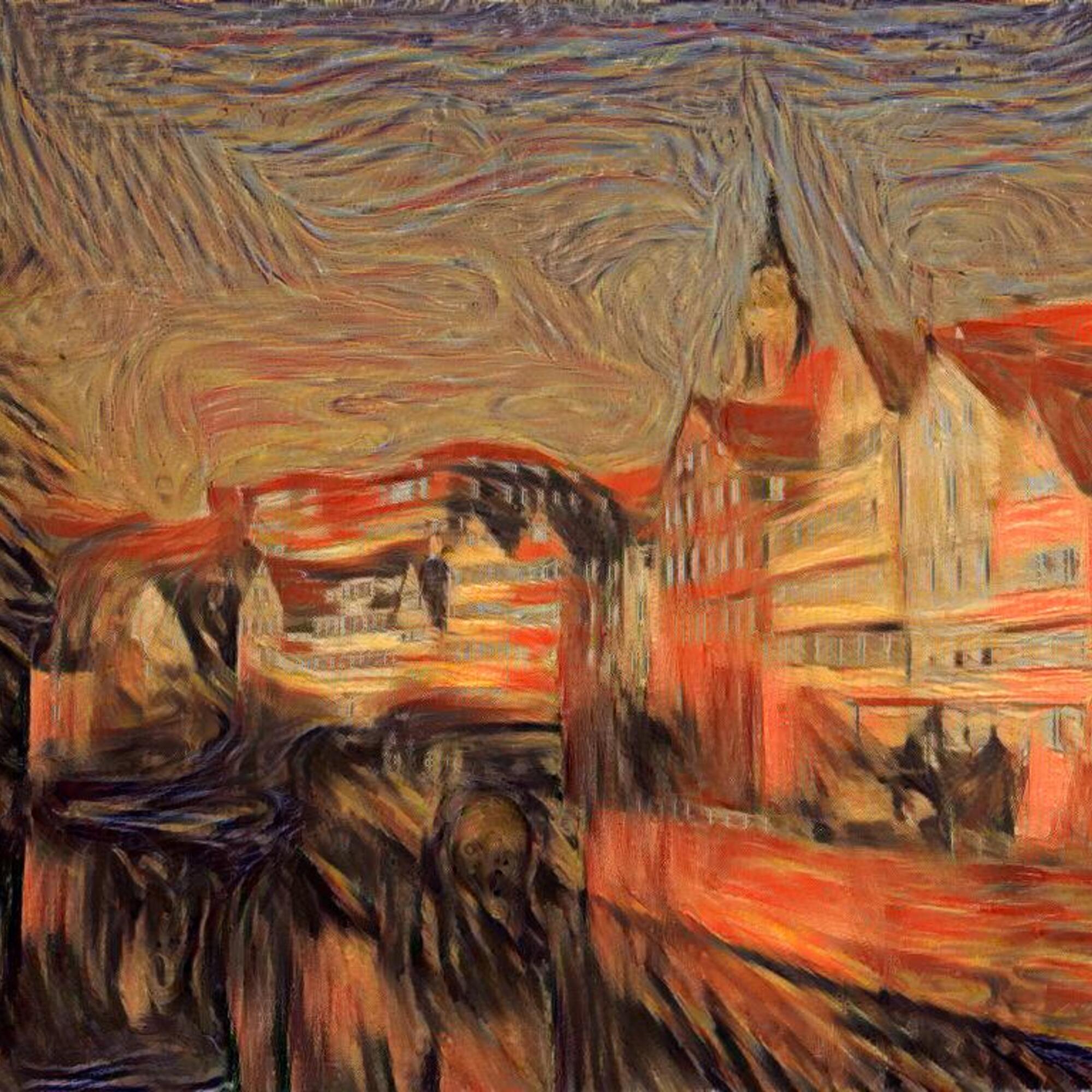}
  \end{subfigure}
  \begin{subfigure}[t]{.32\linewidth}
    \centering
    \includegraphics[width=\linewidth]{figures/style_transfer_K3.jpg}
    \caption{Eiffel $\rightarrow$ The Great Wave}
    \label{fig:eiffel_st}
  \end{subfigure}
  \hfill
  \begin{subfigure}[t]{.32\linewidth}
    \centering
    \includegraphics[width=\linewidth]{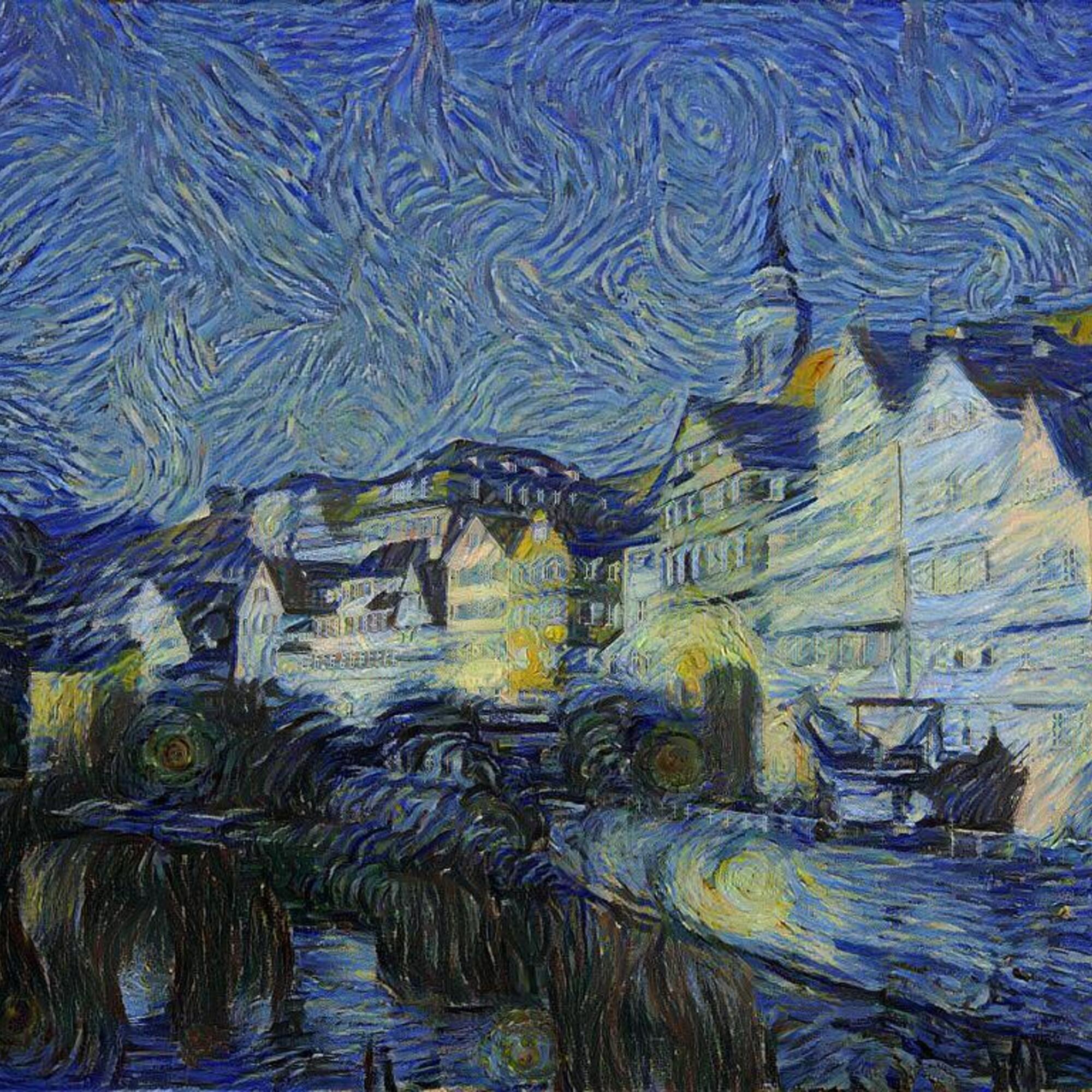}
    \caption{Tuebingen $\rightarrow$ The Starry Night}
  \end{subfigure}
  \hfill
  \begin{subfigure}[t]{.32\linewidth}
    \centering
    \includegraphics[width=\linewidth]{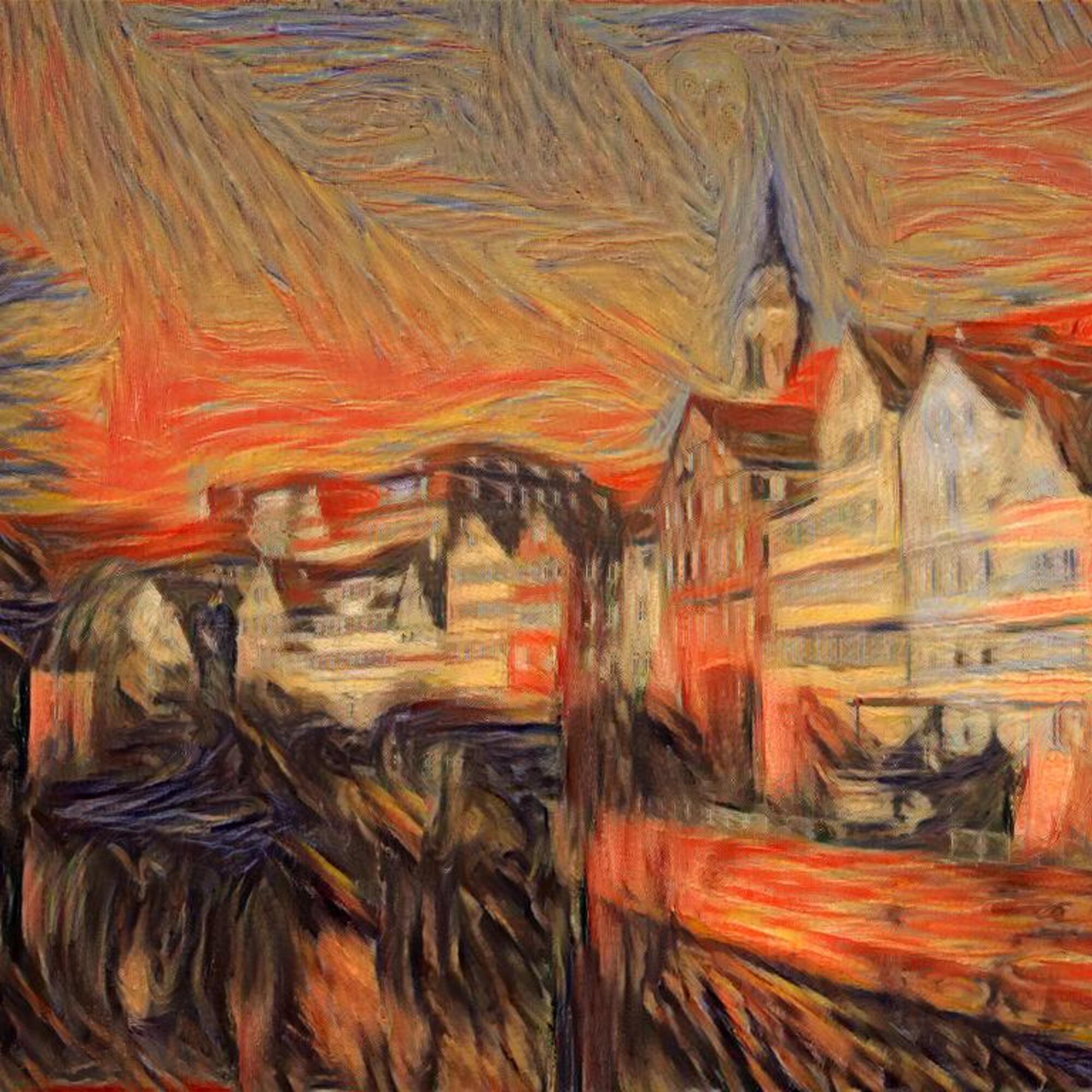}
    \caption{Tuebingen $\rightarrow$ The Scream}
    \label{fig:cri_st}
  \end{subfigure}
  \caption{Taking $K=1$ (top) gives results comparable to $K=3$ (bottom).}
  \label{fig:style_K_comp}
\end{figure}

\begin{figure}[ht]
  \centering
  \begin{subfigure}[t]{.19\linewidth}
    \centering
    \includegraphics[width=\linewidth]{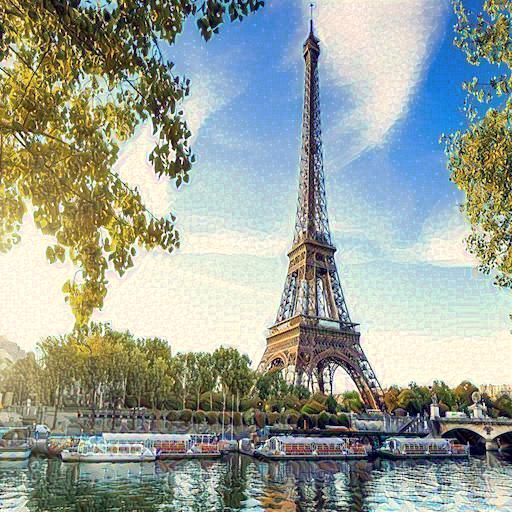}
    \caption{Iteration 10}
  \end{subfigure}
  \hfill
  \begin{subfigure}[t]{.19\linewidth}
    \centering
    \includegraphics[width=\linewidth]{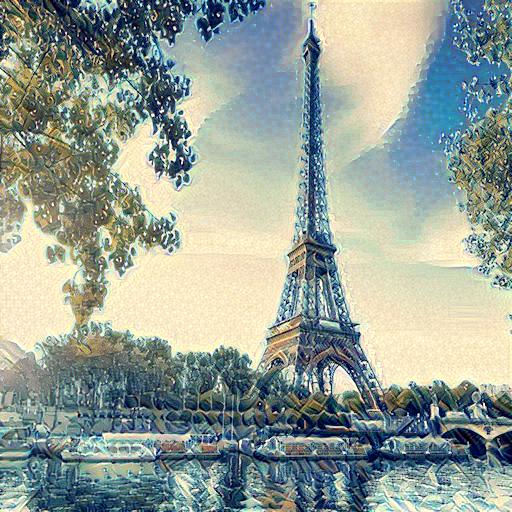}
    \caption{Iteration 30}
  \end{subfigure}
  \hfill
  \begin{subfigure}[t]{.19\linewidth}
    \centering
    \includegraphics[width=\linewidth]{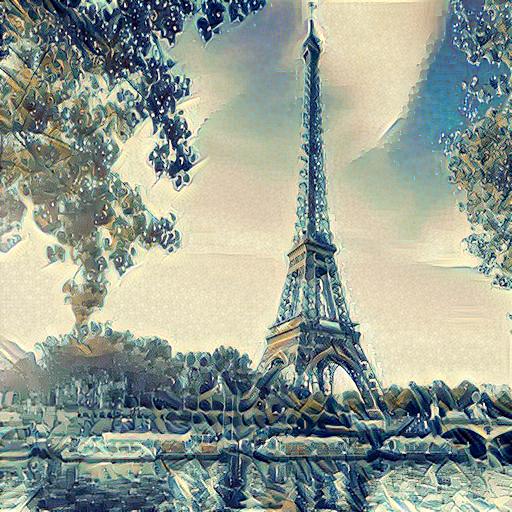}
    \caption{Iteration 50}
  \end{subfigure}
  \hfill
  \begin{subfigure}[t]{.19\linewidth}
    \centering
    \includegraphics[width=\linewidth]{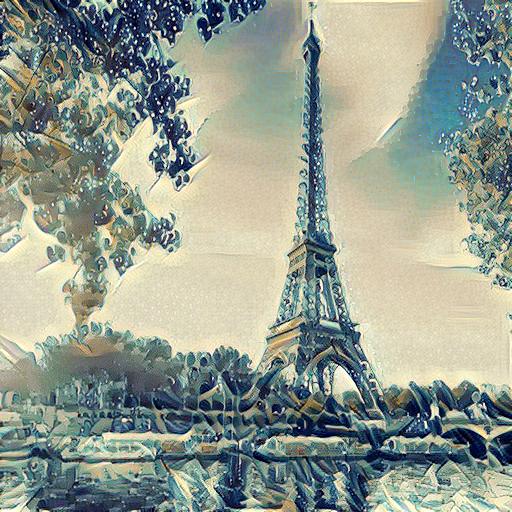}
    \caption{Iteration 70}
  \end{subfigure}
  \hfill
  \begin{subfigure}[t]{.19\linewidth}
    \centering
    \includegraphics[width=\linewidth]{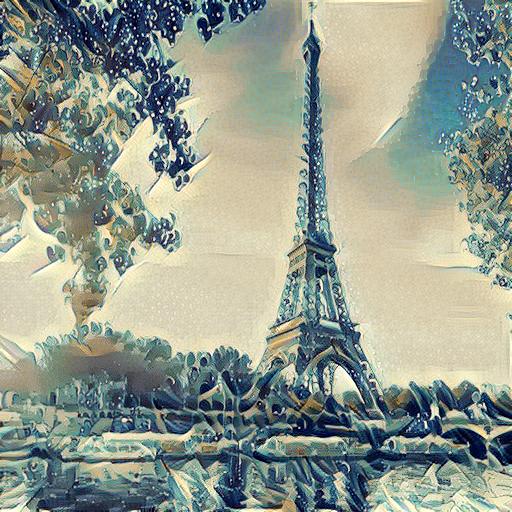}
    \caption{Iteration 90}
  \end{subfigure}
  \caption{Evolution of the style transfer across iterations.}
  \label{fig:style_film}
\end{figure}

\subsubsection{Texture Synthesis}\label{sec:texture_synthesis_details}

Consider a (space-periodic) target texture $u\in [0, 1]^{h\times w \times C}$.
Our objective is to produce a (space-periodic) texture $x \in [0, 1]^{H\times W
\times C}$. To initialise, we sample a stationary Gaussian field $Z \in
\R^{H\times W \times C}$ of i.i.d. entries of law $\mathcal{N}(0, I_C)$. The
initialisation is then a stationary Gaussian field with the same statistics as
$u$, defined as $\mathbb{T}_{H,W}$ by $x_0 := m + u \star Z$, where $m\in \R^C$
is the mean of $u$ and $\star$ denotes the discrete convolution with periodic
boundary conditions.

Given a texture $v \in \R^{H\times W \times C}$, we consider $P_p(v) \subset [0,
1]^{C\times p^2}$ the set of its $p\times p$ patches (with periodic boundary
conditions and indexing from the top-left corner). We also define the
down-sampling operator $D_{s_{i}}$ that shrinks the image by $2^{s_{i}}$. The
patch distribution of a texture $v$ is noted
$\mu_{p}^{v}=\frac{1}{|P_{p}(v)|}\sum_{y\in P_{p}(v)}\delta_{y}$.

For a given list of scales $ \mathcal{S}=\left\{(p_{i},s_{i})\right\}_{1\le i\le
L},$ we aim to make the patch distributions $\mu_{p}(x)$ of the optimised
texture $x$ match the targets $\mu_{p}(v)$. Each $\mu_{p}^{v}$ is approximated
by a Gaussian mixture $\hat{\mu}_{p}^{v}$. The target mixture
$\hat{\mu}_{p}^{u}$ is fitted once by EM and fixed. We use the \emph{Warm-Start
EM} (see \cref{alg:warm_start_EM_flow}) variant during optimisation, and take
$\varepsilon_r = 10^{-3}$. We solve the problem
\[
\min_x \sum_{i=1}^{L}2^{2s_i}
\MW_2^2\left(\hat{\mu}_{p_{i}}^{D_{s_{i}}x},
\hat{\mu}_{p_{i}}^{D_{s_{i}}u}\right).
\]
At the end, we perform a nearest neighbour projection on the largest scale: each
patch in the generated $x$ is matched to its nearest patch in the target $u$.
Then, each pixel is reconstructed by averaging the corresponding patches. In
\cref{fig:texto} we observe an example of this process for a single scale. For
this simple example, even the lighter mono-scale model produces satisfactory
results.

\begin{figure}[ht]
    \centering
    \begin{subfigure}{0.3\textwidth} \centering
    \includegraphics[width=\textwidth]{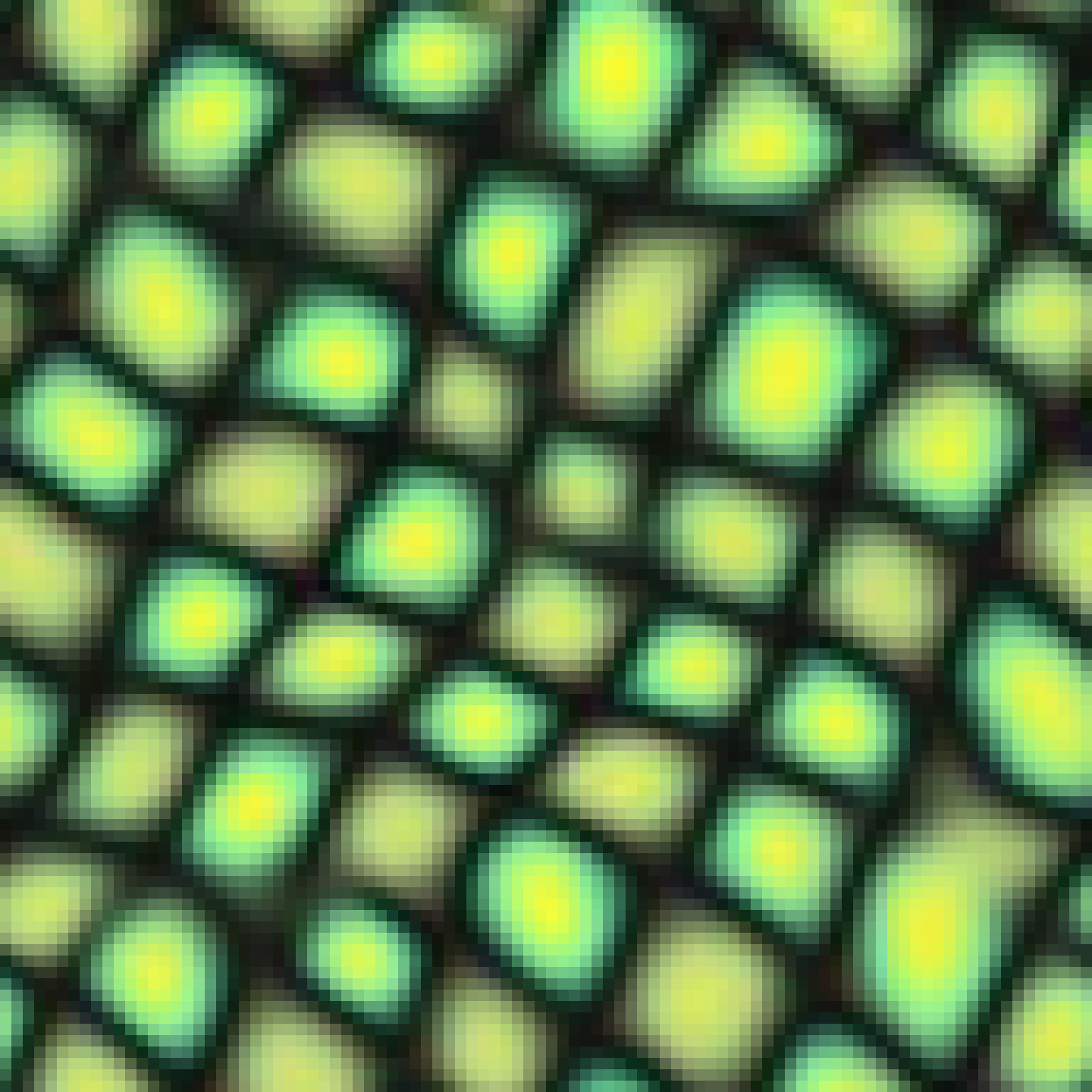}\caption{}
    \end{subfigure}
    \hspace{1cm}
    \begin{subfigure}{0.2\textwidth} \centering
    \includegraphics[width=\textwidth]{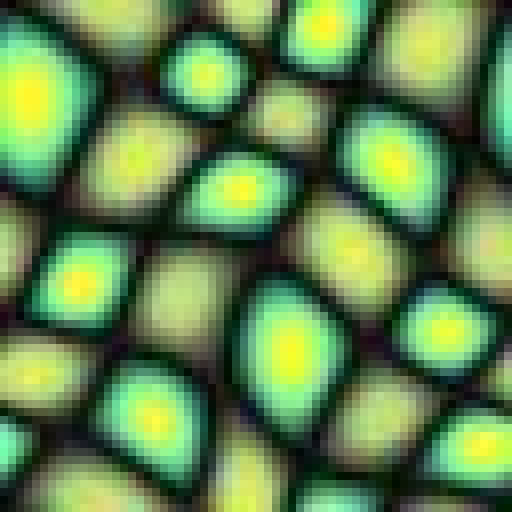}\caption{}
    \end{subfigure}
    \hfill
    \caption{\label{fig:texto}(a) Texture synthesis with $12\times 12$ patches,
    $K=4$, mono-scale, (b) reference.}
\end{figure}

\subsubsection{Image Generation}\label{sec:em_mw2_gan}

As a proof of concept, we train a Generative Adversarial Network (GAN)
\cite{goodfellow2014generative,arjovsky2017wasserstein,deshpande_generative_sw}
using the Mixture--Wasserstein distance as a regularisation. The idea is to
encourage the generator to produce images which have similar features to real
images: while this does not suffice to produce realistic images, it can guide
the training optimisation out of spurious local minima induced by the
notoriously unstable GAN training (\cite{arjovsky2017wasserstein}). Given a
batch size $b$ and a latent dimension $\ell$, we independently sample
$x_{1},\cdots,x_{b}\sim\mathcal{N}(0_{\ell},I_{\ell}),$ and we draw
$y_{1},\cdots,y_{b}$ from our dataset of real images. The aim of our
\emph{generator} $\mathtt{G}:\mathbb{R}^{\ell}\longrightarrow\mathbb{R}^{C\times
H\times W}$ is to map these Gaussian samples to real images of height $H$, width
$W$ and $C$ channels. We also introduce a \emph{feature extractor}
$\mathtt{F}:\mathbb{R}^{C\times H\times W}\longrightarrow\mathbb{R}^{f}$ which
learns to map images to relevant features of dimension $f$, and a
\emph{discriminator} $\mathtt{D}:\mathbb{R}^{f}\longrightarrow\mathbb{R}$ whose
goal is to discriminate real and fake images by producing different vectors of
features. The full adversarial objective on a noise batch $X$ and a real image
batch $Y$ is then defined as follows:
\begin{equation}\label{eqn:mw2_gan}
    \min_\mathtt{G}\ \left[
    \MW_2^2\left(F^{T}(\theta_0,\mathtt{F}\circ\mathtt{G}(X)), 
    F^{T}(\theta_0',\mathtt{F}(Y))\right) +
    \max_{\mathtt F,\mathtt D}\ 
    \left(\sum_{i=1}^{b} \log \mathtt{D}\circ\mathtt{F}\left(y_{i}\right) 
    + \log\left(1-\mathtt{D}\circ\mathtt{F}\circ\mathtt{G}(x_{i})\right)
    \right)\right],
\end{equation}
where $\theta_0$ and $\theta_0'$ are chosen using \emph{k-means} initialisation.
Note that the feature extractor $\mathtt{F}$ appears in the $\MW_2^2$
term, but gradients from this term are not used to update
$\mathtt{F}$. We optimise these losses using Stochastic Gradient Descent and
\emph{Full Automatic Differentiation} (see \cref{sec:em_grad_methods}). For
every sampled batch, we alternate one step on $\mathtt{F}$ and $\mathtt{D}$,
then one step on $\mathtt{G}$. The generator is encouraged to produce images
with similar features to the real ones, and aims to fool the discriminator. The
feature extractor $\mathtt{F}$ attempts to extract features such that
$\mathtt{D}$ is able to discriminate real from fake images. The full network
architecture is described in \cref{fig:image_gen}. Note that as in other GANs
\cite{goodfellow2014generative,arjovsky2017wasserstein,deshpande_generative_sw},
we do not optimise over the full dataset and rather optimise by sampling
mini-batches at each step, incurring a seldom-discussed bias
\cite{fatras2020learning,fatras2021minibatch,fatras2021unbalanced,tong2024improving}
for computational efficiency.

\begin{figure}[ht]
    \centering
    \begin{subfigure}{0.6\textwidth} 
        \centering
        \includegraphics[width=\textwidth]{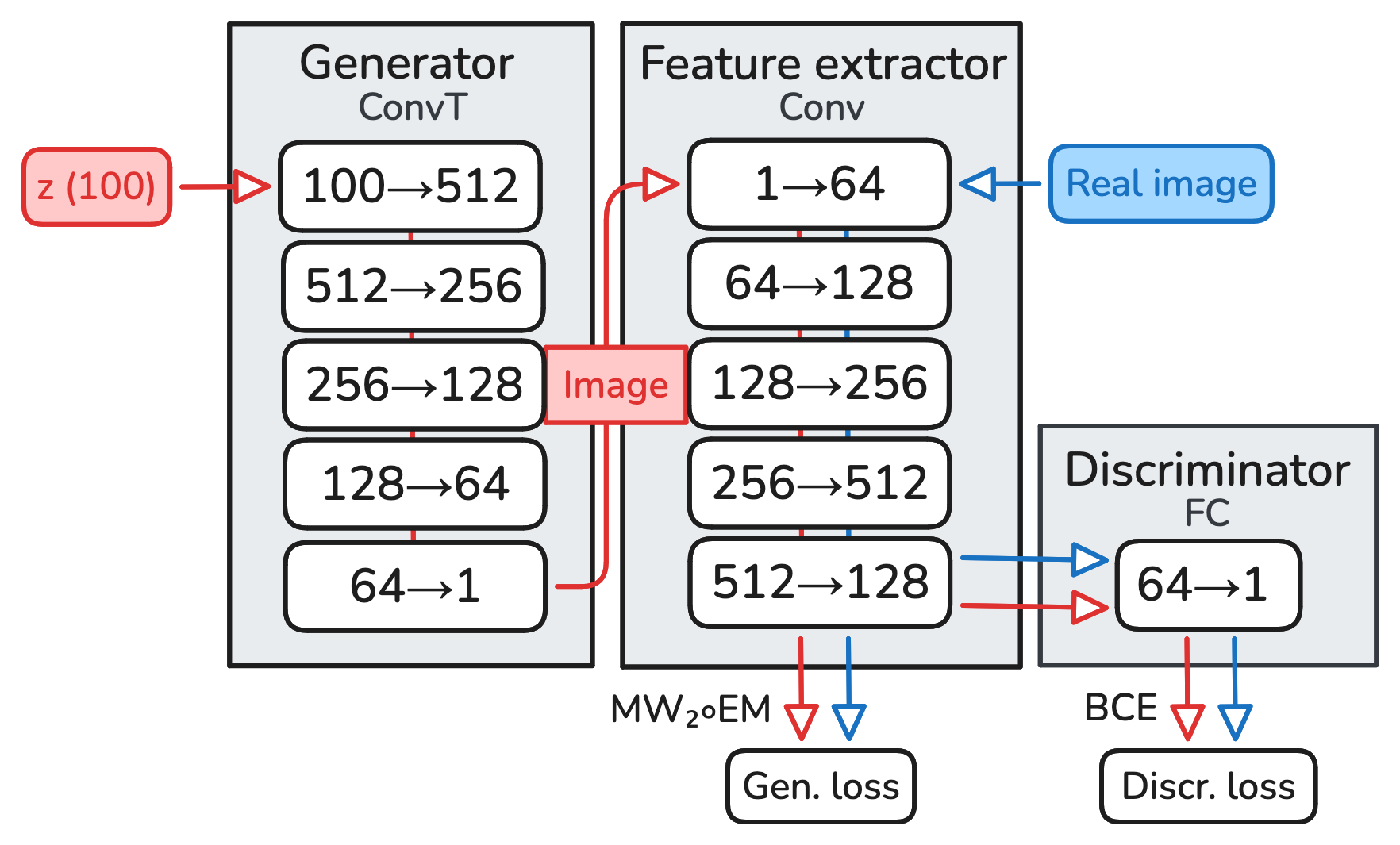}
        \caption{GAN architecture (adapted from DCGAN
        \cite{radford2016unsupervised})} 
    \end{subfigure}
    \hfill
    \begin{subfigure}{0.35\textwidth} 
        \centering
        \includegraphics[width=\textwidth, trim=40 40 40 40, clip]{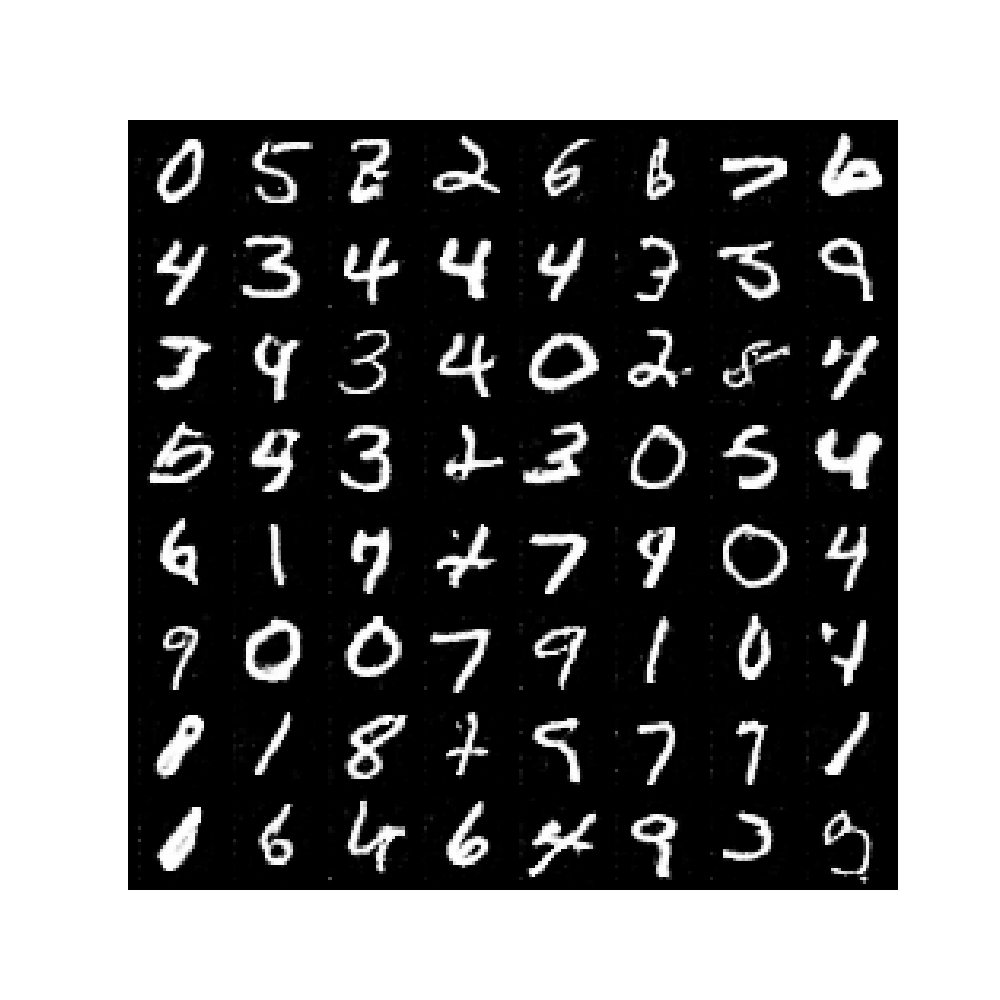}
        \caption{MNIST generation with $K=10$.} 
    \end{subfigure}
    \label{fig:image_gen}
    \caption{Image generation with the $\MW_2^2$-GAN from
    \cref{eqn:mw2_gan}.}
\end{figure}

\end{document}